\documentclass{article}
\usepackage[utf8]{inputenc} %
\usepackage[T1]{fontenc}    %
\usepackage{fullpage}

\usepackage[round]{natbib}

\renewcommand{\cite}{\citep}  %

\usepackage[margin = 1in]{geometry}

\usepackage{xcolor}

\renewcommand{\cite}{\citep}

\usepackage[T1]{fontenc}
\usepackage[utf8]{inputenc}
\usepackage{latexsym}
\usepackage{amsmath}
\usepackage{amssymb}
\usepackage{bm}
\usepackage{bbm}
\usepackage{graphicx}
\usepackage{breakcites}
\usepackage{amsthm}
\usepackage{mdframed}
\usepackage{lipsum}
\usepackage{algorithm}  
\usepackage[noend]{algpseudocode}
\usepackage[english]{babel}
\usepackage{delimset} %
\usepackage{tikz}

\usepackage{caption}
\usepackage{subcaption}

\usepackage{url}            %
\usepackage[colorlinks = true,
            linkcolor = blue,
            urlcolor  = blue,
            citecolor = blue,
            anchorcolor = blue]{hyperref}
\usepackage{booktabs}       %
\usepackage{amsfonts}       %
\usepackage{microtype}      %
\usepackage{wrapfig}
\usepackage{caption}

\usepackage{minitoc} %
\usepackage{xspace}
\usepackage{enumitem}
\usepackage{cleveref}

\usepackage{amsthm}
\usepackage{graphicx}
\usepackage{latexsym}
\usepackage{mathtools}
\usepackage{multirow}
\usepackage{fancyvrb}
\usepackage{adjustbox}

\crefname{section}{Section}{Sections}
\crefname{appendix}{Appendix}{Appendices}

\crefname{theorem}{Theorem}{Theorems}
\crefname{lemma}{Lemma}{Lemmas}
\crefname{corollary}{Corollary}{Corollaries}
\crefname{proposition}{Proposition}{Propositions}
\crefname{definition}{Definition}{Definitions}
\crefname{assumption}{Assumption}{Assumptions}

\Crefname{algorithm}{Algorithm}{Algorithms}
\crefname{figure}{Figure}{Figures}
\crefname{table}{Table}{Tables}

\newcommand{\myparagraph}[1]{\par\noindent\textbf{{#1}.}} %

\newcommand{\Expect}{\mathbb{E}}
\newcommand{\supq}{S}
\newcommand{\Var}{\mathbb{V}}
\newcommand{\Cov}{\mathrm{Cov}}

\newcommand{\SIF}{I}  %
\newcommand{\id}{\mathbf{I}} %

\renewcommand{\epsilon}{\varepsilon}  %

\newcommand \reals {\mathbb{R}}
\newcommand \T {^{\top}}	%
\newcommand{\ones}{\mathbf{1}}
\newcommand \expect {\mathbb{E}}

\newcommand{\Prob}{\mathbb{P}}
\newcommand{\eps}{\epsilon}
\newcommand{\Rank}{\operatorname{\bf Rank}}
\newcommand{\Tr}{\operatorname{\bf Tr}}

\DeclarePairedDelimiterX{\inp}[2]{\langle}{\rangle}{#1, #2} 
\newcommand \diag {\operatorname*{Diag}}
\newcommand{\myabs}[1]{\left\lvert #1 \right\rvert}
\newcommand{\mynorm}[1]{\left\| #1 \right\|}

\newcommand \argmin {\operatorname*{arg\,min}} %
\newcommand \conv {\operatorname*{conv}} %
\newcommand \poly {\operatorname*{poly}} %

\newcommand \grad {\nabla}

\newcommand \D {\mathrm{d}}

\newcommand \Zcal {\mathcal Z}

\newcommand \Acal {\mathcal A}
\newcommand \Bcal {\mathcal B}
\newcommand \Ccal {\mathcal C}
\newcommand \Hcal {\mathcal H}
\newcommand \Kcal {\mathcal K}
\newcommand \Gcal {\mathcal G}

\newcommand \Qcal {\mathcal Q}

\newcommand \Ncal {\mathcal N}

\newcommand{\Sigmam}{\varSigma}

\newtheorem{theorem}{Theorem}
\newtheorem{lemma}[theorem]{Lemma}

\newtheorem{proposition}[theorem]{Proposition}
\newtheorem{corollary}[theorem]{Corollary}
\newtheorem{remark}[theorem]{Remark}
\newtheorem{definition}[theorem]{Definition}
\newtheorem{assumption}{Assumption}

\newtheorem{innercustomasmp}{Assumption}
\newenvironment{customasmp}[1]
  {\renewcommand\theinnercustomasmp{#1}\innercustomasmp}
  {\endinnercustomasmp}
  
\newtheorem{innercustomthm}{Theorem}
\newenvironment{customthm}[1]
  {\renewcommand\theinnercustomthm{#1}\innercustomthm}
  {\endinnercustomthm}

\newcommand{\est}{\theta_n}
\newcommand{\staterr}{\mathcal{E}}
\newcommand{\event}{\mathcal{G}}

\usepackage{sidecap}
\sidecaptionvpos{figure}{c}

\newcommand\blfootnote[1]{%
  \begingroup
  \renewcommand\thefootnote{}\footnote{#1}%
  \addtocounter{footnote}{-1}%
  \endgroup
} 

\title{Statistical and Computational Guarantees for Influence Diagnostics}
\author{Jillian Fisher$^1$ \qquad Lang Liu$^1$ \qquad Krishna Pillutla$^{1\dagger}$ \qquad Yejin Choi$^{1, 2}$ \qquad Zaid Harchaoui$^1$\vspace{0.3cm} \\
$^1$University of Washington
$\qquad$
$^2$Allen Institute for Artificial Intelligence
}\date{\vspace{-1em}}

\begin{document}
\maketitle
\blfootnote{
$^\dagger$Now at Google Research
}

\begin{abstract}

Influence diagnostics such as influence functions and approximate maximum influence perturbations are popular in machine learning and in AI domain applications. Influence diagnostics are powerful statistical tools to identify influential datapoints or subsets of datapoints. We establish finite-sample statistical bounds, as well as computational complexity bounds, for influence functions and approximate maximum influence perturbations using efficient inverse-Hessian-vector product implementations. We illustrate our results with generalized linear models and large attention based models on synthetic and real data.  \end{abstract}

\doparttoc %
\faketableofcontents %

\section{INTRODUCTION}

Statistical machine learning models have been increasingly used in fully or partially automatized data analysis processes and artificial intelligence applications~\cite{rudin2019stop}. The automatizing of decisions impacting the society inspire a parallel effort to develop methods to identify the factors impacting specific decisions. The heightened scrutiny on the way statistical models now operate at a large scale and at a fast pace has led to a renewed interest in statistical diagnostics such as the influence function~\cite{cook,koh,arnoldi,louvet2022influence}.

The influence function or curve of a statistical estimator has been proposed to measure the sensitivity of the estimator to individual datapoints. 
Computing the influence of a particular datapoint boils down to computing an inverse-Hessian-vector product. Due to a greater focus on least-squares-type estimator with small samples, the computational aspects have received relatively little attention until recently~\cite{koh,arnoldi}, while the statistical aspects have mainly focused on large sample classical asymptotics~\cite{rousseeuw2011robust,marco}. 

The statistical analysis of influence functions for generalized linear models presents several challenges. For non-squared loss functions, the curvature captured by the Hessian varies away from the true parameter $\theta_\star$, a property that can be modelled using self-concordance. Moreover, non-asymptotic analyses for misspecified generalized linear models require recently developed tools such as matrix concentration inequalities~\cite{mackey2014matrix}. We present non-asymptotic statistical bounds for influence functions of generalized linear models under pseudo self-concordance assumptions. Thanks to a novel interpretation of ~\citet{broderick2020automatic}'s maximum subset influence using superquantiles, we also obtain non-asymptotic guarantees for this diagnostic tool as well. 

The computational analysis of influence is equally interesting. The statistical and computational trade-offs have not received attention to the best of our knowledge. We review classical algorithms such as the conjugate gradient method~\cite{saad2003iterative,bai2021matrix} and an approach using the Arnoldi iteration~\cite{arnoldi}, and we develop approaches using variance reduced stochastic optimization algorithms~\cite{bertsekas2015convex,bach2021learning}.
 Our analysis reveals interesting trade-offs depending on the near low-rank structure that is the eigendecay of the Hessian for small to moderate sample sizes relative to the dimension, as well as the potential benefits of using linearly convergent stochastic algorithms.
\myparagraph{Outline}
In \Cref{sec:bg}, we introduce influence diagnostics and the computational challenges they present in high dimensional settings. In \Cref{sec:influence}, we obtain finite-sample bounds on empirical influence functions for generalized linear models. We also achieve computational accuracy bounds on empirical influence functions computed using deterministic Krylov-based methods and stochastic optimization based methods. In \Cref{sec:mis}, we provide similar guarantees for maximum subset influence owing to a novel .superquantile interpretation. Lastly, in \Cref{sec:expt}, we provide numerical illustrations of our theoretical bounds on synthetic data and real data, with generalized linear models and large attention based models.

\section{INFLUENCE FUNCTIONS}
\label{sec:bg}

We are interested in the parameter $\theta_\star \in \Theta = \reals^p$ defined as 
\begin{align} \label{eq:population-min}
    \theta_\star := \argmin_{\theta \in \Theta} \Big[ F(\theta):= \expect_{Z \sim P}\left[ \ell(Z, \theta) \right] \Big] \,,
\end{align}
where 
$P$ is an unknown probability distribution over a data space $\Zcal$ and
$\ell: \Zcal \times \Theta \to \reals_+$ is a loss function that is closed, convex, and thrice continuously differentiable in the second argument. We assume this argmin is unique. 

For instance, binary logistic regression corresponds to $\Zcal = \reals^p \times \{\pm 1\}$ and a loss $\ell\big((x, y), \theta\big) = \log\big(1 + \exp(-y \inp{\theta}{x})\big)$. Here, problem \eqref{eq:population-min} is equivalent to finding parameters $\theta_\star \in \Theta$ that minimize the Kullback-Leiblier divergence between the unknown data distribution $P$ and the parametric model $P_\theta(Y|X=x) = 1/\big(1 + \exp(-y \inp{\theta}{x})\big)$. 

Since the data distribution $P$ is unknown, we estimate $\theta_\star$ using an i.i.d. sample $Z_{1:n} := (Z_1, \cdots, Z_n) \sim P^n$. This leads to the M-estimation problem,
\begin{align} \label{eq:sample-min}
    \theta_n := \argmin_{\theta \in \Theta} \frac{1}{n} \sum_{i=1}^n \ell(Z_i, \theta) \,,
\end{align}
where we assume the argmin to be unique. 
For the logistic regression example, $\theta_n$ is also the maximum likelihood estimator of $\theta_\star$. 

\myparagraph{Influence Functions}
We quantify the influence of a fixed data point $z$ on the estimator $\theta_n$ using the perturbation
\begin{align*}
    \theta_{n, \eps, z} := \argmin_{\theta \in \Theta} 
    \left\{
        \frac{1-\eps}{n} \sum_{i=1}^n \ell(Z_i, \theta) 
        + \eps  \, \ell(z, \theta) 
    \right\} 
\end{align*}
for some $\eps > 0$. 
The difference $(\theta_{n, \eps, z} - \theta_n)/\eps$ 
is a measure of the local effect that the datapoint $z$ has on the estimator $\theta_n$, as illustrated in \Cref{fig:inf-illustration}.
Influence functions provide a way to avoid recomputing this estimator for each $z \in \Zcal$ of interest by using a linear approximation of the map $\eps\mapsto \theta_{n, \eps, z}$~\cite{hampel1974influence}.
Concretely, we approximate
\begin{align} \label{eq:linearization-estimate}
    \frac{\theta_{n, \eps, z} - \theta_n}{\eps} 
    \approx \frac{\D \theta_{n, \eps, z}}{\D \eps} \Big|_{\eps=0} =: I_n(z) \,.
\end{align}
This quantity is well-defined when the Hessian $H_n(\theta) := (1/n)\sum_{i=1}^n \grad^2 \ell(Z_i, \theta)$ is invertible at $\theta = \theta_n$. We bound this approximation error in \Cref{thm:linearization_paper}.

This idea of taking infinitesimal perturbations to approximate the effect of modifying data in statistics dates back to the Ph.D. dissertation of \citet{Hampel1968ContributionsTT} and subsequently, the infinitesimal jackknife~\cite{jaeckel1972infinitesimal}. 
A celebrated result of \citet{cook}, obtained from invoking the implicit function theorem to differentiate through the first order optimality conditions of $\theta_n$, gives the closed-form  
\begin{align} \label{eq:IF:closed-form}
    I_n(z) =  - H_n(\theta_n)^{-1}  \grad \ell(z, \theta_n) \,.
\end{align}
Since $I_n(z)$ does not depend on $\theta_{n,\epsilon,z}$, there is no need to re-solve the M-estimation problem for each $z$. 
Instead, we solve
a single linear system involving $H_n(\theta_n)$; we return to the computational aspects later. %

In this work, we are interested in the non-asymptotic statistical behavior of the influence function $I_n(z)$. 
To define the population limit, we denote the perturbed population minimizer with an $\eps$-fraction of the mass moved to $z$ as,
\[
    \theta_{\star, \eps, z} 
    := \argmin_{\theta \in \Theta} 
    \left\{
        \expect_{Z \sim (1-\eps)P + \eps \delta_z} \left[ \ell(Z, \theta) \right]
    \right\} \,,
\]
where $\delta_z$ denotes the point mass at $z$. 
The population influence function 
is defined similar to \eqref{eq:linearization-estimate} as the derivative 
\begin{align}
    I(z) := \frac{\D \theta_{\star, \eps, z}}{\D \eps} \Big|_{\eps = 0} 
    = \lim_{\eps \to 0} \frac{\theta_{\star, \eps, z} - \theta_\star}{\eps} 
    \,.
\end{align}
If the Hessian $H_\star = \grad^2 F(\theta_\star)$ of the population objective \eqref{eq:population-min} is strictly positive definite at $\theta_\star$, we get a closed form expression similar to \eqref{eq:IF:closed-form} due to \citet{cook}: 
\begin{align} \label{eq:IF-pop:closed-form}
    I(z) =  - H_\star^{-1}  \grad \ell(z, \theta_\star) \,.
\end{align}

As $n \to \infty$, uniform convergence arguments would give $\theta_n \to \theta_\star$ in probability under appropriate assumptions. From the continuous mapping theorem, we would expect that the sample influence function $I_n(z) = -H_n(\theta_n)^{-1} \grad \ell(z, \theta_n)$  converges to the population influence $I(z) = -H_\star^{-1} \grad \ell(z, \theta_\star)$. We establish finite-sample bounds in \Cref{sec:influence} to formalize this convergence.

\myparagraph{Most Influential Subset} 
Similar to measuring the influence of a fixed point $z$, we  also consider the influence of subsets of the sample $Z_{1:n}$. Given a scalar $\alpha \in (0, 1)$, the most influential subset method of \citet{broderick2020automatic} aims to find the subset of the data of size at most $\alpha n$ that, when removed, leads to the largest increase of a continuously differentiable test function $h:\reals^p \to \reals$. A typical example of $h$ is the loss $h(\theta) = \ell(z_{\text{test}}, \theta)$ of a fixed test point $z_{\text{test}}$.

This approach relies on perturbing the weights of a weighted M-estimation problem around the nominal weights~\cite{giordano2019swiss}. Given weights $w$ in the probability simplex $\Delta^{n-1}$, define 
$\theta_{n, w} := \argmin_{\theta \in \Theta} \sum_{i=1}^n w_i \ell(Z_i, \theta)$,
so that $\theta_n = \theta_{n, \ones_n/n}$. 
Finding the maximum influence of any subset of data of size at most $\alpha n$ for a test function $h$ amounts to solving $\max_{w \in W_\alpha} h\big( \theta_{n, w}\big)$ where 
\begin{gather*}
    W_\alpha := \left\{ 
    w \in \Delta^{n-1} \, : \,
    \parbox{12em}{at most $\alpha n$ elements of $w$ are zero and the rest are equal} 
    \right\} \,.
\end{gather*}
The most influential subset corresponds to the zero entries of the maximizing $w$. 
Unfortunately, this expression cannot be computed tractably as $|W_\alpha|$ grows exponentially in $n$. 
Instead, \citet{broderick2020automatic} 
use a linear approximation
\[
    h(\theta_{n, w})
    \approx h(\theta_n)
    + \inp*{w - \frac{\ones_n}{n}}{\grad_w h(\theta_{n, w}) \Big|_{w = \ones_n/n}} \,.
\]
Finding the most influential subset according to this linear approximation leads to the maximum subset influence
\begin{align} \label{eq:sif-def}
    \SIF_{\alpha, n}(h)
    := \max_{w \in W_\alpha} \inp*{w}{\grad_w h(\theta_{n, w}) \Big|_{w = \ones_n/n}} \,.
\end{align}
Similar to \eqref{eq:IF:closed-form}, the implicit function theorem together with the chain rule gives the closed form 
\begin{align} \label{eq:sif}
\begin{aligned}
        \SIF_{\alpha, n}(h) = 
        \max_{w \in W_\alpha} \sum_{i=1}^n w_i v_i, \quad \text{where} \\
        v_i = - \inp*{\grad h(\theta_n)}{H_n(\theta_n)^{-1} \grad \ell(Z_i, \theta_n)} \,.
\end{aligned}
\end{align}
While the maximization over $W_\alpha$ in \eqref{eq:sif} is an instance of the NP-hard knapsack problem, its solution coincides with that of its continuous relaxation over $\conv W_\alpha$ when $\alpha n$ is an integer and the $v_i$'s are unique.
This continuous knapsack problem is solved by a greedy algorithm that zeros out the smallest $\alpha n$ entries of $v_i$'s~\cite{dantzig1957discrete}. 

In this work, we also study the non-asymptotic statistical behavior of the subset influence $\SIF_{\alpha, n}$. The population limit in this case is more subtle than for $I_n$ of \eqref{eq:IF:closed-form}. Using similar arguments, we would expect the vector $v$ to be related to the random variable $\phi(Z)$ where $\phi:\Zcal \to \reals$ maps $z \mapsto -\inp{\grad h(\theta_\star)}{H_\star^{-1} \grad \ell(z, \theta_\star)}$, but the maximum over $W_\alpha$ is tricky. 
In \Cref{sec:mis}, we rigorously define this population limit and establish convergence guarantees. 

\myparagraph{Computational Aspects}
While linearization methods based on the infinitesimal jackknife avoid recomputing the M-estimator for each $z$, a na\"ive implementation of $I_n(z)$ (and similarly, $\SIF_{\alpha, n}$) requires materializing and inverting the Hessian matrix $H_n(\theta_n) \in \reals^{p \times p}$ in $O(n p^2 + p^3)$ time with $O(p^2)$ storage. This approach does not scale to modern applications in deep learning with dense Hessians and large $n, p$.  Instead, we rely on iterative algorithms to approximately 
minimize the convex quadratic 
\begin{align} \label{eq:quadratic}
    g_n(u) := \frac{1}{2} \inp{u}{H_n(\theta_n) u} + \inp{\grad \ell(z, \theta_n)}{u} 
    \,.
\end{align}
Indeed, the unique minimizer $u_\star$ of $g_n$ satisfies $0 = \grad g_n(u_\star) = H_n(\theta_n) u_\star + \grad \ell(z, \theta_n)$ so that $u_\star = I_n(z)$ in \eqref{eq:IF:closed-form} as desired.
Modern automatic differentiation software supports the efficient computation of the Hessian-vector product $u \mapsto \grad^2 \ell(z, \theta) u$ without materializing the Hessian. We review some iterative algorithms that can achieve this. 

The conjugate gradient method is a classical algorithm to solve linear systems defined by a positive definite matrix. It converges linearly, but each iteration requires a full batch Hessian-vector product $u \mapsto H_n(\theta_n) u$.
We postpone precise rates to \Cref{sec:influence}.

Alternatively, one might optimize the quadratic $g_n(u)$ with stochastic gradient descent (SGD). Here, each iteration requires a Hessian-vector product at only one sample $Z_i$, but the convergence rate is sublinear. 
We can get a linear rate at the same $O(1)$ per-iteration complexity through the use of variance reduction with the stochastic variance reduced gradient~\cite[SVRG;][]{johnson2013accelerating} or its accelerated counterpart~\cite{lin2018catalyst}.

The LiSSA algorithm~\cite{lissa} solves this linear system by approximating the matrix inverse with its Neumann series $M^{-1} = \sum_{k=0}^\infty (\id - M)^k$ for positive definite $M$ with $\norm{M}_2 < 1$. 
By using an unbiased stochastic estimator $\grad^2 \ell(Z_{I}, \theta_n)$ to $M = H_n(\theta_n)$, where $I$ is a random index, this reduces exactly to the SGD baseline. 
See \Cref{app:pseudocode} for details.

\citet{arnoldi} propose to 
solve the linear system with a low-rank approximation of the Hessian. Concretely, let $H_n(\theta_n) = Q \Lambda Q\T$ denote its eigenvalue decomposition with $\Lambda = \diag(\lambda_1, \ldots, \lambda_d)$ arranged in non-increasing order. The rank-$k$ approximation of $v = H_n(\theta_n)^{-1} u$ is given by 
$v_k = Q \diag(\lambda_1^{-1}, \ldots, \lambda_k^{-1}, 0, \ldots, 0) Q\T u$. The $k$-largest eigenvalues and their eigenvectors are approximated using the Lanczos/Arnoldi iterations~\cite{lanczos1950iteration,arnoldi_it}. 
This algorithm requires computations of a full batch Hessian-vector product.

For a full error characterization of the influence estimate $\hat I_n(z)$ returned by an iterative algorithm, we must take into account both the statistical error $I_n(z) - I(z)$ and the computational error $\hat I_n(z) - I_n(z)$. 
This will be our goal for the next section. 

\begin{SCfigure}
    \includegraphics[width = .3\linewidth]{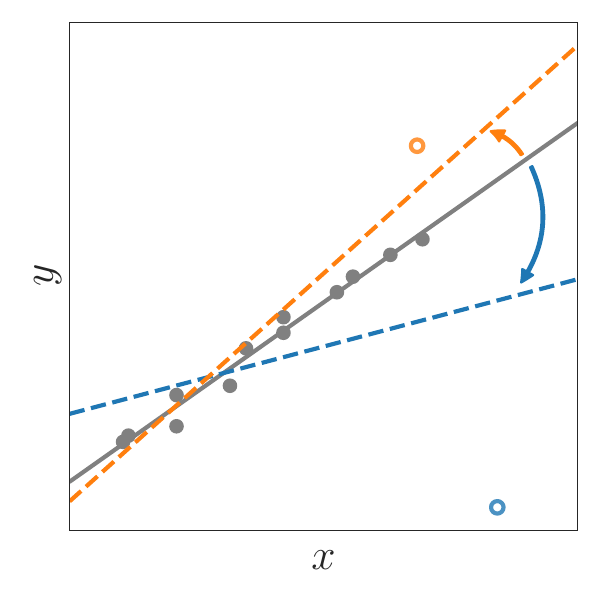}
    \caption{Illustration of the influence of a point $z$ on the model parameters. The base model (gray) line is drastically affected when the blue point is included (blue dotted line) but less affected when the orange point is included (orange dotted line).}
    \label{fig:inf-illustration}
\end{SCfigure} %
\section{ERROR ANALYSIS OF INFLUENCE ESTIMATION}
\label{sec:influence}

We start by establishing a bound on the statistical error of the influence $I_n(z) = - H_n(\theta_n)^{-1} \grad \ell(z, \theta_n)$ of a data point $z$ to the population limit $I(z) = -H(\theta_\star)^{-1}\grad \ell(z, \theta_*)$. 

We give an error bound $\norm{I_n(z) - I(z)}_{H_\star}$ 
in the natural geometry implied by the population Hessian $H_\star := H(\theta_\star)$ at the true parameter $\theta_\star$; here we use the notation $\norm{u}_{A}^2 = \inp{u}{A u}$ for a positive definite matrix $A$. 
The $H_\star$-norm captures the behavior of $I(z)$ and $I_n(z)$ in an \emph{affine-invariant manner}. That is, if we parameterize the problem in terms of $\theta' = A\theta$ for an invertible matrix $A$ so that the loss is $\ell'(z, \theta') = \ell(z, A^{-1}\theta')$, the influence functions $I'$ in this new parameterization satisfies $I'(z) = A \, I(z)$ and similarly for its sample version. Letting $H_\star' := \expect_{Z \sim P} [\grad^2 \ell'(z, \theta_\star')]$ be the (reparameterized) Hessian at the minimizer $\theta_\star' = A \theta_\star$, we can verify that
$\norm{I_n'(z) - I'(z)}_{H_\star'} = \norm{I_n(z) - I(z)}_{H_\star}$, i.e., the error criterion is affine-invariant.

\subsection{Statistical Error Bound}
Our statistical error bound depends on a notion of effective 
dimension of the statistical model. Define the covariance matrix of the gradient as 
\begin{align}
    G(\theta) = \Cov_{Z \sim P}\big(\grad \ell(Z, \theta)\big) \,,
\end{align}
where $\Cov(\xi) = \expect[\xi \xi\T] - \expect[\xi]\,\expect[\xi]\T$ is the covariance matrix of a random vector $\xi$.
We define the \textbf{effective dimension} of this problem as 
\begin{align}
    p_\star = \Tr\left[ H_\star^{-1/2} G_\star H_\star^{-1/2} \right]\,,
\end{align}
where $G_\star := G(\theta_\star)$ is the gradient covariance at $\theta_\star$.

The covariance $G_\star$ has a special meaning for maximum likelihood estimation. Concretely, if the loss $\ell(z, \theta) = -\log P_\theta(z)$ is the negative log likelihood and the statistical model $P_{\theta_\star}$ is well-specified, then $G_\star$ is the information matrix at $\theta_\star$. In this case, we have $G_\star=H_\star$ so that the effective dimension $p_\star$ equals the ambient dimension $p$. 

For misspecified models or for general M-estimation problems beyond maximum likelihood, $G_\star$ and $H_\star$ are distinct in general. The effective dimension $p_\star$ captures the mismatch between the two; it can be much smaller or much larger than $p$. We can have $p_\star \ll p$ when the eigenvalues of $G_\star$ decay faster than those of $H_\star$. Conversely, we get that $p_\star > p$ when 
the eigenvalues of $G_\star$ decay slower than those of $H_\star$.
We refer to \Cref{sec:eigendecay} for precise calculations. Note that regardless of whether $p_\star > p$ or $p_\star <p$, a dependence on $p_\star$ is unavoidable since ${p_\star}/{n}$ is a lower bound on the estimation error \cite{Fortunati2016TheMC}.

\myparagraph{Assumptions}
We make the following assumptions.
\begin{enumerate}[label=(\alph*), nosep]
    \item  \label{asmp1a}
    For any $z\in \Zcal$, the loss function $\ell(z, \cdot)$ is pseudo self-concordant for some $R \ge 1$: 
\[
    \vert D^3_\theta \ell(z, \theta)[u,u,v]\vert \leq R \norm{u}^2_{\grad^2 \ell(z, \theta)} \, \norm{v}_2 \,,
\]
where $D^3_x f(x)[u, u, v] := \frac{\D}{\D t} \inp{u}{\grad^2 f(x + tv) \, u}|_{t=0}$ for $f$ thrice continuously differentiable and where $\norm{\cdot}_2$ denotes the spectral norm for matrices. 
\item \label{asmp1b}
There exists a constant $K_1 \ge 1$ such that the normalized gradient $H_\star^{-1/2} \grad \ell(Z, \theta_\star)$ at $\theta_\star$ is sub-Gaussian with parameter $K_1$. 
\item \label{asmp1c} 
There exists $K_2 \ge 1$ such that the standardized Hessian $H_\star^{-1/2} \, \grad^2 \ell(Z, \theta_\star) \, H_\star^{-1/2} - \id_p$ at $\theta_\star$ satisfies a Bernstein condition with parameter $K_2$ (\Cref{defn:matrix_bernstein} in Appendix \ref{sec:tools}). Moreover, 
    \begin{align*}
        \sigma^2_H := \norm*{ \Var\big(H(\theta_\star)^{-1/2} \, \grad^2 \ell(Z, \theta_\star) \, H(\theta_\star)^{-1/2}\big)}_2
    \end{align*}
is finite, where we denote $\Var(H) = \expect[H H\T] - \expect[H] \, \expect[H]\T$ for a random matrix $H$.
\end{enumerate}

Self-concordance was introduced by \citet{nesterov1994interior} to give an affine-invariant analysis of Newton's method and was adapted by \citet{bach2010self} to apply to logistic regression; we use the latter assumption. 
This assumption prevents $\grad^2 \ell(z, \theta)$ from changing too quickly with $\theta$. 
The most useful consequence of this assumption is a spectral approximation of the Hessian 
$(1/2) H(\theta') \preceq H(\theta) \preceq 2\, H(\theta')$ for $\theta$ and $\theta'$ close enough in terms of the Euclidean distance. 

We make the last two assumptions to argue about the concentration of $\grad \ell(Z, \theta_\star)$ and $\grad^2 \ell(Z, \theta_\star)$ respectively to their expected values for $Z \sim P$. We make appropriate normalizations so that the assumptions are affine invariant, similar to the error criterion.
Since $\expect[\nabla \ell(Z, \theta_\star)] = 0$,
Assumption \ref{asmp1b} gives a high-probability bound on $\norm{\grad \ell(Z, \theta_\star)}_{H_\star^{-1}}$ in the natural $H_\star^{-1}$ norm of the gradient.
Assumption \ref{asmp1c} gives the spectral concentration 
$(1/2) H(\theta) \preceq H_n(\theta) \preceq 2\, H(\theta)$ for a fixed $\theta$ with high probability for $n$ large enough. 

\myparagraph{Example}
The assumptions outlined above hold for all generalized linear models under some regularity conditions.
We give one concrete examples here (more can be found in \cref{sec:techn:glms_examples}). 

\noindent \textit{Logistic Regression:} Let $\Zcal \subset B_{p,M} \times \{\pm 1\}$ for some $M > 0$.
Consider the loss $\ell(z, \theta) = \log\big(1 + \exp(-y \inp{\theta}{x})\big)$ and let $\sigma(z) = \frac{1}{1+e^{-z}}$. Assume that $H(\theta_\star) \succ 0$.
\begin{enumerate}[label=(\alph*), nosep]
    \item Pseudo self-concordance. Note that $\nabla_\theta^2 \ell(z, \theta) = \sigma(\theta^\top x) [1 - \sigma(\theta^\top x)] xx^\top$ and $D_\theta^3 \ell(z, \theta)[u, u, v] = \sigma(\theta^\top x) [1 - \sigma(\theta^\top x)] [1 - 2\sigma(\theta^\top x)] (u^\top x)^2 (v^\top x)$.
    It follows that $\abs{D^3_\theta \ell(z, \theta)[u,u,v]} \le M \norm{v}_2 \norm{u}^2_{\nabla^2 \ell(z,\theta)}$ and thus $\ell$ is pseudo self-concordant with $R \ge M$.
    \item Sub-Gaussian gradient. Note that $\norm{\nabla_\theta \ell(Z, \theta_\star)}_2 = \norm{[1 - \sigma(Y \theta_\star^\top X)] Y X}_2 \le M$.
    Therefore, the normalized gradient $H(\theta_\star)^{-1/2} \nabla \ell(Z,\theta_\star)$ is sub-Gaussian (cf. \Cref{lem:bounded_subG} from \Cref{sec:tools}).
    \item Bernstein Hessian. Note that $\norm{\nabla_\theta^2 \ell(Z, \theta_\star)}_2 \le \norm{XX^\top}_2 / 4 \le M^2/4$. It follows that the standardized Hessian $H(\theta_\star)^{-1/2} \nabla_\theta^2 \ell(Z, \theta_\star) H(\theta_\star)^{-1/2} - I_p$ satisfies the matrix Bernstein condition (cf. \Cref{lem:bounded_bernstein}  from \Cref{sec:tools}).
\end{enumerate}

\myparagraph{Statistical Error Bound} Below and throughout, we omit absolute constants.  
\begin{theorem} \label{thm:if-stat-main}
Suppose the assumptions above hold and
\begin{align*}
    n \ge C_{K_1, K_2, \sigma_H} \left( \frac{R^2 p_\star}{\mu_\star} \log\frac{1}{\delta} + \log\frac{p}{\delta} \right) \,,
\end{align*}
where $\mu_\star = \lambda_{\min}(H_\star)$ and $C_{K_1, K_2, \sigma_H}$ is a constant depending on $K_1, K_2,$ and $\sigma_H$. 
Then, with probability at least $1-\delta$,
we have 
$\frac14 H_\star \preceq H_n(\est) \preceq 3 H_\star$ and 
\begin{align*}
    \norm{I_n(z) - I(z)}_{H_\star}^2 
    \le C_{K_1, K_2, \sigma_H}\, \frac{R^2 p_\star^2}{\mu_\star n} \log^3 \left(\frac{p}{\delta}
    \right) \, .
\end{align*}
\end{theorem}

\myparagraph{Remark} In this result, we view $z$ as a random element following the data distribution $P$. The quantities $\norm{\nabla \ell(z, \theta_\star)}_{H_\star^{-1}}$ and $\norm{H_\star^{-1/2} H(z, \theta_\star) H_\star^{-1/2}}_2$ are controlled using the sub-Gaussian gradient and matrix Bernstein assumptions. A similar result holds if we treat $z$ as a fixed datapoint, since these quantities are now fixed as well.

\Cref{thm:if-stat-main} has several merits.
First, it is adapted to the eigenspectrum of $G_\star$ and $H_\star$ via the effective dimension $p_\star$; the bound only has a logarithmic dependence on the ambient dimension $p$.
The effective dimension $p_\star$ is also affine-invariant, similar to the error criterion. The only geometry-dependent (i.e., not affine-invariant) term in \Cref{thm:if-stat-main} is the minimal eigenvalue $\mu_\star$ of the Hessian $H_\star$. Third, we get a fast $1 / n$ rate, faster than the $1/\sqrt{n}$ rate typical of uniform convergence arguments. 

We now sketch the key aspects of its proof. The full proof is given in \Cref{appx:proof_influence_est}. 
\begin{proof}[Proof Sketch of \Cref{thm:if-stat-main}]
We use the triangle inequality to bound $ \norm{I_n(z) - I(z)}_{H_\star}$ by
\begin{align*}
    & \, \norm*{\big(H_n(\est)^{-1} - H_\star^{-1}\big) \big(\grad \ell(z, \est) - \grad \ell(z, \theta_\star) \big)}_{H_\star} \\ &+ 
    \norm*{\big(H_n(\est)^{-1} - H_\star^{-1}\big){\grad \ell(z, \theta_\star)}}_{H_\star} 
    \\&+ \norm*{H_\star^{-1}\big(\grad \ell(z, \est) - \grad \ell(z, \theta_\star)\big)}_{H_\star}.
\end{align*}
The proof follows from arguing that $\theta_n \to \theta_\star$, $\grad \ell (z, \theta_n) \to \grad \ell(z, \theta_\star)$, and $H_n(\theta_n) \to H_\star$ in the appropriate sense.
The first comes from a localization result of \citet{ostrovskii2021finite} that states that $\theta_n$ lies in a Dikin ellipsoid of radius $\sqrt{p_\star / n}$ around $\theta_\star$ for $n$ large enough, i.e.,  $\norm{\theta_n - \theta_\star}_{H_\star}^2 \lesssim p_\star / n$. The second comes from arguing using pseudo self-concordance that the gradient $\grad \ell(z, \cdot)$ is Lipschitz w.r.t.~$\norm{\cdot}_{H_\star}$ in the Dikin Ellipsoid around $\theta_\star$. 
For the last one, we argue that $H_n(\theta_n) \approx H_n(\theta_\star)$
from pseudo self-concordance, and formalize $H_n(\theta_\star) \to H_\star$ by matrix concentration. 
\end{proof}

In addition to the statistical error bound in \Cref{thm:if-stat-main}, we also provide a bound for the approximation error in \eqref{eq:linearization-estimate}.
Here, we treat $z$ as a fixed data point and make the following boundedness assumptions in addition to the assumptions above.
\begin{enumerate}[label=(\alph*), nosep]
 \setcounter{enumi}{3}
 \item \label{asmp2b} The normalized gradient is bounded in a neighborhood of $\theta_\star$, i.e., there exist $M_1 \ge 1, \rho \in (0, R^{-1}]$ such that $\mynorm{\nabla \ell(z, \theta)}_{H_\star^{-1}} \le M_1$ for all $\mynorm{\theta - \theta_\star}_{H_\star} \le \rho$.
\item \label{asmp2c}The normalized Hessian is bounded in a neighborhood of $\theta_\star$, i.e., there exist $M_2 \ge 1, \rho \in (0, R^{-1}]$ such that $\mynorm{H(z, \theta)}_{H_\star^{-1}} \le M_2$ for all $\mynorm{\theta - \theta_\star}_{H_\star} \le \rho$.
\end{enumerate}

\begin{theorem}\label{thm:linearization_paper}
    Suppose that the assumptions above hold, $\varepsilon \le C \min\{\rho/M_1, 1/M_2, \sqrt{\mu_\star}/RM_1\}$, and
    \begin{align*}
        n \ge C_{K_1, K_2, \sigma_H} \left[ \left( \frac{R^2}{\mu_\star} + \frac1\rho \right) p_\star \log{\frac1\delta} + \log{\frac{p}{\delta}} \right].
    \end{align*}
    Then, with probability at least $1 - \delta$,
    \begin{align*}
        &\mynorm{\frac{\theta_{n,\varepsilon,z} - \theta_n}{\varepsilon} - I_n(z)}_{H_n(\theta_n)}^2 \le C_{M_1, M_2} \times \\
        &\;\; \left\{\left[ \exp\left( \frac{C_{K_1, M_1} R}{\sqrt{\mu_\star}}\left( \sqrt{\frac{p_\star}{n} \log{\frac{1}{\delta}}} + \varepsilon \right) \right)- 1 \right]^2 + \varepsilon^2 \right\}.
    \end{align*}
\end{theorem}
A full proof can be found in \Cref{sec:linearization}.

\begin{table*}[t]
\caption{The number of calls to a {Hessian-vector product oracle} $u\mapsto \grad^2 \ell(z, \theta)u$ so that (a) the 
computational error is at most $\eps$, and 
(b) the total error is 
at most $\eps$ in the sense of \Cref{prop:total_error}. 
We show the dependence of the former on the condition number $\kappa_n = L / \lambda_{\min}(H_n(\theta_n))$, the optimal magnitude $\Delta_n = \|I_n(z)\|_{H_n(\theta_n)}^2$, and the SGD noise $\sigma_n^2$, defined in \Cref{sec:comp:sgd}. 
The total error bound depends on the corresponding population quantities $\kappa_\star = L / \lambda_{\min}(H_\star)$, $\Delta_\star = \|I(z)\|_{H_\star}^2$, and $\sigma_\star^2$, as well the effective dimension $p_\star$. 
We omit the dependence on problem constants $R, L, K_1, K_2, \sigma_H^2$, as well as logarithmic terms in $p, p_\star, \delta$.
For the low-rank approximation, we assume that the total complexity to obtain a rank-$k$ approximation is $O(k)$ full batch Hessian-vector products. We present computational error bounds assuming the eigenvalues $\lambda_i(H_n(\theta_n))$ of $H_n(\theta_n)$ decay polynomially as $i^{-\beta}$ ($\beta > 1$)  or exponentially as $e^{-\nu i}$ ($\nu > 0$). The same decay is assumed for $H_\star$ for the total error bound. 
The full proofs of these bounds are given in \Cref{appx:computation_bound}.
}
\label{tab:total-cost}
\setlength{\tabcolsep}{5pt} %
\begin{center}
{\renewcommand{\arraystretch}{1.4}%
\begin{tabular}{lccc}
\toprule
\textbf{Method}  & 
\textbf{Computational Error}
&
\textbf{Total Error} &
\textbf{Reference}
\\
\midrule
Conjugate Gradient 
& 
$n \sqrt{\kappa_n} \log \frac{\Delta_n}{\eps}$
& 
$\frac{\kappa_\star^{3/2} p_\star^2}{\eps} \log \frac{\Delta_\star}{\eps}$
& \Cref{cor:total:cg}
\\
SGD &
$\frac{\sigma_n^2}{\eps} + \kappa_n \log \frac{\kappa_n \Delta_n}{\eps}$ &
$\frac{\sigma_\star^2}{\eps} + \kappa_\star \log\frac{\kappa_\star \Delta_\star}{\eps}$
& \Cref{cor:sgd:total-error}
\\ 
SVRG & 
$(n + \kappa_n) \log \frac{\kappa_n \Delta_n}{\eps}$ 
& 
$\kappa_\star\left( 1 + \frac{p_\star^2}{ \eps}\right) \log \frac{\kappa_\star \Delta_\star}{\eps}$ 
& \Cref{cor:total:svrg}
\\
Accelerated SVRG &
$(n + \sqrt{n \kappa_n}) \log \frac{\kappa_n \Delta_n}{\eps}$  & 
$\kappa_\star\left( \sqrt{\frac{p_\star^2}{ \eps}} + \frac{p_\star^2}{\eps}\right) \log \frac{\kappa_\star \Delta_\star}{\eps}$ 
& \Cref{cor:total:svrg}
\\
Low-Rank Approx. ($\lambda_i \propto i^{-\beta}$)
& 
$
n\left(\frac{\kappa_n \Delta_n}{\eps}\right)^{\frac{1}{\beta-1}}
$
& 
$
\left( \frac{\kappa_\star}{\eps} \right)^{\frac{\beta}{\beta -1}} p_\star^2 \Delta_\star^{\frac{1}{\beta-1}}
$ 
&  \Cref{cor:total:low-rank}
\\
Low-Rank Approx. ($\lambda_i \propto e^{-\nu i}$) & 
$
\frac{n}{\nu} \log \frac{\kappa_n \Delta_n}{\eps}
$
& 
$
\frac{\kappa_\star p_\star^2}{\nu \eps} \log \frac{\kappa_\star \Delta_\star}{\eps} 
$
& \Cref{cor:total:low-rank}
\\
\bottomrule 
\end{tabular}
}
\end{center}
\end{table*}

\subsection{Computational and Total Error Bounds}
\label{sec:total-err-main}

We consider iterative first-order algorithms to compute the influence function $I_n(z) = \argmin_{u} g_n(u)$ by minimizing the convex quadratic $g_n(u)$ defined in \eqref{eq:quadratic}. 

We aim to find an $\eps$-approximate minimizer $u$ that satisfies $\expect[\norm{u - I_n(z)}_{H_n(\theta_n)}^2 | Z_{1:n}]\le \eps$. This error criterion is not only affine-invariant, but is also equivalent to $\expect[g_n(u) - \min g_n | Z_{1:n}] \le 2\eps$. Throughout this section, we assume for all $z \in \Zcal$ that $\ell(z, \cdot)$ is $L$-smooth, i.e., $\norm{\grad^2 \ell(z, \theta)}_2 \le L$ for all $\theta$. The complexity of minimizing $g_n$ with first order algorithms depends on the condition number $\kappa_n := L / \lambda_{\min}\big(H_n(\theta_n)\big)$. The corresponding condition number of the population Hessian $H_\star$ is $\kappa_\star := L / \lambda_{\min}(H_\star) = L/\mu_\star$. 

Any $\eps$-approximate minimizer $\hat I_n(z)$ of $g_n$ satisfies the following total error bound. 
\begin{proposition} \label{prop:total_error}
    Consider the setting of \Cref{thm:if-stat-main}, and let $\Gcal$ denote the event under which its conclusions hold. 
    Let $\hat I_n(\theta)$ be an estimate of $I_n(\theta)$ that satisfies $\expect\left[\norm{\hat I_n(z) - I_n(z)}_{H_n(\theta_n)}^2 | Z_{1:n}\right] \le \eps$. 
    Then,
    \[
        \expect\left[ \norm{\hat I_n(z) - I(z)}_{H_\star}^2 \, \middle| \, \Gcal \right]
        \le 8 \eps + 
         C \frac{R^2 p_\star^2}{\mu_\star n} \log^3\frac{p}{\delta}  \,,
    \]
    where $C = C_{K_1, K_2, \sigma_H}$ is as in \Cref{thm:if-stat-main}.
\end{proposition}
This bound is obtained by translating the approximation error in the $H_n(\theta_n)$-norm to the $H_\star$-norm using the spectral Hessian approximation under $\Gcal$ and the triangle inequality. 

The conjugate gradient method is known to require $T_n(\eps) := \sqrt{\kappa_n} \,  \log\left(\norm{I_n(z)}_{H_n(\theta_n)}^2 / \eps \right)$ iterations (ignoring constants) to return an $\eps$-approximate minimizer~\cite[e.g.][]{saad2003iterative,chen2005matrix,bai2021matrix}. 
Since each iteration requires $n$ Hessian-vector products, the total computational complexity to obtain an $\eps$-approximate minimizer is $O\big(n \, T_n(\eps)\big)$. To make the statistical error $\norm{I_n(z) - I(z)}_{H_\star}^2$ to be smaller than $\eps$, we must choose $n \ge n(\eps) = \tilde O\big(R^2 p_\star^2 / (\mu_\star \eps)\big)$ (ignoring constants and logarithmic factors). 
\Cref{prop:total_error} now says that the overall computational complexity to reduce the total error under $O(\eps)$ is $O\big(n(\eps) T(\eps)\big)$. 

\Cref{tab:total-cost}
presents this bound with sample-dependent quantities such as $\kappa_n$ and $\norm{I_n(z)}_{H_n(\theta_n)}$ translated to their population versions.   
\Cref{tab:total-cost} also lists the corresponding bounds for the other algorithms we consider. We discuss the implications of the total error bounds. We use $\tilde O(\cdot)$ to suppress logarithmic terms in $1/\epsilon$ below. 

\myparagraph{Marginal Benefits of Variance Reduction}
For a fixed $n$, the computational error bounds agree with the conventional wisdom that SVRG is significantly faster than SGD, especially for small $\eps$. Indeed, the error $\tilde O (n+\kappa_n)$ of SVRG only depends logarithmically on $1/\eps$, while the SGD error $\tilde O(\sigma_n^2 / \eps + \kappa_n)$ is polynomial. 
However, the statistical error bounds suggest that the sample size must be $n = \tilde O(R^2 p_\star^2 / \mu_\star \eps)$, so the total error of SVRG scales as $1/\eps$. This matches SGD up to constants. 
SVRG has better constants only if the SGD noise $\sigma_\star^2 > p_\star^2 / \mu_\star$ is large. 

\myparagraph{Marginal Benefits of Acceleration}
For fixed $n$, accelerated SVRG's rate of $\tilde O(n + \sqrt{n \kappa_n})$ is faster than SVRG for ill-conditioned problems where $\kappa_n > n$, but is no worse for well-conditioned problems where $\kappa_n \le n$. To have a small total error, we need $n = \tilde O(1/ \varepsilon)$, while the condition numbers satisfy $\kappa_n\leq4 \kappa_\star$ for $\kappa_n$ a constant (under~\Cref{thm:if-stat-main}). Thus, for $\eps$ small, the problem is well-conditioned and acceleration offers marginal benefits. 

\myparagraph{Stochastic Methods Outperform Full Batch Methods}
The total error of the conjugate gradient method is $\tilde O(\kappa_\star^{3/2} p_\star^2  / \eps)$ while SVRG is $\tilde O(\kappa_\star p_\star^2 / \eps)$. Thus, SVRG always has better constants than the conjugate gradient method. This is also true of accelerated SVRG. 

\myparagraph{Low-rank Approximations Work for Faster Eigendecay}
For a slow polynomial decay $\lambda_i(H_\star) \propto i^{-\beta}$ of the eigenvalues of $H_\star$ for $\beta > 1$, the total error scales as $\eps^{-\beta / (\beta-1)}$, which is worse than the $1/\eps$ rate for all other methods considered.
However, for a faster exponential decay $\lambda_i(H_\star) \propto e^{-\nu i}$ for $\nu > 0$, its $1/\eps$ rate matches SVRG exactly up to a factor of $\nu$, despite being a full batch method.  %
\section{MOST INFLUENTIAL DATA SUBSETS}\label{sec:mis}

We now turn to the subset influence defined in \Cref{sec:bg}. We start by formalizing the population 
limit and then establish statistical error bounds. 
Let $h: \Theta \to \reals$ be a continuously differentiable test function and $\alpha \in (0, 1)$ be fixed throughout. We only consider $n$ where $\alpha n$ is an integer.

\myparagraph{Population Limit}
In order to derive the population limit of the subset influence $I_{\alpha, n}(h)$ from \eqref{eq:sif}, we interpret the weights $w \in W_\alpha \subset \Delta^{n-1}$ as a probability distribution over the $n$ datapoints. This gives
\[
    I_{\alpha, n}(h) = \max_{w \in W_\alpha} \expect_{i \sim w}[\phi_n(Z_i)], 
\] 
where $\phi_n(z) = - \inp{\grad h(\theta_n)}{H_n(\theta_n)^{-1} \grad \ell(z, \theta_n)}$.
This suggests that the population limit should be $ \sup_{Q \in \Qcal} \expect_{Z \sim Q} [\phi(Z)]$ 
over an appropriate set of distributions $\Qcal$, where
$\phi(z) = - \inp{\grad h(\theta_\star)}{H_\star^{-1} \grad \ell(z, \theta_\star)}$.

Since the maximum of a linear program occurs at a corner, we can pass from the max over $W_\alpha$ to its convex hull 
\[
    \conv W_\alpha = \left\{ w \in \Delta_{n-1} \, :\, w_i(1-\alpha)n \le 1  \,\,\, \forall \, i \, \right\} \,.
\]
Compared to the uniform distribution $\ones_n/n$ over $Z_{1:n}$, $w \in \conv W_\alpha$ allows for weights that are a factor of $(1-\alpha)^{-1}$ larger.
If $P$ is a continuous distribution with density $f_P$, then a natural choice for $\Qcal$ is the set of distributions with density $f_Q(z) \le f_P(z) / (1-\alpha)$. 

We can formalize this discussion through the notion of a tail statistic known as the \emph{superquantile} or the \emph{conditional value at risk}~\cite{rockafellar2000optimization}. The superquantile of a random variable $Z\sim P$ at level $\alpha$ is defined as 
\[
    \supq_\alpha(Z) := \sup\left\{ \expect_{Z \sim Q}[Z] \, :\, \frac{\D Q}{\D P} \le \frac{1}{1-\alpha} \right\} \,,
\]
where $\D Q / \D P$ denotes the Radon-Nikodym derivative of $Q$ w.r.t.~$P$. This constraint subsumes both the density ratio constraint in the continuous case and the weight ratio constraint in the discrete case.
The superquantile has a long and storied history in economics and quantitative finance, with recent applications in machine learning; we refer to \cite{laguel2021superquantiles} for a survey. 
We overload notation to denote the superquantile of the empirical measure over $v_1, \ldots, v_n$ as $\supq_\alpha(v_1, \ldots, v_n)$. 

We formalize the connection between the maximum subset influence $\SIF_{n, \alpha}$ and the superquantile. 
\begin{proposition} \label{prop:sif-sq}
If $\alpha n$ is an integer, then
    $\SIF_{n, \alpha}(h) = \supq_\alpha(v_1, \ldots, v_n)$ where $v_i = - \inp*{\grad h(\theta_n)}{H_n(\theta_n)^{-1} \grad \ell(Z_i, \theta_n)}$.
\end{proposition}

\Cref{prop:sif-sq} motivates us to define the \textbf{population subset influence} as 
\begin{align}
I_\alpha(h) = \supq_\alpha \Big[ - \grad h(\theta_\star)\T {H(\theta_*)^{-1} \grad \ell(Z, \theta_\star)} \Big]  \,.
\end{align}

\myparagraph{Assumptions}
We need to use the strengthen assumptions made in \Cref{thm:linearization_paper} for technical reasons. We also add the following

\begin{enumerate}[label=(\alph*), nosep]
    \setcounter{enumi}{5}
    \item  \label{asmp2d}
    The test function $h$ is bounded as $\norm{\nabla h(\theta)}_{H_\star^{-1}} \le M_1'$ and $\norm{H_\star^{-1/2}\nabla^2 h(\theta)H_\star^{-1/2}}_{2} \le M_2'$ for all $\norm{\theta - \theta_\star}_{H_\star} \le \rho$.
\end{enumerate}
Assumption~\ref{asmp2d} asserts the boundedness of the test function $h$.
We make this assumption in a neighborhood around $\theta_\star$. 

\myparagraph{Statistical Bound}
We now state our main bound. 

\begin{theorem} \label{thm:subset:main}
    Suppose the assumptions above hold and the sample size $n$ satisfies 
    the condition in \Cref{thm:if-stat-main}.
    Then, with probability at least $1-\delta$, we have
    \[
        \big( I_{\alpha, n}(h) - I_\alpha(h)  \big)^2 
        \le 
        \frac{C_{M_1, M_2, M_1', M_2'}}{(1-\alpha)^2}
        \frac{R^2 p_\star}{\mu_\star n} \log \frac{n \vee  p}{\delta} \,.
    \]
\end{theorem}
\Cref{thm:subset:main} has the same merits as \Cref{thm:if-stat-main}: it uses the effective dimension $p_\star$ and exhibits only a logarithmic dependence on the ambient dimension $p$.
We square the left side so that it scales for $\alpha \to 0$ as the squared norm $\norm{(1/n) \sum_{i=1}^n I_n(Z_i) - \expect_{Z \sim P}[I(Z)]}_{H_\star}^2$, comparable to \Cref{thm:if-stat-main}. We get a fast $\log n / n$ rate rather than a slow $1/\sqrt{n}$ rate. 

The proof relies on the equivalent expression
\[
    \supq_\alpha(Z) = \inf_{\eta \in \reals} \left\{
    \Phi(Z, \eta) := 
    \eta + \frac{1}{1-\alpha} \expect(Z - \eta)_+ \right\}
\]
of the superquantile
where $(\cdot)_+ = \max\{\cdot, 0\}$. 
We analyze the convergence of
$\Phi\big(\phi_n(Z_{1:n}), \eta\big)$ to $\Phi(\phi(Z), \eta)$ for fixed $\eta$ using the same techniques as \Cref{thm:if-stat-main}. Then, we construct an $\eps$-net so the bound holds for all $\eta$, including the minimizer. The full proof is given in \Cref{sec:a:subset-if}.

\paragraph{Related Work.}
\label{sec:related}

Influence functions or curves have originally been proposed by~\citet{hampel1974influence}, and partly motivated by \citet{jaeckel1972infinitesimal}'s ``infinitesimal jackknife''. \citet{cook} showed that the influence function can be computed using inverse Hessian gradient products. Recent works on influence functions include~\cite{cook1986assessment,hadi1995unifying,zhu2004diagnostic,ma2014statistical,zhao2019multiple}. 
The theoretical statistical analysis has mostly focused on large-sample asymptotics hence in small dimensions, and we refer to the recent work~\cite{marco} for a comprehensive survey. 

Efficiently computing influence functions, or related inverse-Hessian-vector products, has received attention recently in the context of the training of deep neural networks using natural gradient or Newton-like algorithms~\cite{curveball}. Specifically, on influence functions, stochastic convex optimization algorithms~\cite{lissa}, conjugate gradient methods~\cite{saad2003iterative}, and low-rank variants~\cite{arnoldi} have been applied. The recent discovery of linear convergence for variance-reduced  optimization algorithms makes them potentially competitive for the efficient computation of influence functions.  %
\section{EXPERIMENTS}
\label{sec:expt}

\begin{figure*}
    \centering
\includegraphics[width=1\linewidth]{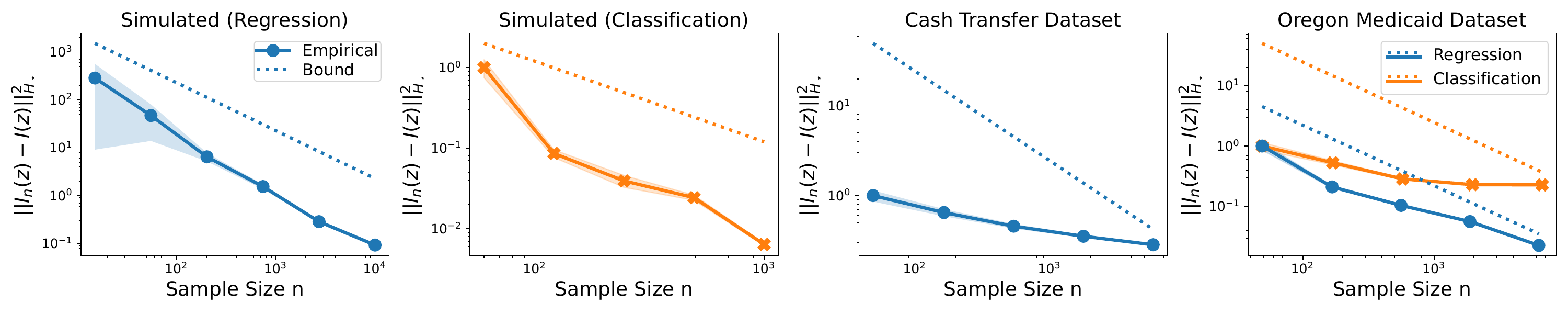}
    \caption{Convergence of the empirical influence function to the population (solid line) compared to the bound of \Cref{thm:if-stat-main} (dotted line) with linear regression and classification models for simulated (left two) and real data (right two). 
    We plot the mean over $100$ repetitions, and the shaded area denotes the $95\%$ standard error.
    }
    \label{fig:sim_realdata}
\end{figure*}

We explore the convergence of the empirical influence function to its population counterpart for classical linear models. We also report the findings from large attention based models, for which little statistical theory is known, yet maximum influential subsets can still be computed as for any black-box model. 
\Cref{appx:exp_details} contains the full details of this section.  The code as well as the scripts to
reproduce the experiments are made publicly available online~\textsf{{https://github.com/jfisher52/influence\_theory}}.

\subsection{Linear Models}
We consider synthetic ridge regression and binary logistic regression in $\reals^9$. The input $x \sim \Ncal(0, \id)$ is normal, and the outputs are generated with a linear or logistic model from i.i.d.~noise based on a fixed $\theta_\star$. We also consider two real datasets: (1) Oregon Medicaid~\cite{oregon_exp}, where the goal is to estimate the overall health (classification) and the number of good health days in the last month (regression) of an individual, and (2) Cash Transfer~\cite{cash_exp}, where the goal is to estimate the total consumption of an individual (regression). 
Both datasets use some economic and demographic features and treatments as inputs to the model; they contain 20K and 50K points respectively. 

We plot the statistical convergence of the exact empirical influence $I_n(z)$ to the population influence $I(z)$ for fixed $z$ using various sample sizes $n$ as well as the bound of \Cref{thm:if-stat-main}. For the real data, we use the full dataset as the population. 
We measure the influence of points $z$ that are outliers added to the training set for the simulations and a random sample for the real data. 

\myparagraph{Results: Tightness of \Cref{thm:if-stat-main}}
The results are given in \Cref{fig:sim_realdata}.
We see for the simulated datasets (left two plots) that the empirical observations for a straight line in log-log scale whose slope matches that of the bound. This indicates that the $1/n$ rate of our bounds is also observed empirically.\footnote{
A log-log plot of $y=cx^a$ is a straight line with slope $a$.
}
This is also approximately true for the regression line in the Oregon Medicaid dataset. We note that its classification line and the Cash Transfer dataset have slopes that differ from the bound. 
This phenomenon could be due to the error in the population influence used for the plots: we approximate it from a larger data sample because we do not have access to the population distribution. Note that we do not see such a behavior in the simulated classification task, where we can more accurately approximate the population.
In all of these cases, \Cref{thm:if-stat-main} is still an upper bound on the empirical error. 

\begin{figure*}%
    \centering
    {\includegraphics[width=1\linewidth]{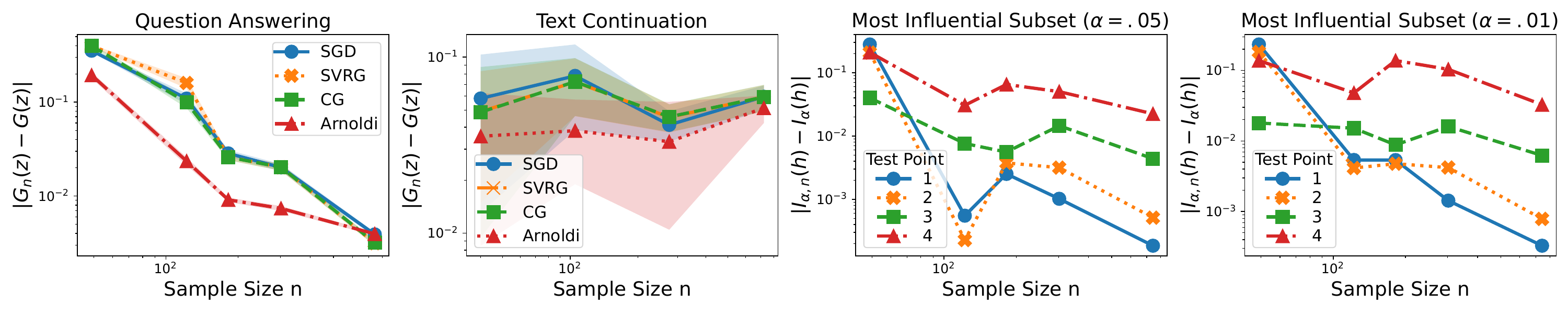} }%
    \caption{
    \textbf{Left two}: Convergence of the approximate empirical influence to the population for text generation tasks measured in terms of predictions as in \eqref{eq:pred-influence}. 
    The solid line denotes the mean of $|G_n(z_i) - G(z_i)|$ for $i=1,\cdots, 4$ and the shaded area denotes its standard deviation.
    \textbf{Right two}: Convergence of the influence value $I_{\alpha, n}(h_i)$ found by the most influential subset method to its population version $I_\alpha(h_i)$ on the question-answering task for different test functions $h_i = \ell(z_{\text{test}, i}, \theta)$. 
    }
    \label{fig:nlp-data}
\end{figure*}

\subsection{Large Transformer Language Models}

\myparagraph{Setup}
We consider (a) a question-answering task where the goal is to respond to a natural language question with a factually correct answer, and (b) a text continuation task where the goal is to generate ten tokens following a given context. 
We use a BART-base model~\cite{lewis2020bart} on the zsRE dataset~\cite{zsre} and a DistilGPT-2 model~\cite{sanh2019distilbert} on the WikiText-103 dataset~\cite{merity2017pointer} respectively.
We subsample the training set size for various $n$ and finetune a pretrained model to get $\theta_n$. 
We take the largest value of $n$ as the population version: this value was 5K and 2K respectively. 
We estimate the population influence with 100 epochs of SVRG, while we use 50 passes through the data for the approximate methods.
We compute the influence $I_n(z)$ for 5 points $z_1, \cdots, z_5$.  The quadratic $g_n$ from \eqref{eq:quadratic} is nonconvex and unbounded below if the Hessian $H_n(\theta_n)$ is not positive semidefinite; we find this to be the case for our experiments with the deep nets. To overcome this, we consider 
\[
    I_{n, \lambda}(z) = - \big(H_n(\theta_n) + \lambda \id\big)^{-1} \grad \ell(z, \theta_n) \,.
\]
We choose the smallest $\lambda$ so that the quadratic objective $g_n(u_t)$ from \eqref{eq:quadratic} is bounded below for iterates $u_t$ obtained from SGD, ensuring that $H+\lambda\id$ is positive semidefinite. 

\myparagraph{Error Criterion}
The norm $\norm{\hat I_n(z) - I(z)}$ bound may be vacuous for failing to capture the permutation symmetries of the parameters of a deep network. Instead, we measure the effect of a point $z$ on a test function $h(\theta) = \ell(z_{\text{test}}, \theta)$ as
\begin{align} \label{eq:pred-influence}
    G_n(z) = \inp{\grad h(\theta_n)}{I_{n, \lambda}(z)} \,,
\end{align}
and compare it against its population counterpart $G(z)$. From the chain rule, it follows that $G(z)$ is the linearization $\frac{\D}{\D \eps} h(\theta_{n, \eps, z})|_{\eps = 0}$ similar to \eqref{eq:linearization-estimate}. 
In our experiments, $h(\theta)$ is the loss on the test set. 
The results are given in \Cref{fig:nlp-data}. 

\myparagraph{Results: Total Error Versus $n$}
For the question-answering task, the error reduces by a factor of $10$ as $n$ increases from $40$ to $300$ (slope $\approx -1.5$) indicating an empirical $n^{-1.5}$ rate. For the text continuation task, we find that the error in influence estimation does not vary significantly with $n$ and has a high variance. Indeed, the open-ended nature of the text continuation task suggests that no one point $z$ should have a large influence on the predictions of a test point $z_{\text{test}}$, leading to noisy influence estimates. 

\myparagraph{Comparing Computational Approximations}
We observe that SGD $\approx$ SVRG in \Cref{fig:nlp-data}. This corroborates the
total error bounds of \Cref{tab:total-cost} which show that variance-reduced SVRG has the same total error as SGD despite being significantly faster in optimization. 
At a large computation budget, we find that the conjugate gradient method also exhibits an error comparable to SGD and SVRG. The benefits of stochastic algorithms such as SGD become evident for large datasets where SGD gives a reasonable estimate without even making a full pass (its error is independent of $n$, cf. \Cref{tab:total-cost}). 
For the question-answering dataset, we find that the low-rank approximation provided by the Arnoldi method \cite{arnoldi} has the smallest error for $n \le 200$, while it is identical to the others for large $n$.

\myparagraph{Most Influential Subsets}
We repeat the question-answering experiment to find the most influential subset of data for different $n$ with test function $h_i(\theta) = \ell(z_{\text{test}, i}, \theta)$ for four chosen test points. We use the low-rank (Arnoldi) method to approximate the inverse Hessian-vector product because this method has the best error properties in \Cref{fig:nlp-data} (left two).
For different values of $\alpha$, we observe that the estimation error tends to decrease with $n$. We note that a few outliers are to be expected with large-scale deep nets with real data where theoretical assumptions are not precisely met.

The type of influential examples recovered varied from surface-level attributes to deeper features, such as topics, as $n$ increased; see~\Cref{fig:MIS_visual} for examples. 
In some cases, the most influential examples were semantically related questions with different answers. For instance, for the test question "Was Goldmoon male or female" (female), a highly negatively influential questions was "What is the gender of Jacques Rivard?" (male). However, for others the relations seemed more structural. For example, the test question
"The nationality of Jean-Louis Laya was what?" (French), we recovered as highly negatively influential, "The nationality of Yitzhak Rabin is?" (Hebrew). 

\label{sec:subset-if}
\begin{figure}
    \centering
    \includegraphics[scale = .2]{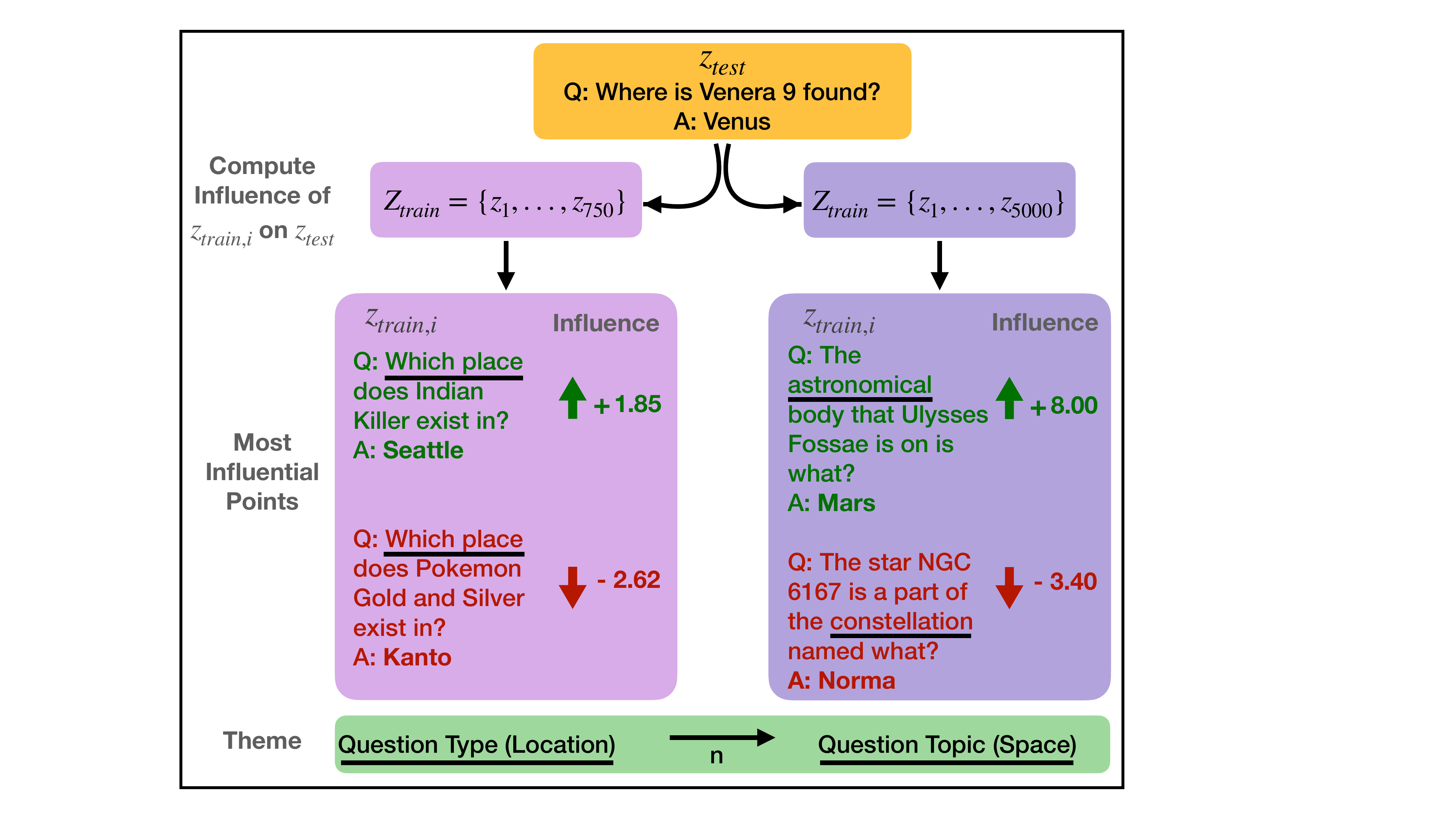}
    \caption{As the sample size $n$ increases, we see a shift in the quality of the most influential questions. Lower $n$ results in surface-level attributes, such as question type, while larger $n$ results in deeper features, such as the topic.}
    \label{fig:MIS_visual}
\end{figure} %
\section{CONCLUSION}
\label{sec:disc}
As statistical learning models and deep nets are being increasingly used, influence diagnostics are precious tools to study the influence of datapoints on predictions, decisions, and outcomes. In this paper, we presented statistical and computational guarantees for influence functions for generalized linear models. We established the statistical consistency of most influential subsets method \citep{broderick2020automatic} together with nonasymptotic bounds. 
We illustrated our results on simulated and real datasets. Extending our results to sparse regularized models as well as deep neural network models are interesting venues for future work.  
\paragraph{Acknowledgements}
This work was supported by NSF DMS-2023166, CCF-2019844, DMS-2052239, DMS-2134012, DMS-2133244, NIH, CIFAR-LMB, and research awards. Part of this work was done while Z. Harchaoui was visiting the Simons Institute for the Theory of Computing, and while K. Pillutla was at the University of Washington. 
\bibliographystyle{abbrvnat}

\bibliography{custom}

\begin{thebibliography}{51}
\providecommand{\natexlab}[1]{#1}
\providecommand{\url}[1]{\texttt{#1}}
\expandafter\ifx\csname urlstyle\endcsname\relax
  \providecommand{\doi}[1]{doi: #1}\else
  \providecommand{\doi}{doi: \begingroup \urlstyle{rm}\Url}\fi

\bibitem[Agarwal et~al.(2017)Agarwal, Bullins, and Hazan]{lissa}
N.~Agarwal, B.~Bullins, and E.~Hazan.
\newblock {Second-Order Stochastic Optimization for Machine Learning in Linear
  Time}.
\newblock \emph{Journal of Machine Learning Research}, 18:\penalty0
  116:1--116:40, 2017.

\bibitem[Allen-Zhu(2017)]{allen2017katyusha}
Z.~Allen-Zhu.
\newblock {Katyusha: The first direct acceleration of stochastic gradient
  methods}.
\newblock \emph{The Journal of Machine Learning Research}, 18\penalty0
  (1):\penalty0 8194--8244, 2017.

\bibitem[Angelucci and De~Giorgi(2009)]{cash_exp}
M.~Angelucci and G.~De~Giorgi.
\newblock Indirect effects of an aid program: How do cash transfers affect
  ineligibles' consumption?
\newblock \emph{American Economic Review}, 99\penalty0 (1):\penalty0 486--508,
  March 2009.

\bibitem[Arnoldi(1951)]{arnoldi_it}
W.~E. Arnoldi.
\newblock The principle of minimized iterations in the solution of the matrix
  eigenvalue problem.
\newblock \emph{Quarterly of applied mathematics}, 9\penalty0 (1):\penalty0
  17--29, 1951.

\bibitem[Avella-Medina(2017)]{marco}
M.~Avella-Medina.
\newblock {Influence functions for penalized {M}-estimators}.
\newblock \emph{Bernoulli}, 23\penalty0 (4B):\penalty0 3178 -- 3196, 2017.
\newblock \doi{10.3150/16-BEJ841}.

\bibitem[Axelsson and Kaporin(2000)]{axelsson2000sublinear}
O.~Axelsson and I.~Kaporin.
\newblock On the sublinear and superlinear rate of convergence of conjugate
  gradient methods.
\newblock \emph{Numerical Algorithms}, 25\penalty0 (1), 2000.

\bibitem[Bach(2010)]{bach2010self}
F.~Bach.
\newblock Self-concordant analysis for logistic regression.
\newblock \emph{Electronic Journal of Statistics}, 4, 2010.

\bibitem[Bach(2021)]{bach2021learning}
F.~Bach.
\newblock \emph{Learning Theory from First Principles}.
\newblock Online version, 2021.

\bibitem[Bai and Pan(2021)]{bai2021matrix}
Z.-Z. Bai and J.-Y. Pan.
\newblock \emph{Matrix analysis and computations}.
\newblock SIAM, 2021.

\bibitem[Bertsekas(2015)]{bertsekas2015convex}
D.~Bertsekas.
\newblock \emph{Convex optimization algorithms}.
\newblock Athena Scientific, 2015.

\bibitem[Broderick et~al.(2020)Broderick, Giordano, and
  Meager]{broderick2020automatic}
T.~Broderick, R.~Giordano, and R.~Meager.
\newblock {An Automatic Finite-Sample Robustness Metric: When Can Dropping a
  Little Data Make a Big Difference?}
\newblock \emph{arXiv Preprint}, 2020.

\bibitem[Chen(2005)]{chen2005matrix}
K.~Chen.
\newblock \emph{Matrix Preconditioning Techniques and Applications}.
\newblock Cambridge University Press, 2005.

\bibitem[Cook and Weisberg(1982)]{cook}
R.~Cook and S.~Weisberg.
\newblock \emph{Residuals and influence in regression}.
\newblock New York: Chapman and Hall, New York: Chapman Hall, 1982.

\bibitem[Cook(1986)]{cook1986assessment}
R.~D. Cook.
\newblock Assessment of local influence.
\newblock \emph{Journal of the Royal Statistical Society: Series B
  (Methodological)}, 48\penalty0 (2):\penalty0 133--155, 1986.

\bibitem[Dantzig(1957)]{dantzig1957discrete}
G.~B. Dantzig.
\newblock Discrete-variable extremum problems.
\newblock \emph{Operations research}, 5\penalty0 (2):\penalty0 266--288, 1957.

\bibitem[De~Cao et~al.(2021)De~Cao, Aziz, and Titov]{knowledge_editor}
N.~De~Cao, W.~Aziz, and I.~Titov.
\newblock Editing factual knowledge in language models, 2021.
\newblock arXiv Preprint.

\bibitem[Finkelstein et~al.(2012)Finkelstein, Taubman, Wright, Bernstein,
  Gruber, Newhouse, Allen, Baicker, and Group]{oregon_exp}
A.~Finkelstein, S.~Taubman, B.~Wright, M.~Bernstein, J.~Gruber, J.~P. Newhouse,
  H.~Allen, K.~Baicker, and O.~H.~S. Group.
\newblock {The Oregon Health Insurance Experiment: Evidence from the First
  Year}.
\newblock \emph{The Quarterly Journal of Economics}, 127\penalty0 (3):\penalty0
  1057--1106, 07 2012.

\bibitem[Fortunati et~al.(2016)Fortunati, Gini, and Greco]{Fortunati2016TheMC}
S.~Fortunati, F.~Gini, and M.~S. Greco.
\newblock The misspecified {C}ramer-{R}ao bound and its application to scatter
  matrix estimation in complex elliptically symmetric distributions.
\newblock \emph{IEEE Transactions on Signal Processing}, 64:\penalty0
  2387--2399, 2016.

\bibitem[Giordano et~al.(2019)Giordano, Stephenson, Liu, Jordan, and
  Broderick]{giordano2019swiss}
R.~Giordano, W.~Stephenson, R.~Liu, M.~Jordan, and T.~Broderick.
\newblock {A Swiss Army Infinitesimal Jackknife}.
\newblock In \emph{International Conference on Artificial Intelligence and
  Statistics}, pages 1139--1147. PMLR, 2019.

\bibitem[Hadi et~al.(1995)Hadi, Jones, and Ling]{hadi1995unifying}
A.~S. Hadi, W.~D. Jones, and R.~F. Ling.
\newblock A unifying representation of some case-deletion influence measures in
  univariate and multivariate linear regression.
\newblock \emph{Journal of statistical planning and inference}, 46\penalty0
  (1):\penalty0 123--135, 1995.

\bibitem[Hampel(1968)]{Hampel1968ContributionsTT}
F.~R. Hampel.
\newblock Contributions to the theory of robust estimation.
\newblock \emph{PhD Dissertation}, 1968.

\bibitem[Hampel(1974)]{hampel1974influence}
F.~R. Hampel.
\newblock The influence curve and its role in robust estimation.
\newblock \emph{Journal of the american statistical association}, 69\penalty0
  (346):\penalty0 383--393, 1974.

\bibitem[Henriques et~al.(2019)Henriques, Ehrhardt, Albanie, and
  Vedaldi]{curveball}
J.~F. Henriques, S.~Ehrhardt, S.~Albanie, and A.~Vedaldi.
\newblock Small steps and giant leaps: Minimal newton solvers for deep
  learning.
\newblock In \emph{IEEE/CVF International Conference on Computer Vision}, pages
  4763--4772, 2019.

\bibitem[Hofmann et~al.(2015)Hofmann, Lucchi, Lacoste-Julien, and
  McWilliams]{hofmann2015variance}
T.~Hofmann, A.~Lucchi, S.~Lacoste-Julien, and B.~McWilliams.
\newblock {Variance Reduced Stochastic Gradient Descent with Neighbors}.
\newblock \emph{Advances in Neural Information Processing Systems}, 28, 2015.

\bibitem[Jaeckel(1972)]{jaeckel1972infinitesimal}
L.~A. Jaeckel.
\newblock \emph{The infinitesimal jackknife}.
\newblock Bell Telephone Laboratories, 1972.

\bibitem[Jain et~al.(2017{\natexlab{a}})Jain, Kakade, Kidambi, Netrapalli,
  Pillutla, and Sidford]{jain2017markov}
P.~Jain, S.~M. Kakade, R.~Kidambi, P.~Netrapalli, V.~K. Pillutla, and
  A.~Sidford.
\newblock A {M}arkov {C}hain {T}heory {A}pproach to {C}haracterizing the
  {M}inimax {O}ptimality of {S}tochastic {G}radient {D}escent (for {L}east
  {S}quares).
\newblock In \emph{Conference on Foundations of Software Technology and
  Theoretical Computer Science}, pages 2:1--2:10, 2017{\natexlab{a}}.

\bibitem[Jain et~al.(2017{\natexlab{b}})Jain, Kakade, Kidambi, Netrapalli, and
  Sidford]{jain2017parallelizing}
P.~Jain, S.~M. Kakade, R.~Kidambi, P.~Netrapalli, and A.~Sidford.
\newblock Parallelizing stochastic gradient descent for least squares
  regression: Mini-batching, averaging, and model misspecification.
\newblock \emph{Journal of Machine Learning Research}, 18:\penalty0
  223:1--223:42, 2017{\natexlab{b}}.

\bibitem[Johnson and Zhang(2013)]{johnson2013accelerating}
R.~Johnson and T.~Zhang.
\newblock {Accelerating Stochastic Gradient Descent using Predictive Variance
  Reduction}.
\newblock \emph{Advances in neural information processing systems}, 26, 2013.

\bibitem[Koh and Liang(2017)]{koh}
P.~W. Koh and P.~Liang.
\newblock {Understanding Black-box Predictions via Influence Functions}.
\newblock In \emph{International Conference on Machine Learning}, volume~70,
  pages 1885--1894, 2017.

\bibitem[Laguel et~al.(2021)Laguel, Pillutla, Malick, and
  Harchaoui]{laguel2021superquantiles}
Y.~Laguel, K.~Pillutla, J.~Malick, and Z.~Harchaoui.
\newblock {Superquantiles at Work: Machine Learning Applications and Efficient
  Subgradient Computation}.
\newblock \emph{Set-Valued and Variational Analysis}, 29\penalty0 (4):\penalty0
  967--996, 2021.

\bibitem[Lanczos(1950)]{lanczos1950iteration}
C.~Lanczos.
\newblock {An Iteration Method for the Solution of the Eigenvalue Problem of
  Linear Differential and Integral Operators}.
\newblock \emph{Journal of Research of the National Bureau of Standards}, 1950.

\bibitem[Levy et~al.(2017)Levy, Seo, Choi, and Zettlemoyer]{zsre}
O.~Levy, M.~Seo, E.~Choi, and L.~Zettlemoyer.
\newblock Zero-shot relation extraction via reading comprehension.
\newblock In \emph{Proceedings of the 21st Conference on Computational Natural
  Language Learning ({C}o{NLL} 2017)}, pages 333--342, Vancouver, Canada, Aug.
  2017. Association for Computational Linguistics.
\newblock \doi{10.18653/v1/K17-1034}.
\newblock URL \url{https://aclanthology.org/K17-1034}.

\bibitem[Lewis et~al.(2020)Lewis, Liu, Goyal, Ghazvininejad, Mohamed, Levy,
  Stoyanov, and Zettlemoyer]{lewis2020bart}
M.~Lewis, Y.~Liu, N.~Goyal, M.~Ghazvininejad, A.~Mohamed, O.~Levy, V.~Stoyanov,
  and L.~Zettlemoyer.
\newblock {BART: Denoising Sequence-to-Sequence Pre-training for Natural
  Language Generation, Translation, and Comprehension}.
\newblock In \emph{ACL}, pages 7871--7880, 2020.

\bibitem[Lin et~al.(2018)Lin, Mairal, and Harchaoui]{lin2018catalyst}
H.~Lin, J.~Mairal, and Z.~Harchaoui.
\newblock {Catalyst Acceleration for First-order Convex Optimization: from
  Theory to Practice}.
\newblock \emph{Journal of Machine Learning Research}, 18\penalty0
  (212):\penalty0 1--54, 2018.

\bibitem[Louvet et~al.(2022)Louvet, Raymaekers, Van~Bever, and
  Wilms]{louvet2022influence}
G.~Louvet, J.~Raymaekers, G.~Van~Bever, and I.~Wilms.
\newblock The influence function of graphical lasso estimators.
\newblock \emph{COMPSTAT}, 2022.

\bibitem[Ma et~al.(2014)Ma, Mahoney, and Yu]{ma2014statistical}
P.~Ma, M.~Mahoney, and B.~Yu.
\newblock A statistical perspective on algorithmic leveraging.
\newblock In \emph{International Conference on Machine Learning}, pages 91--99.
  PMLR, 2014.

\bibitem[Ma(2021)]{distilgpt2}
Y.~Ma.
\newblock distilgpt2-finetuned-wikitext2.
\newblock \url{https://huggingface.co/MYX4567/distilgpt2-finetuned-wikitext2},
  2021.

\bibitem[Mackey et~al.(2014)Mackey, Jordan, Chen, Farrell, and
  Tropp]{mackey2014matrix}
L.~Mackey, M.~I. Jordan, R.~Y. Chen, B.~Farrell, and J.~A. Tropp.
\newblock Matrix concentration inequalities via the method of exchangeable
  pairs.
\newblock \emph{The Annals of Probability}, 42\penalty0 (3):\penalty0 906--945,
  2014.

\bibitem[Merity et~al.(2017)Merity, Xiong, Bradbury, and
  Socher]{merity2017pointer}
S.~Merity, C.~Xiong, J.~Bradbury, and R.~Socher.
\newblock {Pointer Sentinel Mixture Models}.
\newblock In \emph{ICLR}, 2017.

\bibitem[Nesterov and Nemirovskii(1994)]{nesterov1994interior}
Y.~E. Nesterov and A.~Nemirovskii.
\newblock \emph{{Interior-point polynomial algorithms in convex programming}},
  volume~13.
\newblock {SIAM}, 1994.

\bibitem[Ostrovskii and Bach(2021)]{ostrovskii2021finite}
D.~M. Ostrovskii and F.~Bach.
\newblock Finite-sample analysis of {M}-estimators using self-concordance.
\newblock \emph{Electronic Journal of Statistics}, 15\penalty0 (1), 2021.

\bibitem[Rockafellar and Uryasev(2000)]{rockafellar2000optimization}
R.~T. Rockafellar and S.~Uryasev.
\newblock Optimization of conditional value-at-risk.
\newblock \emph{Journal of risk}, 2:\penalty0 21--42, 2000.

\bibitem[Rousseeuw et~al.(2011)Rousseeuw, Hampel, Ronchetti, and
  Stahel]{rousseeuw2011robust}
P.~J. Rousseeuw, F.~R. Hampel, E.~M. Ronchetti, and W.~A. Stahel.
\newblock \emph{Robust statistics: the approach based on influence functions}.
\newblock John Wiley and Sons, 2011.

\bibitem[Rudin(2019)]{rudin2019stop}
C.~Rudin.
\newblock Stop explaining black box machine learning models for high stakes
  decisions and use interpretable models instead.
\newblock \emph{Nature Machine Intelligence}, 1\penalty0 (5):\penalty0
  206--215, 2019.

\bibitem[Saad(2003)]{saad2003iterative}
Y.~Saad.
\newblock \emph{Iterative methods for sparse linear systems}.
\newblock SIAM, 2003.

\bibitem[Sanh et~al.(2019)Sanh, Debut, Chaumond, and Wolf]{sanh2019distilbert}
V.~Sanh, L.~Debut, J.~Chaumond, and T.~Wolf.
\newblock {DistilBERT, a distilled version of BERT: smaller, faster, cheaper
  and lighter}.
\newblock In \emph{NeurIPS EMC$^2$ Workshop}, 2019.

\bibitem[Schioppa et~al.(2022)Schioppa, Zablotskaia, Vilar, and
  Sokolov]{arnoldi}
A.~Schioppa, P.~Zablotskaia, D.~Vilar, and A.~Sokolov.
\newblock Scaling up influence functions.
\newblock In \emph{AAAI Conference on Artificial Intelligence}, volume~36,
  pages 8179--8186, 2022.

\bibitem[Vershynin(2018)]{vershynin2018high}
R.~Vershynin.
\newblock \emph{High-Dimensional Probability: An Introduction with Applications
  in Data Science}.
\newblock Cambridge University Press, 2018.

\bibitem[Wainwright(2019)]{wainwright2019high}
M.~J. Wainwright.
\newblock \emph{High-Dimensional Statistics: A Non-Asymptotic Viewpoint}.
\newblock Cambridge University Press, 2019.

\bibitem[Zhao et~al.(2019)Zhao, Liu, Niu, and Leng]{zhao2019multiple}
J.~Zhao, C.~Liu, L.~Niu, and C.~Leng.
\newblock Multiple influential point detection in high dimensional regression
  spaces.
\newblock \emph{Journal of the Royal Statistical Society: Series B (Statistical
  Methodology)}, 81\penalty0 (2):\penalty0 385--408, 2019.

\bibitem[Zhu and Zhang(2004)]{zhu2004diagnostic}
H.~Zhu and H.~Zhang.
\newblock A diagnostic procedure based on local influence.
\newblock \emph{Biometrika}, 91\penalty0 (3):\penalty0 579--589, 2004.

\end{thebibliography}

\clearpage
\onecolumn
\appendix
\part{Appendix} %
\addcontentsline{toc}{section}{Appendix}
\parttoc %
\clearpage

\section{Notation Review}
\myparagraph{Setup}
We review notation from the paper, which will be used throughout the appendix. 
We define the parameter of interest $\theta_\star \in \Theta = \reals^p$ as 
\begin{align*} 
    \theta_\star := \argmin_{\theta \in \Theta} \Big[ F(\theta):= \expect_{Z \sim P}\left[ \ell(Z, \theta) \right] \Big] \,,
\end{align*}
where 
$P$ is an unknown probability distribution over a data space $\Zcal$ and
$\ell: \Zcal \times \Theta \to \reals_+$ is a loss function. We define the estimate of $\theta_\star$ using an i.i.d. sample $Z_{1:n} := (Z_1, \cdots, Z_n) \sim P^n$ as
\begin{align*}
\theta_n := \argmin_{\theta \in \Theta} \frac{1}{n} \sum_{i=1}^n \ell(Z_i, \theta) .
\end{align*}
We define the gradient of the loss function as $S(z, \theta) = \nabla_\theta\ell (Z, \theta)$ and the empirical gradient of the loss function as $S_n(\theta) := \frac{1}{n} \sum_{i=1}^n \nabla_{\theta} \ell(Z_i,\theta)$.

We define the population Hessian $H_\star = \nabla_{\theta_\star}^2 \ell (z,\theta_\star)$ of the population objective and the estimate of the Hessian as $H_n(\theta):= \frac{1}{n}\sum_{i=1}^n \nabla_\theta^2 \ell (Z_i, \theta)$. 
\myparagraph{Influence Function}
We define $G_\star = \text{Cov}_{Z\sim P}(\nabla_{\theta_\star} \ell(Z,\theta_\star))$ the gradient covariance at $\theta_\star$ and the  effective dimension $p_\star = \Tr(H_\star^{-1/2}G_\star H_\star^{-1/2})$.
 We define the population influence function as $I(z) := H_\star^{-1} \nabla_{\theta_\star} \ell(z, \theta_\star)$. We quantify the influence of a fixed data point $z$ on the estimator $\theta_n$ as $I_n(z)$ defined as
\begin{align*} 
    I_n(z) =  - H_n(\theta_n)^{-1}  \grad \ell(z, \theta_n) \,.
\end{align*}

\myparagraph{Most Influential Subset}
Let $\alpha \in (0, 1)$ and $h:\reals^p \to \reals$ be a continuously differentiable test function. Then we define the weights $w$ in the probability simplex $\Delta^{n-1} 
\theta_{n, w} := \argmin_{{\theta \in \Theta}} \sum_{i=1}^n w_i \ell(Z_i, \theta)$ and use them to define $W_\alpha$ as
\begin{gather*}
    W_\alpha := \left\{ 
    w \in \Delta^{n-1} \, : \,
    \parbox{12em}{at most $\alpha n$ elements of $w$ are zero and the rest are equal} 
    \right\} \,.
\end{gather*}
The maximum influence of any subset of data of size at most $\alpha n$ for a test function $h$ is expressed by 
\begin{align*}
        \SIF_{\alpha, n}(h) = 
        \max_{w \in W_\alpha} \left\{-\sum_{i=1}^n w_i \inp*{\grad h(\theta_n)}{H_n(\theta_n)^{-1} \grad \ell(Z_i, \theta_n)} \right\}.
\end{align*}
The population subset influence is defined as,
\begin{align}
I_\alpha(h) = \supq_\alpha \Big[ - \grad h(\theta_\star)\T {H(\theta_*)^{-1} \grad \ell(Z, \theta_\star)} \Big]  \,,
\end{align}
where $S_\alpha$ is the superquantile at level $\alpha$. We refer to \Cref{sec:a:superquantile}.

\myparagraph{Miscellaneous}
An unqualified norm $\norm{\cdot}$
refers to the Euclidean norm $\norm{v}_2$ for a vector $v$
and the spectral norm $\norm{M}_2$ for a matrix $M$.
We define a vector norm 
$\norm{x}_A = \inp{x}{Ax}$ 
and matrix norm $\norm{B}_A = \norm{A^{1/2}B A^{1/2}}_2$ for a positive definite $A$. 
We define the convex hull as $\conv T$ for a set $T \subset \reals^n$. 

We define $\Var(M) = \expect[MM\T] - \expect[M]\expect[M]\T$ for a random matrix $M$. 
We also denote $\D Q / \D P$ as the Radon-Nikodym derivative of $Q$ w.r.t. $P$. When $P$ and $Q$ have respective densities $p, q$, we have $\D Q /\D P(z) = q(z) / p(z)$ as simply the density ratio or likelihood ratio.

Lastly, we define the condition number of a positive definite matrix $A$ with spectral norm $\|A\|_2 \leq L$ and
minimum eigenvalue $\lambda_{\text{min}}(A)$ as $\kappa = L / \lambda_{\text{min}}(A)$. 

\section{Review of Computational Approaches} 
\label{app:pseudocode}

We present the pseudocode of the various computational approaches we consider in this work: 
\begin{itemize}
    \item \Cref{alg:cgd}: Conjugate gradient method,
    \item \Cref{alg:sgd}: Stochastic gradient descent, 
    \item \Cref{alg:lissa}: LiSSA, 
    \item \Cref{alg:svrg}: Stochastic variance-reduced gradient (SVRG) method, 
    \item \Cref{alg:arnoldi}: Low-rank approximation via the Arnoldi/Lanczos iterations. 
\end{itemize}

\begin{algorithm}[H]
\caption{Conjugate Gradient Method to Compute the Influence Function}
   \begin{algorithmic}[1]
   \Require vector $v$, batch Hessian vector product oracle $\text{HVP}_n(u)=H_n(\theta_n)u$, number of iterations $T$
        \State $u_0 =0$, $r_0 =- v - \text{HVP}_n(u_0), d_0 = r_0$
        \For{$t = 0,...,T-1$}
            \State $\alpha_t =  \frac{d^\top_t r_t}{d^\top_t \text{HVP}_n(d_t)}$
            \State $u_{t+1} = u_t + \alpha_t d_t$
            \State $r_{t+1} = - v - \text{HVP}_n(u_{t+1})$
            \State $\beta_t = \frac{r_{t+1}^\top r_{t+1}}{r_{t}^\top r_{t}}$
            \State $d_{t+1} = r_{t+1} + \beta_t d_t$
        \EndFor
        \State \textbf{return }$u_{T}$
\end{algorithmic}
\label{alg:cgd}
\end{algorithm}

\begin{algorithm}[H]
\caption{Stochastic Gradient Descent Method to Compute the Influence Function}
   \begin{algorithmic}[1]
   \Require vector $v$, Hessian vector product oracle $\text{HVP}(i, u)= \nabla^2 \ell(z_i, \theta_n)u$, number of iterations $T$, learning rate $\gamma$
        \State $u_0 = 0$
        \For{$t = 0,...,T-1$}
            \State Sample $i_t \sim \text{Unif}([n])$ 
            \State $u_{t+1} = u_t - \gamma (\text{HVP}(i_t, u_t)+v)$
        \EndFor
        \State \textbf{return }$u_{T}$
\end{algorithmic}
\label{alg:sgd}
\end{algorithm}

\begin{algorithm}[H]
\caption{The LiSSA Method to Compute the Influence Function \cite{lissa}}
\begin{algorithmic}[1]
  \Require  vector $v$, Hessian vector product oracle $\text{HVP}(i, u)= \nabla^2 \ell(z_i, \theta_n)u$, number of approximations $S$, number of iterations $T$, scaling factor $\gamma$ 
         \For{$s = 1,...,S$}
             \State $u_0^{(s)} = -v$
             \For{$t = 0,...,T-1$}
                 \State Sample $i_t \sim \text{Unif}([n])$ 
                 \State $u_{t+1}^{(s)} = -\gamma v + u_t^{(s)} - \gamma\,\text{HVP}(i_t, u_t^{(s)})$
             \EndFor
        \EndFor
     \State $u_T= \frac{1}{S}\bigg(\sum_{s=1}^{S} u_{T}^{(s)}\bigg)$   
     \State \textbf{return }$u_T$
\end{algorithmic}
\label{alg:lissa}
\end{algorithm}

\myparagraph{Connection between SGD and LiSSA} 
Observe that the updates of LiSSA for a fixed $s$ are identical to that of SGD: 
\begin{align*}
    u_{t+1}^{(s)} =   -\gamma v + u_t^{(s)} - \gamma \text{HVP}(i_t,u_t^{(s)}) = u_t^{(s)} - \gamma(\text{HVP}(i_t, u_t^{(s)}) + v) \,.
\end{align*}
Formally, we show that the sequence $u_1,...,u_t$ produced by stochastic gradient descent with initial guess $u_0 = -v $ (instead of $u_0 = 0$ as required by \Cref{alg:sgd}) and $u_1',...,u_t'$ produced by LiSSA with number of repetitions $S=1$ are identical. 
Note that $u_0 = u_0' = -v$. We show by induction that the two sequences $(u_t)$ and $(u_t')$ are identical provided the same samples $i_0, \cdots, i_{T-1}$ are drawn.
Suppose $u_t = u_t'$ for some $t \ge 0$. 
We have, 
\begin{align*}
    u_{t+1}' =   -\gamma v + u_t' - \gamma \text{HVP}(i_t,u_t') = u_t' - \gamma(\text{HVP}(i_t, u_t') + v)
    = u_t - \gamma(\text{HVP}(i_t, u_t) + v) = u_{t+1},
\end{align*}
showing that the sequences are identical.

\begin{algorithm}[H]
\caption{Stochastic Variance Reduced Gradient Method to Compute the Influence Function}
   \begin{algorithmic}[1]
   \Require vector $v$, Hessian vector product oracle $\text{HVP}(i, u)= \nabla^2 \ell(z_i, \theta_n)u$, number of epochs $S$, number of iterations per epoch $T$, learning rate $\gamma$
        \State $u_T^{(0)} = 0$
        \For {$s= 1,2,...,S$}
        \State $u_0^{(s)} = u_T^{(s-1)}$
        \State $\tilde{u}_0^{(s)} = \frac{1}{n}\sum_{i=1}^n \text{HVP}(u_0^{(s)}) - v$
        \For{$t = 0,...,T-1$}
            \State Sample $i_t \sim \text{Unif}([n])$
            \State $u_{t+1}^{(s)} = u_t^{(s)} - \gamma (\text{HVP}(i_t, u_t^{(s)})-\text{HVP}(i_t,u_0^{(s)})+\tilde{u}_0^{(s)})$
        \EndFor
    \EndFor
    \State \textbf{return }$u_{T}^{(S)}$
\end{algorithmic}
\label{alg:svrg}
\end{algorithm}

\begin{algorithm}[H]
\caption{Arnoldi Method to Compute the Influence Function \cite{arnoldi}}

\begin{algorithmic}[1]
   \Require vector $v$, test function $h$,
   initial guess $u_0$, batch Hessian vector product oracle $\text{HVP}_n(u)= H_n(\theta_n)u$, number of top eigenvalues $k$, number of iterations $T$
   \Ensure An estimate of $\inp{\grad h(\theta)}{H_n(\theta_n)^{-1} v}$
\State Obtain $\Lambda, G$ = \textproc{Arnoldi}$(u_0, T, k)$
\Comment{Cache the results for future calls}
\State \Return $\inp{G \grad h(\theta)}{\Lambda^{-1} G v}$ 
\Statex 
\Procedure{Arnoldi}{$u_0$, $T$, $k$}
\State $w_0=1 = u_0 / \|u_0 \|_2$
\State $A = \textbf{0}_{T+1\times T}$
\For{$t = 1,...,T$}
    \State Set $u_t = \text{HVP}_n(w_t) - \sum_{j=1}^t \inp{u_t}{w_j} \, w_j$
    \State Set $A_{j, t} = \inp{u_t}{w_j}$ for $j = 1, \ldots, t$ and $A_{t+1, t} = \norm{u_t}_2$
    \State Update $w_{t+1} = u_t / \norm{u_t}$
    \EndFor
\State Set $\tilde{A} = A[1:T,\,\, :] \in \reals^{T \times T}$ (discard the last row)
\State Compute an eigenvalue decomposition $\tilde A = \sum_{j=1}^{T} \lambda_j e_j e_j^\top$ with $\lambda_j$’s in descending order
\State Define $G:\reals^p \to \reals^k$ as the operator 
$Gu =\big(\inp{u}{W\T e_1}, \cdots, \inp{u}{W\T e_k}\big)$, where 
$W = (w_1\T; \cdots; w_{T}\T) \in \reals^{T \times p}$
\State \textbf{return }
diagonal matrix
$\Lambda = \diag(\lambda_1, \cdots, \lambda_k)$ and the operator $G$
\EndProcedure
\end{algorithmic}
\label{alg:arnoldi}
\end{algorithm} %
\section{Effective Dimensions and Eigenspectra of the Hessian and Gradient Covariance}
\label{sec:eigendecay}

Recall the following definitions, the population Hessian $H_\star = \nabla^2 F(\theta_\star)$ of the population objective and $G_\star = \text{Cov}_{Z\sim P}(\nabla \ell(Z,\theta_\star))$ the gradient covariance at $\theta_\star$. We are interested in how the effective dimension $p_\star = \text{Tr}(H_\star^{-1/2}G_\star H_\star^{-1/2})$ differs from the parameter dimension $p$ due to the eigendecay of $H_\star$. First, we assume that $H_\star$ and $G_\star$ share the same eigenvectors. Then, using the eigenvalue decomposition of a matrix, we can say that for $Q$ containing the eigenvectors as its columns,
\begin{align*}
    H_\star = Q \Lambda_H Q^\top,\\
    G_\star = Q \Lambda_G Q^\top 
\end{align*}
where $\Lambda_A = \text{Diag}\{\lambda_{a,i}\}$ contains the eigenvalues of $A$ in non-increasing order. Therefore, we get
\begin{align*}
    H_\star^{-1/2} = Q \Lambda_H^{-1/2} Q^\top \,.
\end{align*}
Using these definitions we now show the following,
\begin{align*}
    H_\star^{-1/2}G_\star H_\star^{-1/2} &= (Q \Lambda_H^{-1/2} Q^\top)(Q \Lambda_G Q^\top)(Q \Lambda_H^{-1/2} Q^\top) \\
    & = Q \Lambda_H^{-1/2} \Lambda_G \Lambda_H^{-1/2} Q^\top\\
    & = Q \text{Diag}\bigg\{\frac{\lambda_{g,1}}{\lambda_{h,1}}...\frac{\lambda_{g,p}}{\lambda_{h,p}}\bigg \} Q^\top.
\end{align*}
Therefore, due to the cyclic property of traces we define,
\begin{align*}
    \Tr(H_\star^{-1/2}G_\star H_\star^{-1/2}) = \sum_{i=1}^p \frac{\lambda_{g,i}}{\lambda_{h,i}}.
\end{align*}
Here we have shown that the dimension dependency of $p_\star$ is dependent on the eigendecay of $G_\star$ and $H_\star$. To illustrate this point, we show four examples of how these calculations continue. All examples are outlined in Table \ref{tab:decay}.
\myparagraph{Polynomial - Polynomial Eigendecay}
We assume that both $G_\star$ and $H_\star$ have polynomial eigendecay, that is, $\lambda_{g,i} \lesssim i^{-\alpha}$ and $\lambda_{h,i} \lesssim i^{-\beta}$. Then we can write,
\begin{align*}
    p_\star \lesssim \sum_{i=1}^p i^{\beta - \alpha} \lesssim \int_1^p x^{\beta - \alpha} \D x \lesssim p^{\beta - \alpha +1}.
\end{align*}

\myparagraph{Polynomial - Exponential Eigendecay}
We assume that $G_\star$ has polynomial eigendecay and $H_\star$ have exponential eigendecay, that is $\lambda_{g,i} \lesssim i^{-\alpha}$ and $\lambda_{h,i} \lesssim e^{-\nu i}$. Then we can write,
\begin{align*}
    p_\star \lesssim \sum_{i=1}^p e^{\nu i} i^{-\alpha} \lesssim p^{1-\alpha}e^{\nu p},
\end{align*}
where the last inequality holds because $e^{\nu x} x^{-\alpha}$ is increasing when $x$ is large enough. 

\myparagraph{Exponential - Polynomial Eigendecay}
We assume that $G_\star$ has exponential eigendecay and $H_\star$ have polynomial eigendecay, that is $\lambda_{g,i} \lesssim e^{-\mu i}$ and $\lambda_{h,i} \lesssim i^{-\beta} $. Then we can write,
\begin{align*}
    p_\star \lesssim \sum_{i=1}^p e^{-\mu i} i^{\beta} \lesssim 1,
\end{align*}
where the last inequality holds because $e^{-\mu x} x^{\beta}$ is decreasing when $x$ is large enough. 

\myparagraph{Exponential - Exponential Eigendecay}
We assume that $G_\star$ has exponential eigendecay and $H_\star$ have exponential eigendecay, that is $\lambda_{g,i} \lesssim e^{-i \mu}$ and $\lambda_{h,i} \lesssim e^{-i \nu}$ . Then we can write,
\begin{align*}
    p_\star \lesssim \sum_{i=1}^p e^{(\nu - \mu)i}. 
\end{align*} 
If $\mu > \nu$, then 
\begin{align*}
   \sum_{i=1}^p e^{ (\nu - \mu)i} \lesssim 1.
\end{align*} 
If $\mu < \nu$, then 
\begin{align*}
   \sum_{i=1}^p e^{ (\nu - \mu)i} \lesssim \int_1^p e^{ (\nu - \mu)i} = \frac{1}{\nu - \mu}\bigg(e^{(\nu-\mu)p} - e^{(\nu-\mu)}\bigg) \lesssim e^{(\nu - \mu)p}.
\end{align*} 
And if $\mu = \nu$, then 
\begin{align*}
   \sum_{i=1}^p e^{0} = p.
\end{align*} 
\begin{table}[t]
    \caption{Comparison between the effective dimension $p_\star$ and the parameter dimension $p$ in different regimes of eigendecays of $G_\star$ and $H_\star$ assuming they share the same eigenvectors.}
    \label{tab:decay}
    \centering
    \renewcommand{\arraystretch}{1.2}
    \begin{tabular}{lccccc}
        \addlinespace[0.4em]
        \toprule
        & \multicolumn{2}{c}{\textbf{Eigendecay}} & \multicolumn{2}{c}{\textbf{Dimension Dependency}} & \textbf{Ratio} \\
        & $G_\star$ & $H_\star$ & $p_\star$ & $p$ & $p_\star/p$ \\
        \midrule
        Poly-Poly & $i^{-\alpha}$ & $i^{-\beta}$ & $p^{(\beta - \alpha + 1) \vee 0}$ & $p$ & $p^{(\beta - \alpha)\vee(-1)}$ \\
        Poly-Exp & $i^{-\alpha}$ & $ e^{-\nu i}$ & $p^{1-\alpha} e^{\nu p}$ & $p$ & $p^{-\alpha} e^{\nu p}$ \\
        Exp-Poly & $e^{-\mu i}$ & $i^{-\beta}$ & $1$ & $p$ & $p^{-1}$ \\
        Exp-Exp & $e^{-\mu i}$ & $e^{-\nu i}$ & \begin{tabular}{@{}c@{}} \qquad $p$ \mbox{ if }  $\mu = \nu$ \\ \qquad $1$ \mbox{ if } $\mu > \nu$ \\ $e^{(\nu - \mu) p}$ \mbox{ if } $\mu < \nu$ \end{tabular} & $p$ & \begin{tabular}{@{}c@{}} $1$ \mbox{ if }  $\mu = \nu$ \\ $p^{-1}$ \mbox{ if } $\mu > \nu$ \\ $p^{-1} e^{(\nu - \mu) p}$ \mbox{ if } $\mu < \nu$ \end{tabular} \\
        \bottomrule
    \end{tabular}
\end{table} %
\section{Statistical Error Bounds for Influence Estimation}
\label{appx:proof_influence_est}

The main purpose of this section is to prove the statistical error bound \Cref{thm:if-stat-main}.
We use $C$ to denote an absolute constant which may change from line to line.
We use subscripts to emphasize the dependency on problem-specific constants, e.g., $C_{K_1}$ is a constant that only depends on $K_1$. 

\myparagraph{Notation}
Let $z$ be a fixed data point not related to the sample $Z_1, \cdots, Z_n \sim P$.
Recall that the influence of upweighting an observation $z$ on the model parameter $\theta$ is given by
\begin{align}
    I_n(z) = - H_n(\est)^{-1} S(z, \est),
\end{align}
where $H_n(\theta) := \frac1n \sum_{i=1}^n \nabla_\theta^2 \ell(Z_i, \theta)$ is the empirical Hessian and $S(z, \theta) := \nabla_\theta \ell(z, \theta)$ is the gradient at $z$.
Let $\theta_\star$ be the minimizer (assumed to exist) of the population risk $\mathbb{E}[\ell(z, \theta)]$ and $H(\theta) := \Expect[\nabla_\theta^2 \ell(z, \theta)]$.
We write $H_\star := H(\theta_\star)$ for short.
We are interested in bounding the difference
\begin{align*}
    \staterr := \norm{H_n(\est)^{-1} S(z, \est) - H_\star^{-1} S(z, \theta_\star)}_{H_\star},
\end{align*}
where $\norm{u}_{A} := \sqrt{u^\top A u}$ for a vector $u$ and a positive semidefinite matrix $A$.

\subsection{Assumptions}
We state the full assumptions under which the statistical bound holds. 

\begin{assumption}\label{assm:pseudo_self_conc}
For any $z\in \Zcal$, the loss function $\ell(z, \cdot)$ is pseudo self-concordant for some $R \ge 1$: 
\[
    \vert D^3_\theta \ell(z, \theta)[u,u,v]\vert \leq R \norm{u}^2_{\grad^2 \ell(z, \theta)} \, \norm{v}_2 \,,
\]
where $D^3_x f(x)[u, v, w] := \frac{\D}{\D t} \inp{u}{\grad^2 f(x + tw) \, v}|_{t=0}$ for $f$ thrice continuously differentiable. 
\end{assumption}
The most useful consequence of this assumption is a spectral approximation of the Hessian 
$(1/2) H(\theta') \preceq H(\theta) \preceq 2\, H(\theta')$ for $\theta$ and $\theta'$ close enough in terms of the $L_2$ distance. 

\begin{assumption}\label{assm:sub_gaussian_grad}
(Sub-Gaussian Gradient). There exists a constant $K_1 \ge 1$ such that the normalized gradient $H(\theta_\star)^{-1/2} \grad \ell(Z, \theta_\star)$ at $\theta_\star$ is sub-Gaussian with parameter $K_1$ (see \Cref{sec:defs} for a precise definition). 
\end{assumption}

\begin{assumption}\label{assm:matrix_berstein}
(Matrix Bernstein of Hessian). The standardized Hessian $H(\theta_\star)^{-1/2} \, \grad^2 \ell(Z, \theta_\star) \, H(\theta_\star)^{-1/2} - \id_p$ at $\theta_\star$ satisfies a Bernstein condition with parameter $K_2 \ge 1$ (see \Cref{sec:defs} for a definition). Moreover, 
    \begin{align*}
        \sigma^2_H := \norm*{ \Var\big(H(\theta_\star)^{-1/2} \, \grad^2 \ell(Z, \theta_\star) \, H(\theta_\star)^{-1/2}\big)}_2
    \end{align*}
is finite, where we denote $\Var(M) = \expect[M M\T] - \expect[M] \, \expect[M]\T$ for a random matrix $M$. 
\end{assumption}

\subsection{Proof of the Statistical Bound of Theorem \ref{thm:if-stat-main}}

We now state and prove the full version of \Cref{thm:if-stat-main}. Note that this bound is stated in terms of the $H_\star$ norm, but without the square.

\begin{customthm}{\ref{thm:if-stat-main}}
Under Assumptions \ref{assm:pseudo_self_conc},\ref{assm:sub_gaussian_grad}, and \ref{assm:matrix_berstein}, we have, with probability at least $1-\delta$,
\begin{align*}
    \staterr \le C_{K_1, K_2, \sigma_H}  \log\left(\frac{2p}{\delta}\right) \sqrt{\log\left(\frac{e}{\delta}\right)}
    \left( 1 + R\sqrt{\frac{p_\star}{\mu_\star}} \right) \sqrt{\frac{ p_\star}{n}} 
\end{align*}
whenever $n \ge C_{K_1, K_2, \sigma_H} \left( \frac{p_\star}{\mu_\star} R^2 \log\left(\frac{e}{\delta}\right) + \log\left(\frac{2p}{\delta}\right)\right)$, where $p_\star := \Tr \{ H_\star^{-1/2}G_\star H_\star^{-1/2} \}$ and $\mu_\star = \lambda_{\min}(H_\star)$.
\end{customthm}

\begin{proof}
Define
\begin{equation}\label{eq:rn_tn}
    \begin{split}
        r_n &:= \sqrt{CK_1^2 \log{(2e/\delta)} \frac{p_\star}{n}} \\
        t_n &:= \frac{2\sigma_H^2}{-K_2 + \sqrt{K_2^2 + 2\sigma_H^2 n / \log{(4p/\delta)}}}.
    \end{split}
\end{equation}
Note that they both decay as $O(n^{-1/2})$.
The proof consists of several key steps.

\textbf{Step 1. Upper bound $\staterr$ by basic terms involving the standardized gradient and the standardized Hessian.}
By the triangle inequality, it holds that
\begin{align}\label{eq:diff_first_bound}
    \staterr \le \norm{(H_n(\est)^{-1} - H_\star^{-1}) S(z, \est)}_{H_\star} + \norm{H_\star^{-1}(S(z, \est) - S(z, \theta_\star))}_{H_\star}.
\end{align}
The first term in \eqref{eq:diff_first_bound} can be upper bounded by
\begin{align}\label{eq:diff_first_term}
    \norm{[H_n(\est)^{-1} - H_\star^{-1}] [S(z, \est) - S(z, \theta_\star)]}_{H_\star} + \norm{[H_n(\est)^{-1} - H_\star^{-1}] S(z, \theta_\star)}_{H_\star}.
\end{align}
By the triangle inequality again, it can be shown that, for any $v \in \reals^p$,
\begin{align*}
    \norm{[H_n(\est)^{-1} - H_\star^{-1}] v}_{H_\star}
    &= \norm{[H_\star^{1/2} H_n^{-1}(\est) H_\star^{1/2} - H_\star^{-1/2}H_\star^{1/2}] H_\star^{-1/2} v}_2\\
    &\le \norm{H_\star^{1/2} H_n^{-1}(\est) H_\star^{1/2} - \id_p}_2 \norm{H_\star^{-1/2} v}_2.
\end{align*}
As a result, \eqref{eq:diff_first_term} can be further upper bounded by
\begin{align*}
    \underbrace{\norm{H_\star^{1/2} H_n(\est)^{-1} H_\star^{1/2} - \id_p}_2}_{A_3} \Big\{ \underbrace{\norm{H_\star^{-1/2} [S(z, \est) - S(z, \theta_\star)]}_2}_{A_2} + \underbrace{\norm{H_\star^{-1/2} S(z, \theta_\star)}_2}_{A_1} \Big\}.
\end{align*}
Similarly, the second term in \eqref{eq:diff_first_bound} can be upper bounded by
\begin{align*}%
    \norm{H_\star^{-1/2} [S(z, \est) - S(z, \theta_\star)]}_2 = A_2.
\end{align*}
Hence, it suffices to bound the three terms $A_1$, $A_2$, and $A_3$.
For that purpose, we define the following events
\begin{align*}
    \event_1 &:= \left\{ \norm{H_\star ^{-1/2} S(z,\theta_\star)}^2_2 \le C K_1^2 \log{(e/\delta)} p_\star \right\} \\
    \event_2 &:= \left\{ \norm{\est - \theta_\star}_{H_\star}^2 \le C K^2_1\log(e/\delta) \frac{p_\star}{n} \right\} \\
    \event_3 &:= \left\{ \norm{H_\star ^{-1/2}H(z,\theta_\star)H_\star^{-1/2} - \id_p}_2 \leq t_1 \right\}\\
    \event_4 &:= \left\{ \norm{H_\star^{1/2} H_n(\est)^{-1} H_\star^{1/2} - \id_p}_2 \le \frac{Rr_n/\sqrt{\mu_\star} + t_n}{1 - Rr_n/\sqrt{\mu_\star} - t_n} \right\}.
\end{align*}
Moreover, we assume $n \ge \max\{ 4(K_2+2\sigma^2_H)\log(16p/\delta), C K_1^2\log(e/\delta)p_\star R^2/\mu_\star \}$ throughout the proof.
In the following, we bound $A_1$, $A_2$, $A_3$ on the event $\event_1 \event_2 \event_3 \event_4$, and control the probability of this event.

\textbf{Step 2. Control $A_1$.}
On the event $\event_1$, we know
\begin{align*}
    A_1 \leq \sqrt{C K_1^2 \log{(e/\delta)} p_\star}.
\end{align*}

\textbf{Step 3. Control $A_2$.}
According to Taylor's theorem, it holds that
\begin{align*}
    S(z, \est) - S(z, \theta_\star) = H(z, \bar \theta) (\est - \theta_\star),
\end{align*}
where $\bar \theta \in \text{Conv}\{\est, \theta_\star\}$.
Therefore, we can rewrite $A_2$ as
\begin{align*}
    A_2 &= \norm{H_\star^{-1/2}H(z, \bar \theta) (\est - \theta_\star)}_2\\
    & =\norm{H_\star^{-1/2}H(z, \bar \theta)H_\star^{-1/2}H_\star^{1/2} (\est - \theta_\star)}_2.
\end{align*}
Consequently,
\begin{align*}
    A_2 \le \norm{H_\star^{-1/2} H(z, \bar \theta) H_\star^{-1/2}}_2 \norm{H_\star^{1/2}(\est - \theta_\star)}_2.
\end{align*}
According to \Cref{prop:hessian}, we have
\begin{align*}
    e^{-R\norm{\bar \theta - \theta_\star}_2} H(z, \theta_\star) \preceq H(z,\bar{\theta}) \preceq e^{R\norm{\bar \theta - \theta_\star}_2} H(z, \theta_\star).
\end{align*}
Note that $R \norm{\bar \theta - \theta_\star}_2 \le R \norm{\est - \theta_\star}_2 \le R\mu_\star^{-1/2} \norm{\est - \theta_\star}_{H_\star}$.
It follows from the event $\event_2$ that
\begin{align}\label{eq:one_point_hess_sand}
    \frac12 H(z, \theta_\star) \preceq H(z, \bar \theta) \preceq 2H(z, \theta_\star).
\end{align}
As a result, we have
\begin{align*}
    \norm{H_\star^{-1/2} H(z, \bar \theta) H_\star^{-1/2}}_2 \le 2 \norm{H_\star^{-1/2} H(z, \theta_\star) H_\star^{-1/2}}_2.
\end{align*}
On the event $\event_3$, we know
\begin{align}\label{eq:normalized_hessian}
    \norm{H_\star ^{-1/2}H(z, \theta_\star)H_\star^{-1/2}}_2 \leq 1 + t_1.
\end{align}
Therefore, by the event $\event_2$ and \eqref{eq:normalized_hessian}, $A_2$ is upper bounded by
\begin{align*}
    A_2 \le C (1 + t_1) r_n.
\end{align*}

\textbf{Step 4. Control $A_3$.}
On the event $\event_4$, we have
\begin{align*}
    A_3 \le \frac{Rr_n/\sqrt{\mu_\star} + t_n}{1 - Rr_n/\sqrt{\mu_\star} - t_n}.
\end{align*}

\textbf{Step 5. Control the probability of the event $\event_1\event_2\event_3\event_4$.}

\textit{Event $\event_1$}.
Since $\theta_\star$ is a minimizer of the population risk, then, by the first order optimality condition, we have $E[S(z, \theta_\star)]=0$.
Moreover, we have 
\begin{align*}
     \text{Cov}(G_\star^{-1/2} S(z, \theta_\star))&= E[G_\star^{-1/2} S(z, \theta_\star)S(z, \theta_\star)^\top G_\star^{-1/2}] \\
    & = G_\star^{-1/2} E[ S(z, \theta_\star)S(z, \theta_\star)^\top] G_\star^{-1/2}\\
    & = G_\star^{-1/2} G_\star G_\star^{-1/2} = \id_p.
\end{align*}
It follows that $G_\star^{-1/2} S(z, \theta_\star)$ is an isotropic random vector.
Let $J := G_\star^{1/2} H_\star^{-1} G_\star^{1/2}$.
It can be checked that
\begin{align*}
    \norm{H_\star^{-1/2} S(z, \theta_\star)}^2_2 = \norm{G_\star^{-1/2} S(z, \theta_\star)}_J^2, 
\end{align*}
where we denote $\norm{A}_B = \norm{B^{1/2} A B^{1/2}}_2$ for positive semidefinite $B$. 
Now it follows from \Cref{thm:isotropic_tail} that, with probability at least $1-\delta/4$,
\begin{align*}
    \norm{H_\star ^{-1/2} S(z,\theta_\star)}^2_2 \le C \left[ \textbf{Tr}(J) + K_1^2 \left(\norm{J}_2 \sqrt{\log(e/\delta)}+\norm{J}_\infty \log(1/\delta) \right) \right] \le C K_1^2 \log{(e/\delta)} p_\star,
\end{align*}
since $\norm{J}_\infty \le \norm{J}_2 \le \textbf{Tr}(J) = p_\star$.
Therefore, $\Prob(\event_1) \ge 1 - \delta/4$.

\textit{Event $\event_2$}. By \Cref{lem:est_error}, we have $\Prob(\event_2) \ge 1 - \delta/4$.

\textit{Event $\event_3$}. By Assumption \ref{assm:matrix_berstein}, we know that
\begin{align*}
    H_\star^{-1/2}H(z,\theta_\star)H_\star^{-1/2} - \id_p
\end{align*}
satisfies a Bernstein condition with parameter $K_2$.
It follows from \Cref{thm:bernstein_matrix} that $\Prob(\event_3) \ge 1 - \delta/4$.

\textit{Event $\event_4$}. It follows directly from \Cref{prop:bound_inverse_hess} that $\Prob(\event_4) \ge 1 - \delta/4$.

Now, by a union bound, we obtain $\Prob(\event_1\event_2\event_3\event_4) \ge 1 - \delta$.

\textbf{Step 6. Conclusion.} Putting all the above results together, we have shown that, with probability at least $1 - \delta$,
\begin{align*}
    \staterr \le C \frac{Rr_n/\sqrt{\mu_\star} + t_n}{1 - Rr_n/\sqrt{\mu_\star} - t_n} \left[ \sqrt{K_1^2 \log{(e/\delta)} p_\star} + (1 + t_1) r_n\right] + (1 + t_1) r_n.
\end{align*}
\end{proof}

\subsection{Intermediate Results}

The proof of \Cref{thm:if-stat-main} relies on two key results: 1) the estimator $\theta_n$ belongs to a neighborhood of $\theta_\star$ stated in \Cref{lem:est_error}, and 2) the inverse empirical Hessian $H_n(\theta_n)^{-1}$ is close to it population counterpart $H_\star^{-1}$ stated in \Cref{prop:bound_inverse_hess}.
Before we prove them, we give several useful lemmas.

\begin{lemma}\label{lem:emp_risk_self_concord}
    Under \Cref{assm:pseudo_self_conc}, the empirical risk $F_n$ is pseudo self-concordant with parameter $R$.
\end{lemma}

\begin{proof}
    By \Cref{assm:pseudo_self_conc}, the loss $\ell(Z_i, \cdot)$ is pseudo self-concordant with parameter $R$ for every $i \in \{1, \dots, n\}$.
    Since $F_n(\theta) = \frac1n \sum_{i=1}^n \ell(Z_i, \theta)$, we have
    \begin{align*}
        \abs{D_\theta^3 F_n(\theta)[u, u, v]}
        &= \myabs{\frac1n \sum_{i=1}^n D_\theta^3 \ell(Z_i, \theta)[u, u, v]}
        \le \frac1n \sum_{i=1}^n \abs{D_\theta^3 \ell(Z_i, \theta)[u, u, v]} \\
        &\le \frac1n \sum_{i=1}^n R \norm{v}_2 u^\top \nabla_\theta^2 \ell(Z_i, \theta) u = R \norm{v}_2 u^\top \nabla_\theta^2 F_n(\theta) u.
    \end{align*}
    This completes the proof.
\end{proof}

The next lemma provides a sufficient condition for the estimator $\theta_n$ to be close to $\theta_\star$.
\begin{lemma}\label{lem:localization}
    Under \Cref{assm:pseudo_self_conc}, whenever
    \begin{align*}
        \norm{S_n(\theta_\star)}_{H_n^{-1}(\theta_\star)} \le \sqrt{\lambda_{\min}(H_n(\theta_\star))}/(2R),
    \end{align*}
    the estimator $\est$ uniquely exists and satisfies
    \begin{align*}
        \norm{\est - \theta_\star}_{H_n(\theta_\star)} \le 4 \norm{S_n(\theta_\star)}_{H_n^{-1}(\theta_\star)}.
    \end{align*}
\end{lemma}
\begin{proof}
    By \Cref{lem:emp_risk_self_concord}, we have $F_n$ is pseudo self-concordant with parameter $R$.
    Since $\theta_n$ is the empirical risk minimizer, the claim follows from \Cref{prop:f_localization} with $f = F_n$ and $x = \theta_\star$.
\end{proof}

\begin{lemma}\label{lem:norm_score}
    Under \Cref{assm:sub_gaussian_grad}, it holds that, with probability at least $1 - \delta$,
    \begin{align*}
        \norm{S_n(\theta_\star)}_{H_\star^{-1}}^2 \le \frac1n C K_1^2 \log{(e/\delta)} p_\star.
    \end{align*}
\end{lemma}
\begin{proof}
    Define $W := \sqrt{n} G_\star^{-1/2} S_n(\theta_\star)$.
    It can be verified that $\Expect[W] = \sqrt{n} G_\star^{-1/2} S(\theta_\star) = 0$ and
    \begin{align*}
        \Expect[WW^\top]
        &= \frac1n G_\star^{-1/2} \Expect\left[ \left( \sum_{i=1}^n S(Z_i, \theta_\star) \right) \left( \sum_{i=1}^n S(Z_i, \theta_\star) \right)^\top  \right]^2 G_\star^{-1/2} \\
        &= G_\star^{-1/2} \frac1n \sum_{i=1}^n \Expect[S(Z_i, \theta_\star) S(Z_i, \theta_\star)^\top] G_\star^{-1/2} = \id_p.
    \end{align*}
    Moreover, by \Cref{lem:sum_subG} and \Cref{assm:sub_gaussian_grad}, we get that $W$ is sub-Gaussian with $\norm{W}_{\psi_2} \le C K_1$.
    Define $J := G_\star^{1/2} H_\star^{-1} G_\star^{1/2} / n$.
    It is clear that $\norm{S_n(\theta_\star)}_{H_\star^{-1}}^2 = \norm{W}_{J}^2$.
    By \Cref{thm:isotropic_tail}, we have, with probability at least $1 - \delta$,
    \begin{align*}
    \norm{S_n(\theta_\star)}_{H_\star^{-1}}^2 \le C K_1^2 \log{(e/\delta)} p_\star.
    \end{align*}
    Here we have used $\norm{J}_{\infty} \le \norm{J}_2 \le \Tr(J) = p_\star$, $\log{(1/\delta)} \le \log{(e/\delta)}$, and $\sqrt{\log{(e/\delta)}} \le \log{(e/\delta)}$.
\end{proof}

\begin{lemma}\label{lem:hess_sand}
    Under \Cref{assm:matrix_berstein}, it holds that, with probability at least $1 - \delta$,
    \begin{align*}
        \frac12 H_\star \preceq H_n(\theta_\star) \preceq \frac32 H_\star,
    \end{align*}
    whenever $n \ge 4(K_2 + 2\sigma_H^2) \log{(2p/\delta)}$.
\end{lemma}
\begin{proof}
    By \Cref{assm:matrix_berstein} and \Cref{thm:bernstein_matrix}, it holds that, for any $t > 0$,
    \begin{align*}
        \Prob\left( \norm{H_\star^{-1/2} H_n(\theta_\star) H_\star^{-1/2} - \id_p}_{2} \ge t \right) \le 2p \exp\left\{ -\frac{nt^2}{2(\sigma_H^2 + K_2 t)} \right\}.
    \end{align*}
    The claim then follows by setting $t = 1/2$.
\end{proof}

Now we are ready to prove the localization result.

\begin{proposition}\label{lem:est_error}
    Under Assumptions \ref{assm:pseudo_self_conc},\ref{assm:sub_gaussian_grad}, and \ref{assm:matrix_berstein}, we have, with probability at least $1 - \delta$, the estimator $\est$ uniquely exists and satisfies
    \begin{align}\label{eq:est_error}
        \norm{\est - \theta_\star}_{H_\star}^2 \le CK_1^2 \, \frac{p_\star}{n} \,  \log\left(\frac{e}{\delta}\right) 
    \end{align}
    whenever $n \ge \max\{ 4(K_2+2\sigma^2_H)\log(4p/\delta), \frac{C K_1^2p_\star R^2}{\mu_\star}  \log(e/\delta)\}$.
\end{proposition}
\begin{proof}
    We define two events,
    \begin{align*}
        \event_1 &:= \left\{ \norm{S_n(\theta_\star)}_{H_\star^{-1}}^2 \le \frac{1}{n} C K^2_1\log(e/\delta)p_\star \right\} \\
        \event_2 &:= \left\{ \frac{1}{2}  H_\star\preceq  H_n(\theta_\star) \preceq \frac{3}{2} H_\star \right\}.
    \end{align*}
    It suffices to prove the bound \eqref{eq:est_error} on $\event_1 \event_2$ and show $\Prob(\event_1 \event_2) \ge 1 - \delta$.
    
    \textbf{Step 1. Prove the bound.}
    By the events $\event_2$, we have $\sqrt{\lambda_{\min}(H_n(\theta_\star))}/(2R) \ge \sqrt{\mu_{\star}}/(2\sqrt{2}R)$.
    Note that $n \ge C K_1^2\log(e/\delta)p_\star R^2/\mu_\star$.
    It follows from the event $\event_1$ that $\norm{S_n(\theta_\star)}_{H_\star^{-1}} \le \sqrt{\lambda_{\min}(H_n(\theta_\star))}/(2\sqrt{2}R)$.
    By the event $\event_2$, we have
    \begin{align*}
        \norm{S_n(\theta_\star)}_{H_n^{-1}(\theta_\star)} \le \sqrt{2} \norm{S_n(\theta_\star)}_{H_\star^{-1}} \le \frac{\sqrt{\lambda_{\min}(H_n(\theta_\star))}}{2R}.
    \end{align*}
    According to \Cref{lem:localization}, $\est$ uniquely exists and satisfies
    \begin{align*}
        \norm{\est - \theta_\star}_{H_\star}^2 \le 16 \norm{S_n(\theta_\star)}_{H_n^{-1}(\theta_\star)}^2.
    \end{align*}
    Now the bound \eqref{eq:est_error} follows from the event $\event_1$.
    
    \textbf{Step 2. Control the probability.}
    According to \Cref{lem:norm_score} and \Cref{lem:hess_sand}, we know $\Prob(\event_1) \ge 1 - \delta/2$ and $\Prob(\event_2) \ge 1 - \delta/2$, respectively.
    Consequently,
    \begin{align*}
        \Prob(\event_1 \event_2) = 1 - \Prob(\event_1^c \event_2^c) \ge 1 - \Prob(\event_1^c) - \Prob(\event_2^c) \ge 1 - \delta,
    \end{align*}
    which completes the proof.
\end{proof}

We then bound the difference between the inverse empirical Hessian and the inverse population Hessian.
Recall that we use the notation $\norm{A}_B := \norm{B^{1/2} A B^{1/2}}_2$ for $B$ positive semidefinite.
\begin{proposition}\label{prop:bound_inverse_hess}
    Under Assumptions \ref{assm:pseudo_self_conc}, \ref{assm:sub_gaussian_grad}, and \ref{assm:matrix_berstein}, we have, with probability at least $1 - \delta$,
    \begin{align*}
        \norm{H_n(\est)^{-1} - H_\star^{-1}}_{H_\star} \le C_{K_1, K_2, \sigma_H} \left( \sqrt{\log\left(\frac{2p}{\delta}\right)} + R\sqrt{\frac{p_\star}{\mu_\star} \log\left(\frac{e}{\delta}\right) } \right) \frac1{\sqrt{n}}
    \end{align*}
    whenever $n \ge C_{K_1, K_2, \sigma_H} \left(\log\left(\frac{2p}{\delta}\right) + \frac{p_\star}{\mu_\star}R^2\log\left(\frac{e}{\delta}\right)\right)$.
\end{proposition}
\begin{proof}
    Define
    \begin{equation*}
    \begin{split}
        r_n &:= \sqrt{CK_1^2 \log{(2e/\delta)} \frac{p_\star}{n}} \\
        t_n &:= \frac{2\sigma_H^2}{-K_2 + \sqrt{K_2^2 + 2\sigma_H^2 n / \log{(4p/\delta)}}}.
    \end{split}
    \end{equation*}
    Note that they both decays as $O(n^{-1/2})$.
    In the following of the proof, we assume that $n \ge \max\{ 4(K_2+3\sigma^2_H)\log(4p/\delta), CK_1^2\log(2e/\delta)p_\star R^2/\mu_\star \}$.
    According to \Cref{lem:spectral_norm}, it suffices to bound $\norm{H_n(\est) - H_\star}_{H_\star^{-1}}$.
    By the triangle inequality, we have
    \begin{align}\label{eq:inverse_hess_first_bound}
        \norm{H_n(\est) - H_\star}_{H_\star^{-1}} \le \underbrace{\norm{H_n(\est) - H_n(\theta_\star)}_{H_\star^{-1}}}_{A} + \underbrace{\norm{H_n(\theta_\star) - H_\star}_{H_\star^{-1}}}_{B}.
    \end{align}
    We will control these two terms separately.
    The strategy is similar to the proof of \Cref{lem:est_error}: we prove the bound on some events and control the probability of these events.
    Define
    \begin{align*}
        \event_1 &:= \left\{ \norm{S_n(\theta_\star)}_{H_\star^{-1}}^2 \le \frac{1}{n} C K^2_1\log(2e/\delta)p_\star \right\} \\
        \event_2 &:= \left\{ (1 - t_n)  H_\star\preceq  H_n(\theta_\star) \preceq (1 + t_n) H_\star \right\}.
    \end{align*}
    When $n \ge 4(K_2+2\sigma^2_H)\log(4p/\delta)$, we have $t_n \le 1/3$.
    It then follows from the proof of \Cref{lem:est_error} that
    \begin{align}\label{eq:est_error_2}
        \norm{\est - \theta_\star}_{H_\star}^2 \le \frac1n CK_1^2 \log{(2e/\delta)} p_\star
    \end{align}
    on the event $\event_1 \event_2$ and $\Prob(\event_1) \ge 1 - \delta/2$.
    
    \textbf{Step 1. Control $A$ and $B$.}
    By \eqref{eq:est_error_2}, it holds that $\norm{\est - \theta_\star}_{H_\star} \le r_n$.
    By \Cref{lem:emp_risk_self_concord} and \Cref{lem:hessian_lip}, we have
    \begin{align*}
        A = \norm{H_n(\est) - H_n(\theta_\star)}_{H_\star^{-1}}
        \le R e^{R \norm{\theta_n - \theta_\star}_2} \norm{H_n(\theta_\star)}_{H_\star^{-1}} \norm{\est - \theta_\star}_2.
    \end{align*}
    Since $\norm{\est - \theta_\star}_2 \le \mu_\star^{-1/2} r_n$ and $n \ge CK_1^2\log(2e/\delta)p_\star R^2/\mu_\star$, we have $\norm{\est - \theta_\star}_2 \le 1/R$.
    As a result,
    \begin{align*}
        A \le R e \norm{H_n(\theta_\star)}_{H_\star^{-1}} r_n / \sqrt{\mu_\star} \le 3Re r_n /(2\sqrt{\mu_\star}),
    \end{align*}
    where the last inequality follows from the event $\event_2$ and $t_n \le 1/2$.
    As for $B$, it follows from the event $\event_2$ that $B \le t_n$.
    Therefore, absorbing $3e/2$ into the constant $C$ in $r_n$, we obtain
    \begin{align*}
        \norm{H_n(\est) - H_\star}_{H_\star^{-1}}
        \le Rr_n/\sqrt{\mu_\star} + t_n.
    \end{align*}
    And it follows from \Cref{lem:spectral_norm} that
    \begin{align*}
        \norm{H_n(\est)^{-1} - H_\star^{-1}}_{H_\star} \le \frac{Rr_n/\sqrt{\mu_\star} + t_n}{1 - Rr_n/\sqrt{\mu_\star} - t_n}.
    \end{align*}
    
    \textbf{Step 2. Control the probability of $\event_1\event_2$.}
    By the matrix Bernstein inequality \Cref{thm:bernstein_matrix}, we have $\Prob(\event_2) \ge 1 - \delta/2$.
    This implies that $\Prob(\event_1\event_2) \ge 1 - \delta$ since $\Prob(\event_1) \ge 1 - \delta/2$.
\end{proof}
\begin{table}[t]
    \centering
    \renewcommand{\arraystretch}{1.5}
    \begin{adjustbox}{max width=0.99\linewidth}
    \begin{tabular}{llll}
    \toprule
   \textbf{Model} & \textbf{Data} &\textbf{Loss Function} & \textbf{Self-Concordance Parameter $R$}\\
    \hline
   Linear Regression & $x \in \mathbb{R}^p , y \in \mathbb{R}$ & $\ell(\theta,z) := \frac{1}{2}(y-\theta^\top x)^2$ & $0$\\
   Binary Logistic Regression& $x \in \mathbb{R}^p , y \in \{0,1\}$ & $\ell(\theta,z) := -\log(\sigma(y \cdot \theta^\top x))$ & $ \|x\|_2$\\
   Poisson Regression& $x \in \mathbb{R}^p , y \in \mathbb{N}$ & $\ell(\theta,z) := -y(\theta^\top x) + \exp(\theta^\top x) + \log(y!)$& $\|x\|_2$\\
   Multiclass Logistic Regression & $x \in \mathbb{R}^p , y \in \{1,...,K\}$  & $\ell(\theta,z) := \log(1+ \sum_{i=1}^{K} e^{w_i^Tx}) - \sum_{i=2}^{K} y_i(w_i^\top X)$ & $2\|x\|_2$ \\
   \bottomrule
    \end{tabular}
    \end{adjustbox}
    \caption{Examples of M-estimation for various generalized linear models and the corresponding values of the pseudo self-concordance parameter $R$. Each regression estimates a set of parameters $\theta$ based on input values $x$ and output values $y$. 
    }
\end{table}

\section{Linearization Error Bound}
\label{sec:linearization}

We control in this section the linearization error in \Cref{thm:linearization_paper}.

\subsection{Setup}
\label{sub:linear_setup}

Recall that
\begin{align*}
    \theta_n := \argmin_{\theta \in \Theta} \left[ F_n(\theta) := \frac1n \sum_{i=1}^n \ell(Z_i, \theta) \right]
\end{align*}
and
\begin{align*}
    \theta_{n, \varepsilon, z} := \argmin_{\theta \in \Theta} \left[ (1 - \varepsilon) F_n(\theta) + \varepsilon \ell(z, \theta) \right].
\end{align*}

Since $z$ is a fixed data point, we make the following boundedness assumptions at $z$ in addition to \Cref{assm:pseudo_self_conc,assm:sub_gaussian_grad,assm:matrix_berstein}. 

\begin{assumption}[Bounded Gradient at $z$]\label{assm:linear_bounded_gradient}
    The normalized gradient at $z$ is bounded in a neighborhood of $\theta_\star$, i.e., there exist $M_1 \ge 1, \rho \in (0, R^{-1}]$ such that $\mynorm{\nabla \ell(z, \theta)}_{H_\star^{-1}} \le M_1$ for all $\mynorm{\theta - \theta_\star}_{H_\star} \le \rho$.
\end{assumption}

\begin{assumption}[Bounded Hessian at $z$]\label{assm:linear_bounded_hessian}
    The normalized Hessian at $z$ is bounded in a neighborhood of $\theta_\star$, i.e., there exist $M_2 \ge 1, \rho \in (0, R^{-1}]$ such that $\mynorm{H(z, \theta)}_{H_\star^{-1}} \le M_2$ for all $\mynorm{\theta - \theta_\star}_{H_\star} \le \rho$.
\end{assumption}

\textbf{Remark.} When the Hessian $H(z, \theta)$ is well-defined, we know $\nabla \ell(z, \cdot)$ is continuous and thus \Cref{assm:linear_bounded_gradient} is satisfied automatically.

\subsection{Proof of the Linearization Error Bound}
\label{sub:linear_proof}

\begin{customthm}{\ref{thm:linearization_paper}'}
    Under \Cref{assm:pseudo_self_conc,assm:sub_gaussian_grad,assm:matrix_berstein,assm:linear_bounded_gradient,assm:linear_bounded_hessian}, it holds that, with probability at least $1 - \delta$,
    \begin{align*}
        \mynorm{\frac{\theta_{n,\varepsilon,z} - \theta_n}{\varepsilon} - I_n(z)}_{H_n(\theta_n)}
        \leq \frac{\sqrt{2} M_1 \left((1 - \varepsilon) (e^{R \Ccal_n} - 1) + \varepsilon(2M_2 + 1) \right)}{1 - (1 - \varepsilon) (e^{R \Ccal_n} - 1) - \varepsilon(2M_2 + 1) } ,
    \end{align*}
    where $\Ccal_{n} := C \mu_\star^{-1/2} \left[ K_1 \sqrt{p_\star \log{\frac{e}{\delta}} / n} + \varepsilon M_1/(1-\varepsilon) \right]$, whenever $\varepsilon \le \min\{\rho / (C M_1 + \rho), C/M_2, \sqrt{\mu_\star} / (\sqrt{\mu_\star} + 8RM_1)\}$ and
    \begin{align*}
      n \ge \max\left\{ 8(K_2 + 4\sigma_H^2) \log{\frac{4p}\delta}, \frac{CK_1^2p_\star R^2}{\min\{\mu_\star, \rho^2 R^2\}} \log{\frac{e}\delta} \right\}.
    \end{align*}
\end{customthm}
\begin{proof}
    The proof is inspired by \citet{giordano2019swiss}.
    By the optimality of $\theta_{n,\varepsilon, z}$, it holds that
    \begin{align*}
        (1 - \varepsilon) \nabla F_n(\theta_{n, \varepsilon, z}) + \varepsilon \nabla \ell(z, \theta_{n,\varepsilon,z}) = 0.
    \end{align*}
    Define $\bar H_n(\theta) := \int_0^1 H_n(\theta_n + t(\theta - \theta_n)) \D t$ and $\bar H(z, \theta) := \int_0^1 H(z, \theta_n + t(\theta - \theta_n)) \D t$, where $H(z, \theta) := \grad^2 \ell(z, \theta)$.
    It follows from the Integral form of the Remainder of Taylor's theorem (defined in \Cref{sec:tools}) that
    \begin{align*}
        (1 - \varepsilon) \bar H_n(\theta_{n,\varepsilon,z}) (\theta_{n,\varepsilon,z} - \theta_n) + \varepsilon \nabla \ell(z, \theta_n) + \varepsilon \bar H(z, \theta_{n,\varepsilon,z}) (\theta_{n,\varepsilon,z} - \theta_n)=0,
    \end{align*}
    where we have used $\grad F_n(\theta_n) = 0$. 
    This implies that
    \begin{align*}
        \theta_{n,\varepsilon,z} - \theta_n = -\left[ (1 - \varepsilon) \bar H_n(\theta_{n,\varepsilon,z}) + \varepsilon \bar H(z, \theta_{n,\varepsilon,z}) \right]^{-1} \varepsilon \nabla \ell(z, \theta_n),
    \end{align*}
    and thus
    \begin{align*}
        &\quad \mynorm{\frac{\theta_{n,\varepsilon,z} - \theta_n}{\varepsilon} - I_n(z)}_{H_n(\theta_n)} \\
        &= \mynorm{\left\{ \left[ (1 - \varepsilon) \bar H_n(\theta_{n,\varepsilon,z}) + \varepsilon \bar H(z, \theta_{n,\varepsilon,z}) \right]^{-1} - H_n(\theta_n)^{-1} \right\} \nabla \ell(z, \theta_n)}_{H_n(\theta_n)} \\
        &= \mynorm{\left\{ H_n(\theta_n)^{1/2} \left[ (1 - \varepsilon) \bar H_n(\theta_{n,\varepsilon,z}) + \varepsilon \bar H(z, \theta_{n,\varepsilon,z}) \right]^{-1} H_n(\theta_n)^{1/2} - \id_p \right\} H_n(\theta_n)^{-1/2}\nabla \ell(z, \theta_n)}_2 \\
        &\le \mynorm{H_n(\theta_n)^{1/2} \left[ (1 - \varepsilon) \bar H_n(\theta_{n,\varepsilon,z}) + \varepsilon \bar H(z, \theta_{n,\varepsilon,z}) \right]^{-1} H_n(\theta_n)^{1/2} - \id_p}_2 \mynorm{H_n(\theta_n)^{-1/2}\nabla \ell(z, \theta_n)}_2 \\
        &= \underbrace{\mynorm{\left[ (1 - \varepsilon) \bar H_n(\theta_{n,\varepsilon,z}) + \varepsilon \bar H(z, \theta_{n,\varepsilon,z}) \right]^{-1} - H_n(\theta_n)^{-1}}_{H_n(\theta_n)}}_{A_1} \underbrace{\mynorm{H_n(\theta_n)^{-1}\nabla \ell(z, \theta_n)}_{H_n(\theta_n)}}_{A_2}.
    \end{align*}

    Recall $r_n$ and $t_n$ from \eqref{eq:rn_tn}.
    To proceed, we define the following events
    \begin{align*}
      \Gcal_1 &:= \left\{ \norm{S_n(\theta_\star)}_{H_\star^{-1}}^2 \le \frac1n CK_1^2 \log{(e/\delta)} p_\star \right\} \\
      \Gcal_2 &:= \left\{\frac12 H_\star \preceq H_n(\theta_\star) \preceq \frac32 H_\star \right\} \\
      \Gcal_3 &:= \left\{ \norm{H_\star^{1/2} H_n(\est)^{-1} H_\star^{1/2} - \id_p}_2 \le \frac{Rr_n/\sqrt{\mu_\star} + t_n}{1 - Rr_n/\sqrt{\mu_\star} - t_n} \right\}.
    \end{align*}
    Moreover, we assume $\varepsilon \le \min\{\rho / (C M_1 + \rho), C/M_2, \sqrt{\mu_\star} / (\sqrt{\mu_\star} + 8RM_1)\}$ and
    \begin{align*}
      n \ge \max\left\{ 8(K_2 + 4\sigma_H^2) \log{\frac{4p}\delta}, \frac{CK_1^2p_\star R^2}{\min\{\mu_\star, \rho^2 R^2\}} \log{\frac{e}\delta} \right\}.
    \end{align*}
    throughout the proof.
    Note that $Rr_n/\sqrt{\mu_\star} + t_n \le 1/2$ under this requirement of $n$.
    Recall from the proof of \Cref{lem:est_error}, \Cref{prop:bound_inverse_hess}, and \Cref{prop:perturb} that $\Prob(\Gcal_1 \Gcal_2 \Gcal_3) \ge 1 - \delta$ and
    \begin{equation}\label{eq:est_error_thetas}
    \begin{split}
      \norm{\theta_n - \theta_\star}_{H_\star}^2 &\le CK_1^2 \frac{p_\star}{n} \log{\frac{e}{\delta}} \\
      \norm{\theta_{n,\varepsilon,z} - \theta_\star}_{H_\star}^2 &\le C K_1^2 \frac{p_\star}{n} \log{\frac{e}{\delta}} + \frac{128\varepsilon^2}{(1 - \varepsilon)^2} M_1^2.
    \end{split}
    \end{equation}
    Therefore, it suffices to bound $A_1$ and $A_2$ on the event $\Gcal_1 \Gcal_2 \Gcal_3$.

    \textbf{Step 1. Bound $A_1$.}
    We will  use \Cref{lem:spectral_norm} to bound $A_1$. We define
    \begin{align*}
        B
        &:= \mynorm{(1 - \varepsilon) \bar H_n(\theta_{n,\varepsilon,z}) + \varepsilon \bar H(z, \theta_{n,\varepsilon,z}) - H_n(\theta_n)}_{H_n(\theta_n)^{-1}} \\
        &\le (1 - \varepsilon) \underbrace{\mynorm{\bar H_n(\theta_{n,\varepsilon,z}) - H_n(\theta_n)}_{H_n(\theta_n)^{-1}}}_{B_1} + \varepsilon \underbrace{\mynorm{\bar H(z, \theta_{n,\varepsilon,z}) - H_n(\theta_n)}_{H_n(\theta_n)^{-1}}}_{B_2}.
    \end{align*}
    We first bound $B_1$.
    By Jensen's inequality, we get
    \begin{align*}
        B_1 &\le \int_0^1 \mynorm{H_n(\theta_n + t(\theta_{n,\varepsilon,z} - \theta_n)) - H_n(\theta_n)}_{H_n(\theta_n)^{-1}} \D t\\
        &= \int_0^1 \mynorm{H_n(\theta_n + t(\theta_{n,\varepsilon, z} - \theta_n))}_{H_n(\theta_n)^{-1}}\D t + 1.
    \end{align*}
    By \Cref{lem:emp_risk_self_concord} and \Cref{prop:hessian}, it holds that
    \begin{align*}
        e^{-R t\norm{\theta_{n,\varepsilon,z} - \theta_n}_2} H_n(\theta_n) \preceq H_n(\theta_n + t(\theta_{n,\varepsilon,z} - \theta_n)) \preceq e^{R t\norm{\theta_{n,\varepsilon,z} - \theta_n}_2} H_n(\theta_n).
    \end{align*}
    It then follows from \Cref{prop:perturb} and $t \in [0, 1]$ that
        \begin{align*}
        e^{-R \Ccal_n} H_n(\theta_n) \preceq H_n(\theta_n + t(\theta_{n,\varepsilon,z} - \theta_n)) \preceq e^{R\Ccal_n}H_n(\theta_n),
    \end{align*}
    where $\Ccal_{n} := C \mu_\star^{-1/2} \left[ K_1 \sqrt{p_\star \log{\frac{e}{\delta}} / n} + \varepsilon M_1/(1-\varepsilon) \right]$.
    Since $1 - e^{-x} \le e^x - 1$ for all $x \ge 0$, we get
    \begin{align*}
        B_1 \le e^{R\Ccal_{n}} - 1.
    \end{align*}
    
    We then bound $B_2$.
    We start the same as before using Jensen's inequality, we get
    \begin{align*}
        B_2 \le \int_0^1 \mynorm{H(z, \theta_n + t(\theta_{n, \varepsilon,z} - \theta_n)) - H_n(\theta_n)}_{H_n(\theta_n)^{-1}} \D t.
    \end{align*}
    Using the triangle inequality we can write
    \begin{align*}
        B_2
        &\le \int_0^1 \left[ \mynorm{H(z, \theta_n + t(\theta_{n,\varepsilon, z} - \theta_n))}_{H_n(\theta_n)^{-1}}  + \norm{H_n(\theta_n)}_{H_n(\theta_n)^{-1}} \right] \D t \\
        &= \int_0^1 \mynorm{H(z, \theta_n + t(\theta_{n,\varepsilon, z} - \theta_n))}_{H_n(\theta_n)^{-1}}\D t + 1.
    \end{align*}
    Then it follows from the event $\event_3$ and the requirement of $n$ that
        \begin{align*}
        B_2
        &\le \frac{1}{1 - Rr_n/\sqrt{\mu_\star} - t_n} \int_0^1 \mynorm{H(z, \theta_n + t(\theta_{n,\varepsilon, z} - \theta_n))}_{H_\star^{-1}} \D t + 1 \\
        &\le 2\int_0^1 \mynorm{H(z, \theta_n + t(\theta_{n,\varepsilon, z} - \theta_n))}_{H_\star^{-1}} \D t + 1
    \end{align*}
    Since $\norm{\theta_n + t(\theta_{n, \varepsilon, z} - \theta_n)- \theta_\star}_{H_\star} \leq \max\{\norm{\theta_n - \theta_\star}_{H_\star}, \norm{\theta_{n, \varepsilon, z}- \theta_\star}_{H_\star}\}$ for $t \in [0,1]$, it follows from \Cref{prop:perturb} that
    \begin{align*}
      \norm{\theta_n + t(\theta_{n, \varepsilon, z} - \theta_n)- \theta_\star}_{H_\star} \le C \left[ K_1 \sqrt{\frac{p_\star}{n} \log{\frac{e}{\delta}}} + \frac{\varepsilon}{1-\varepsilon} M_1 \right] < \rho
    \end{align*}
    by the requirement of $n$ and $\varepsilon$.
    As a result, we have
    \begin{align*}
       \mynorm{H(z, \theta_n + t(\theta_{n,\varepsilon, z} - \theta_n))}_{H_\star^{-1}} \le M_2
    \end{align*}
    by Assumption \ref{assm:linear_bounded_hessian}.
    Combining the above results we obtain
    \begin{align*}
        B_2 \leq 2M_2 + 1,
    \end{align*}
    which implies
    \begin{align*}
        B \leq (1 - \varepsilon) (e^{R \Ccal_n} - 1) + \varepsilon (2M_2 + 1)  \leq \lambda_{\min}(\id_p) = 1,
    \end{align*}
    where the last inequality holds by the requirements of $n$ and $\varepsilon$.
    
    Hence, applying \Cref{lem:spectral_norm} to $H_n(\theta_n)^{-1/2} [(1-\varepsilon) \bar H_n(\theta_{n,\varepsilon,z}) + \varepsilon \bar H(z, \theta_{n, \varepsilon, z})] H_n(\theta_n)^{-1/2}$ and $\id_p$ yields
    \begin{align*}
        A_1 \leq \frac{(1 - \varepsilon) (e^{R \Ccal_n} - 1) + \varepsilon(2M_2 + 1) }{1 - (1 - \varepsilon) (e^{R \Ccal_n} - 1) - \varepsilon (2M_2 + 1) }.
    \end{align*}
    \textbf{Step 2. Bound $A_2$.}
    By the event $\event_3$ and the requirement of $n$, we have (similar to the bound of $B_2$)
    \begin{align*}
        A_2
        = \norm{\nabla \ell(z, \theta_n)}_{H_n(\theta_n)^{-1}} \le \sqrt{2} \norm{\nabla \ell(z, \theta_n)}_{H_\star^{-1}}.
    \end{align*}
    By \eqref{eq:est_error_thetas} and the requirement of $n$, it holds that $\norm{\theta_n - \theta_\star}_{H_\star} < \rho$ and thus, by Assumption \ref{assm:linear_bounded_gradient},
    \begin{align*}
        A_2
        \le \sqrt{2} M_1.
    \end{align*}
    \textbf{Step 3. Combine the bounds of $A_1$ and $A_2$.}
    Combining the bounds for $A_1$ and $A_2$ we arrive at the final result,
    \begin{align*}
        \mynorm{\frac{\theta_{n,\varepsilon,z} - \theta_n}{\varepsilon} - I_n(z)}_{H_n(\theta_n)}
        \leq \frac{\sqrt{2} M_1\left((1 - \varepsilon) (e^{R \Ccal_n} - 1) + \varepsilon (2M_2 + 1) \right )}{1 - (1 - \varepsilon) (e^{R \Ccal_n} - 1) - \varepsilon(2M_2 + 1) }.
    \end{align*}
\end{proof}

\subsection{Intermediate Results}
\label{sub:linear_intermediate}

The proof of \Cref{thm:linearization_paper} relies on a key result: the perturbed estimator $\theta_{n,\varepsilon,z}$ is close to $\theta_n$ stated in \Cref{prop:perturb}.

\begin{proposition}\label{prop:perturb}
    Under \Cref{assm:pseudo_self_conc,assm:sub_gaussian_grad,assm:matrix_berstein,assm:linear_bounded_gradient,assm:linear_bounded_hessian}, it holds that
    \begin{align*}
      \norm{\theta_{n,\varepsilon,z} - \theta_n}_{H_\star}^2 \le C K_1^2 \frac{p_\star}{n} \log{\frac{e}{\delta}} + \frac{128\varepsilon^2}{(1 - \varepsilon)^2} M_1^2,
    \end{align*}
    whenever $\varepsilon \le \sqrt{\mu_\star} / (\sqrt{\mu_\star} + 8RM_1)$ and
    \begin{align*}
      n \ge \max\left\{ 4(K_2 + 2\sigma_H^2) \log{\frac{4p}\delta}, \frac{CK_1^2p_\star R^2}{\mu_\star} \log{\frac{e}\delta} \right\}.
    \end{align*}
  \end{proposition}
  \begin{proof}
    By the triangle inequality, we have
    \begin{align*}
      \norm{\theta_{n,\varepsilon,z} - \theta_n}_{H_\star} \le \norm{\theta_{n,\varepsilon,z} - \theta_\star}_{H_\star} + \norm{\theta_{n} - \theta_\star}_{H_\star}.
    \end{align*}
    It remains to control $\norm{\theta_{n,\varepsilon,z} - \theta_\star}_{H_\star}$ and $\norm{\theta_{n} - \theta_\star}_{H_\star}$.
    The second term is controlled by \Cref{lem:est_error}.
    We will control the first term with a similar argument.

    We define two events
    \begin{align*}
      \Gcal_1 &:= \left\{ \norm{S_n(\theta_\star)}_{H_\star^{-1}}^2 \le \frac1n CK_1^2 \log{(e/\delta)} p_\star \right\} \\
      \Gcal_2 &:= \left\{\frac12 H_\star \preceq H_n(\theta_\star) \preceq \frac32 H_\star \right\},
    \end{align*}
    and assume that $\varepsilon \le \sqrt{\mu_\star} / (\sqrt{\mu_\star} + 8RM_1)$ and
    \begin{align*}
      n \ge \max\left\{ 4(K_2 + 2\sigma_H^2) \log{\frac{4p}\delta}, \frac{CK_1^2p_\star R^2}{\mu_\star} \log{\frac{e}\delta} \right\}.
    \end{align*}
    It follows from \Cref{lem:est_error} that $\Prob(\Gcal_1 \Gcal_2) \ge 1 - \delta$ and
    \begin{align*}
      \norm{\theta_n - \theta_\star}_{H_\star}^2 \le C K_1^2 \frac{p_\star}{n} \log{\frac{e}{\delta}}.
    \end{align*}

    We then control $\norm{\theta_{n,\varepsilon,z} - \theta_\star}_{H_\star}$ on the event $\Gcal_1 \Gcal_2$.
    Following the proof of \Cref{lem:emp_risk_self_concord}, we know that $(1 - \varepsilon) F_n(\cdot) + \varepsilon \ell(z, \cdot)$ is pseudo self-concordant with parameter $R$.
    Let
    \begin{align*}
      S_{n,\varepsilon,z}(\theta) := (1 - \varepsilon) S_n(\theta) + \varepsilon S(z, \theta) \quad \mbox{and} \quad H_{n,\varepsilon,z}(\theta) := (1 - \varepsilon) H_n(\theta) + \varepsilon H(z, \theta).
    \end{align*}
    Since we assume $\ell(z,\theta)$ is convex then $H(z,\theta) \succeq 0$. Then, by the event $\Gcal_2$, we have
    \begin{align*}
      H_{n,\varepsilon,z}(\theta_\star) \succeq \left( \frac{1 - \varepsilon}2 \right) H_\star.
    \end{align*}
    As a result, it holds that
    \begin{align*}
      \norm{S_{n,\varepsilon,z}(\theta_\star)}_{H_{n,\varepsilon,z}(\theta_\star)^{-1}}
      &\le \sqrt{\frac{2}{1 - \varepsilon}} \norm{S_{n,\varepsilon,z}(\theta_\star)}_{H_\star^{-1}} \\
      &\le \sqrt{\frac{2}{1 - \varepsilon}} \left[ (1 - \varepsilon)\norm{S_n(\theta_\star)}_{H_\star^{-1}} + \varepsilon \norm{S(z, \theta_\star)}_{H_\star^{-1}} \right].
    \end{align*}
    By \Cref{assm:linear_bounded_gradient}, we obtain
    \begin{align*}
      \norm{S_{n,\varepsilon,z}(\theta_\star)}_{H_{n,\varepsilon,z}(\theta_\star)^{-1}}
      \le \sqrt{\frac{2}{1 - \varepsilon}} \left[ (1 - \varepsilon)\norm{S_n(\theta_\star)}_{H_\star^{-1}} + \varepsilon M_1 \right]
    \end{align*}
    Since $\sqrt{\lambda_{\min}(H_{n,\varepsilon,z}(\theta_\star))} \ge \sqrt{(1 - \varepsilon) \mu_\star/2}$, it follows from the event $\Gcal_1$ and the requirement of $n$ that
    \begin{align*}
      \norm{S_{n,\varepsilon,z}(\theta_\star)}_{H_{n,\varepsilon,z}(\theta_\star)^{-1}} \le \frac{\sqrt{\lambda_{\min}(H_{n,\varepsilon,z}(\theta_\star))}}{2R}.
    \end{align*}
    According to \Cref{prop:f_localization}, $\theta_{n,\varepsilon,z}$ uniquely exists and satisfies
    \begin{align*}
        \norm{\theta_{n,\varepsilon,z} - \theta_\star}_{H_{n,\varepsilon,z}(\theta_\star)}^2
        \le 16 \norm{S_{n,\varepsilon,z}(\theta_\star)}_{H_{n,\varepsilon,z}(\theta_\star)^{-1}}^2
        \le \frac{64}{1-\varepsilon}\left[ (1-\varepsilon)^2 C K_1^2 \frac{p_\star}{n} \log{\frac{e}{\delta}} + \varepsilon^2 M_1^2 \right],
    \end{align*}
    which implies
    \begin{align}\label{eq:dist_perturb_star}
      \norm{\theta_{n,\varepsilon,z} - \theta_\star}_{H_\star}^2
      \le C K_1^2 \frac{p_\star}{n} \log{\frac{e}{\delta}} + \frac{128\varepsilon^2}{(1 - \varepsilon)^2} M_1^2.
    \end{align}
  \end{proof} %
\section{Computational Error Bounds}
\label{appx:computation_bound}

We analyze the computation error of the algorithms discussed in \Cref{sec:bg} used to compute the empirical influence function.
Throughout, we assume that 
the target precision satisfies $\eps \le \norm{I(z)}_{H_\star}^2$. 
If not, taking $\hat I_n(z) = 0$ satisfies the desired precision and there is nothing to do. 

\myparagraph{Condition Numbers}
Throughout, we assume that the loss function $\ell(\cdot, z)$ is $L$-smooth for each $Z$ and that
$H_n(\est)$ is invertible. Let $\mu_n = \lambda_{\min}(H_n(\theta_n))$ denote the minimal eigenvalue.
The computational bounds depend on the condition number 
\[
    \kappa_n := \frac{L}{\mu_n} \,.
\] 
The corresponding population condition number is
\[
    \kappa_\star = \frac{L}{\mu_\star} \,,
\]
where $\mu_\star = \lambda_{\min}(H_\star)$. They are related as follows. 

\myparagraph{$K$-Condition Numbers}
Another useful notion to obtain the convergence rate of the conjugate gradient method is the $K$-condition number defined as
\begin{align*}
    K_n := \frac{[\Tr{H_n(\theta_n)}/p]^p}{\det{H_n(\theta_n)}}.
\end{align*}
Its population counterpart is defined as
\begin{align*}
    K_\star := \frac{[\Tr{H_\star}/p]^p}{\det{H_\star}}.
\end{align*}

\begin{proposition} \label{prop:cond-num}
    Consider the setting of \Cref{thm:if-stat-main}, and let $\Gcal$ denote the event under which its conclusions hold. Under this event $\Gcal$, we have, 
    \begin{enumerate}[label=(\alph*)]
        \item $\kappa_n \le 4 \kappa_\star$, and
        \item if $\norm{I_n(z) - I(z)}_{H_\star}^2 = \eps$, then 
        $\norm{I_n(z)}_{H_n(\theta_n)}^2 \le 6 \norm{I(z)}_{H_\star}^2 + 6 \eps$.
    \end{enumerate}
\end{proposition}
\begin{proof}
    We have under $\Gcal$ that $(1/4) H_\star \preceq H_n(\theta_n) \preceq 3 H_\star$.
    This implies that $\mu_n \ge \mu_\star / 4$, $\Tr{H_n(\theta_n)} \le 3\Tr{H_\star}$, and $\det{H_n(\theta_n)} \ge \det{H_\star}/4^p$.
    For the second part, we get from the triangle inequality, 
    \begin{align*}
    \norm{I_n(z)}_{H_n(\theta_n)}^2 
    \le 3 \norm{I_n(z)}_{H_\star}^2
     \le 6 \norm{I(z)}_{H_\star}^2 + 6 \norm{I_n(z) - I(z)}_{H_\star}^2 \,.
    \end{align*}
\end{proof}

\subsection{Total Error}
\label{appx:total_error}

We combine the computational error with the statistical error to get the total error bound.
This is a restatement of \Cref{prop:total_error} of the main paper. 

\begin{proposition} \label{prop:a:total_error}
    Consider the setting of \Cref{thm:if-stat-main}, and let $\Gcal$ denote the event under which its conclusions hold. 
    Let $\hat I_n(\theta)$ be an estimate of $I_n(\theta)$ that satisfies $\expect\left[\norm{\hat I_n(z) - I_n(z)}_{H_n(\theta_n)}^2 \middle| Z_{1:n} \right] \le \eps$. 
    Then, we have, 
    \[
        \expect\left[ \norm{\hat I_n(z) - I(z)}_{H_\star}^2 \, \middle| \, \Gcal \right]
        \le 8 \eps + 
         C_{K_1, K_2, \sigma_H} \frac{R^2 p_\star^2}{\mu_\star n} \poly\log\frac{p}{\delta}  \,,
    \]
    whenever $n \ge C_{K_1, K_2, \sigma_H} \left( \frac{p_\star}{\mu_\star} R^2 \log\left(\frac{e}{\delta}\right) + \log\left(\frac{2p}{\delta}\right)\right)$.
\end{proposition}
\begin{proof}
    Following the proof of \Cref{thm:if-stat-main}, we have under $\Gcal$ that 
    \[
        \frac{1}{4}H_\star \preceq H_n(\theta_n) \preceq 3 H_\star \,.
    \]
    Therefore, $\norm{u}_{H_\star}^2 \le 4 \norm{u}_{H_n(\theta_n)}^2$. Combining this with the triangle inequality completes the proof. 
\end{proof}

\subsection{The Conjugate Gradient Method} \label{sec:comp:cg}

We start by recalling the convergence analysis of the conjugate gradient method, providing a full proof for completeness. 
\begin{proposition}\label{prop:cgd_bound}
    Consider the sequence $(u_t)$ produced by the conjugate gradient method for solving $u_\star = H_n(\est)^{-1} S(z, \est)$.
    It holds that
    \begin{align*}
        \norm{u_t - u_\star}_{H_n(\est)}^2 \le 4 \left( \frac{\sqrt{\kappa_n} - 1}{\sqrt{\kappa_n} + 1} \right)^{2t} \norm{u_0 - u_\star}_{H_n(\est)}^2.
    \end{align*}
    In other words, we get $\norm{u_t - u_\star}_{H_n(\theta_n)}^2 \le \eps$ after $t_{\mathrm{cg}}$ iterations, where 
    \[
        t_{\mathrm{cg}} \le \frac{\sqrt{\kappa_n}}{2} \log \left(\frac{4 \norm{u_0 - u_\star}_{H_n(\theta_n)}^2}{\eps}\right) \,.
    \]
\end{proposition}

\begin{proof}
    We follow the proof template of \citet[Chapter 3.4]{chen2005matrix}.
    Throughout, we use the shorthand $A = H_n(\theta_n)$.
    By construction, we have $u_{k} \in \text{Span}\{p_0, \dots, p_{k-1}\}$.
    It then follows from $p_{k} = r_{k} + \beta_{k-1} p_{k-1}$ that $\text{Span}\{p_0, \dots, p_{k-1}\} = \text{Span}\{r_0, \dots, r_{k-1}\}$.
    Moreover, since $r_{k} = b - Au_k = r_{k-1} - \alpha_{k-1} Ap_{k-1}$, we get
    \begin{align*}
        \text{Span}\{r_0, \dots, r_{k-1}\} = \text{Span}\{r_0, Ar_0, \dots, A^{k-1}r_0\} =: \mathcal{K}_{k}(A, r_0),
    \end{align*}
    where $\mathcal{K}_{k}(A, r_0)$ is known as the Krylov subspace of order $k$ for the matrix $A$ and the generating vector $r_0$.
    Since $u_0 = 0$, it holds that $r_0 = b = Au_\star$ and thus
    \begin{align*}
        \mathcal{K}_k(A, r_0) = \text{Span}\{b, Ab, \dots, A^{k-1}b\}.
    \end{align*}
    We will write $\mathcal{K}_k$ for short.
    
    For an arbitrary $x \in \Kcal_k$, there exists $\{\alpha_i\}_{i=0}^{k-1}$ such that $x = \sum_{i=0}^{k-1} \alpha_i A^i b$.
    Let $f(t) := \sum_{i=0}^{k-1} \alpha_i t^i$.
    It follows that
    \begin{align*}
        \norm{u - u_\star}_A^2 = (f(A)Au_\star - u_\star)^\top A (f(A)Au_\star - u_\star) = u_\star^\top g(A) A g(A) u_\star,
    \end{align*}
    where $g(t) := 1 - f(t) t$ and $A = A^\top$ has been used.
    Since $A$ is positive semi-definite, it admits an eigenvalue decomposition $A = Q \Lambda Q^\top$.
    It then follows from $A^k = Q \Lambda^k Q$ that
    \begin{align*}
        u_\star^\top g(A) A g(A) u_\star = u_\star^\top Q g(\Lambda) \Lambda g(\Lambda) Q^\top u_\star.
    \end{align*}
    Denote $y := Q^\top u_\star$ and $\Lambda = \diag\{\lambda_j\}$.
    Then we get
    \begin{align*}
        u_\star^\top Q g(\Lambda) \Lambda g(\Lambda) Q^\top u_\star = \sum_{j=1}^p \lambda_j g(\lambda_j)^2 y_j^2.
    \end{align*}
    
    Note that
    \begin{align*}
        \norm{u - u_\star}_A^2 = u^\top A u - 2 u^\top Au_\star + u_\star^\top A u_\star = u^\top A u - 2 u^\top b + u_\star^\top A u_\star
    \end{align*}
    According to \citet[Equation 3.31]{chen2005matrix},
    \begin{align*}
        \norm{u_k - u_\star}_A^2 = \min_{x \in \text{Span}\{p_0, \dots, p_{k-1}\}} \norm{x - u_\star}_A^2 = \min_{g \in \Gcal_k} \sum_{j=1}^p \lambda_j g(\lambda_j)^2 y_j^2,
    \end{align*}
    where $\Gcal_k$ is the collection of polynomials of degree $k$ that take value $1$ at $0$.
    Define
    \begin{align*}
        C(\Lambda) := \min_{g \in \Gcal_k} \max_{j \in [p]} \abs{g(\lambda_j)}.
    \end{align*}
    Using properties of Chebyshev polynomials, we obtain \citep[e.g.,][Equation 3.46]{chen2005matrix}
    \begin{align*}
        C(\Lambda) \le 2 \left( \frac{\sqrt{\kappa} - 1}{\sqrt{\kappa} + 1} \right)^k,
    \end{align*}
    where $\kappa := \lambda_{\max}(A) / \lambda_{\min}(A)$.
    As a result,
    \begin{align*}
        \norm{u_k - u_\star}_A^2 &\le \min_{g \in \Gcal_k} \sum_{j=1}^p \lambda_j \max_{j' \in [p]} g(\lambda_{j'})^2 y_j^2 = C(\Lambda)^2 \sum_{j=1}^p \lambda_j y_j^2 = C(\Lambda)^2 y^\top \Lambda y = C(\Lambda)^2 u_\star^\top A u_\star \\
        &\le 4\left( \frac{\sqrt{\kappa} - 1}{\sqrt{\kappa} + 1} \right)^{2k} \norm{u_0 - u_\star}_A^2.
    \end{align*}
    We use the bound $\kappa \le \kappa_n$ to complete the proof. 
\end{proof}

\begin{corollary}[Total Computational Cost; Conjugate Gradient Method]
\label{cor:total:cg}
    Fix $\eps > 0$. Consider the setting of \Cref{thm:if-stat-main}, and let $\Gcal$ denote the high probability event under which its conclusions hold.
    Choose a sample size $n$ such that 
    \[
        n = C_{K_1, K_2, \sigma_H}   \frac{R^2 p_\star^2 }{\mu_\star \eps} \poly\log \frac{p}{\delta} \,.
    \]
    Then, under $\Gcal$, the number $N_{\mathrm{cg}}$ of gradient and Hessian-vector oracle calls required to obtain a point $\hat I_n(z)$ using the conjugate gradient method initialized at $u_0 = 0$ such that 
    $\norm{\hat I_n(z) - I(z)}_{H_\star}^2 \le \eps$ is bounded by 
    \[
        N_{\mathrm{cg}} \le 
        C_{K_1, K_2, \sigma_H}  \,\,
        \frac{R^2 p_\star^2 \kappa_\star^{3/2}}{L \eps} \,  \, \log\left(\frac{\norm{I(z)}_{H_\star}^2}{\eps} + 1\right) \,
         \poly\log \frac{p}{\delta}
         \,.
    \]
\end{corollary}
\begin{proof}
    We combine the total error bound of \Cref{prop:a:total_error} with the computational bound of \Cref{prop:cgd_bound}. 
    Under $\Gcal$, note that the choice of the sample size $n$ implies that the statistical error is bounded from \Cref{thm:if-stat-main} by 
    \[
        \norm{I_n(z) - I(z)}_{H_\star}^2 \le \frac{\eps}{2}. 
    \]
    Let $t_{\mathrm{cg}}$ be the number of conjugate gradient iterations $t$ such that the $\norm{\hat I_n(z) - I_n(z)}_{H_n(\theta_n)}^2 \le \eps / 16$ as given in \Cref{prop:cgd_bound}. By \Cref{prop:a:total_error}, the total error is then $\eps$ and the total number of gradient and Hessian-vector product oracle calls in $N = t_{\mathrm{cg}} n$, since each iteration requires a full pass over the data. 
    To complete the proof, we invoke \Cref{prop:cond-num} to bound the initial gap $\norm{u_0 - u_\star}^2_{H_n(\theta_n)} = \norm{I_n(z)}_{H_n(\theta_n)}$ and the condition number $\kappa_n$ in terms of their respective population quantities.
\end{proof}

\begin{remark}
    When the spectrum of $H_\star$ decays as $O(i^{-\beta})$ for $\beta \in [0, 1)$, we can obtain a more refined analysis using the K-condition number.
    In the following, we assume that $p > 1$ and
    \begin{align*}
        n \ge C_{K_1, K_2, \sigma_H} (p^2 + \varepsilon^{-1}) R^2 \frac{p_\star}{\mu_\star} \poly\log \frac{p}{\delta}.
    \end{align*}
    
    Following the proof of \Cref{prop:cgd_bound}, it holds that
    \begin{align*}
        \norm{u_t - u_\star}_A^2 \le C^2(\Lambda) \norm{u_0 - u_\star}_A^2.
    \end{align*}
    According to \citet[Theorem 4.3]{axelsson2000sublinear}, we have
    \begin{align*}
        C(\Lambda) \le \left( \frac{3 \log{K_n}}{t} \right)^{t/2}.
    \end{align*}
    Using the event $\event_4$ from the proof of \Cref{thm:if-stat-main}, we know that $(1 - p^{-1}) H_\star \preceq H_n(\theta_n) \preceq (1 + p^{-1}) H_\star$.
    As a result, we have $K_n \le (1 + p^{-1})^p (1 - p^{-1})^{-p} K_\star \le C K_\star$.
    Moreover, it follows from \Cref{thm:if-stat-main} that the statistical error is controlled by $\varepsilon / 2$.
    
    We then control the computational error.
    Since $\lambda_i \sim i^{-\beta}$, we have $\Tr{H_\star} \sim p^{1-\beta}/(1 - \beta)$ and $\det{H_\star} \sim (p!)^{-\beta}$.
    Consequently, it follows from Stirling's approximation that $K_\star \sim (2\pi p)^{\beta/2} e^{-\beta p} (1 - \beta)^{-p}$.
    If $t > 6\log{(C K_\star)} > 6 \log{K_n}$, then we only need $t > C \log{\left( 1 + \frac{\norm{I(z)}_{H_\star}^2}{\varepsilon} \right)}$ to achieve $\varepsilon/2$ computation error.
    Therefore, we have
    \begin{align*}
        t_{\text{cg}} \gtrsim 6\log\left[ C (2\pi p)^{\beta/2} e^{-\beta p} (1 - \beta)^{-p} \right] + C \log{\left( 1 + \frac{\norm{I(z)}_{H_\star}^2}{\varepsilon} \right)},
    \end{align*}
    and thus
    \begin{align*}
        N_{\text{cg}} \sim C_{K_1, K_2, \sigma_H} (p^2 + \varepsilon^{-1}) R^2 \frac{p_\star}{\mu_\star} \left\{ 6\log\left[ C (2\pi p)^{\beta/2} e^{-\beta p} (1 - \beta)^{-p} \right] + C \log{\left( 1 + \frac{\norm{I(z)}_{H_\star}^2}{\varepsilon} \right)} \right\} \poly\log \frac{p}{\delta}.
    \end{align*}
\end{remark}

\subsection{Stochastic Gradient Descent} \label{sec:comp:sgd}

We consider using SGD to solve the linear system $H_n(\theta_n) u + \grad \ell(z, \theta_n) = 0$. We do so by minimizing the quadratic $g_n$ from \eqref{eq:quadratic}: 
\begin{align*} %
    g_n(u) = \frac{1}{2} \inp{u}{H_n(\theta_n) u} + \inp{\grad \ell(z, \theta_n)}{u} \,.
\end{align*}
We run SGD by sampling an index $i_t$ uniformly at random to update 
\[
    u_{t+1} = u_t - \gamma\big( H(Z_{i_t}, \theta_n ) u_t + \ell(z, \theta_n)\big) \,.
\]

The bounds depend on the following quantities:
\begin{enumerate}[label=(\alph*),nolistsep,leftmargin=\widthof{ (3) }]
    \item Let $\mu_n = \lambda_{\min}(H_n(\theta_n))$ be the minimal eigenvalue of $H_n(\theta_n)$. 
    \item Define the matrix $W_n = \big(H_n(\theta_n)^{-1/2} H(Z_i, \theta_n) H_n(\theta_n)^{-1/2} - \id_p \big)$ and 
    \[
        \Sigmam_n = \frac{1}{n} \sum_{i=1}^n 
        W_n
         H_n(\theta_n)^{1/2} I_n(z) I_n(z)\T  H_n(\theta_n)^{1/2} W_n \,.
    \]  
    \item Define the noise term 
    \[
         \sigma_n^2 := \Tr \Sigmam_n  + p \norm{\Sigmam_n}_2 \,.
    \]

\end{enumerate}

We have the following convergence bound for SGD~\cite{jain2017parallelizing,jain2017markov}; cf. \Cref{sec:techn:convergence} for details.

\begin{lemma} \label{lem:sgd:computational_bound}
    The sequence $(\bar u_t)$ produced by tail-averaged SGD on the function $g_n(u)$ from \eqref{eq:quadratic} with a learning rate of $\gamma = (2 L)^{-1}$ satisfies
    \[
        \expect\norm{\bar u_t- u_\star}_{H_n(\theta_n)}^2
        \le C \left(
            \kappa_n \, \norm{u_0 - u_\star}_{H_n(\theta_n)}^2
            \exp\left(- \frac{t}{4\kappa_n}\right) + \frac{\sigma_n^2}{t}
        \right) \,.
    \]
    Therefore, it returns a point $\bar u_t$ satisfying $\expect\norm{\bar u_t- u_\star}_{H_n(\theta_n))}^2 \le \eps$ after  $t\ge t_{\mathrm{sgd}}$ steps where
    \[
        t_{\mathrm{sgd}} \le C \left( 
            \frac{\sigma_n^2}{\eps} 
            + \kappa_n \log \left( 
            \frac{\kappa_n \norm{u_0 - u_\star}^2_{H_n(\theta_n)}}{\eps}
            \right)
        \right)  \,,
    \]
    where $\kappa_n = L / \mu_n$ is the condition number. 
\end{lemma}

\myparagraph{Total Error Bound}
We give a total error bound under a stronger assumption on the normalized Hessian.
We strengthen the matrix Bernstein condition on the normalized Hessian into a spectral norm bound in a neighborhood around $\theta_\star$ as formalized below. 
\begin{customasmp}{3'}[Bounded Hessian]\label{assm:bounded_hessian_1}
    The normalized Hessian is bounded in a neighborhood of $\theta_\star$, i.e., there exist $M_2 > 1$ and $\rho > 0$ such that $\mynorm{H(z, \theta)}_{H_\star^{-1}} \le M_2$ for all $z \in \Zcal$ and $\mynorm{\theta - \theta_\star}_{H_\star} \le \rho$.
\end{customasmp}

This gives the following total error bound. 
\begin{proposition}[Total Error bound for SGD]
\label{prop:sgd-total-error}
    Fix $\eps > 0$. Consider the setting of \Cref{thm:if-stat-main} and let $\Gcal$ denote the event under which its conclusions hold. Suppose also that \Cref{assm:bounded_hessian_1} is true. 
    With probability at least $1-\delta$, 
    the total error of $\hat I_n(z)$ obtained from $t$ iterations of tail-averaged SGD is bounded as 
        \[
    \expect\left[ \norm{\hat I_n(z) - I(z)}_{H_\star}^2
    \, \middle| \Gcal \right] 
        \le 
        C_{K_1, M_2, \sigma_H} \, 
        \left( \Acal_1 + \Acal_2 + \Acal_3 
        \right) \poly\log\frac{p}{\delta} \,,
    \]
    where
    \begin{align*}
        \Acal_1 &= \frac{R^2 p_\star^2}{n \mu_\star}
        \left(1 + \kappa_\star \exp\left( - \frac{t}{16 \kappa_\star}\right) \right)
        \\
        \Acal_2 &= \kappa_\star \norm{I(z)}_{H_\star}^2
        \exp\left(-\frac{t}{16 \kappa_\star }\right)
        \\
        \Acal_3 &= \frac{p_\star p^2 }{n t}
        + \frac{R^2 p_\star p^2}{\mu_\star n t}
        + \frac{p_\star}{t} \norm{I(z)}_{H_\star}^2
    \end{align*}
    whenever
    \[
        n \ge C_{K_1, M_2, \sigma_H}
        \,  p_\star \left(\frac{R^2}{\mu_\star} + \frac{1}{\rho}\right) \log\frac{p}{\delta} \,.
    \]
\end{proposition}

Before proving \Cref{prop:sgd-total-error}, 
we state the final total error bound in terms of the number of calls to a Hessian-vector product oracle. To this end, define
the coefficient $\sigma_\star^2$ as 
\begin{align}
    \sigma_\star^2 := p_\star^2\left( \frac{R^2}{\mu_\star} + 1 \right) + 
    p^2 \norm{I(z)}_{H_\star}^2 \,.
\end{align}

\begin{corollary}[Total Oracle Complexity for SGD]
\label{cor:sgd:total-error}
    Consider the setting of \Cref{prop:sgd-total-error}.
    If we choose
    \[
        n \ge \max\left\{1,
        \frac{R^2}{\mu_\star} \right\} \frac{p_\star^2}{\eps} \poly\log\frac{p}{\delta}
        \quad \mbox{and}  \quad
        t \ge \left( \frac{p^2 \norm{I(z)}_{H_\star}^2}{\eps} 
        + \kappa_\star \log \left(\frac{\kappa_\star \norm{I(z)}_{H_\star}^2}{\eps}\right)
        \right) \poly\log \frac{p}{\delta} \,,
    \]
    we have $\expect\left[ \norm{\hat I_n(z) - I(z)}_{H_\star}^2
    \, \middle| \Gcal \right]  \le \eps$. Then, the minimal total number of calls to a Hessian-vector product oracle is 
    \[
        N_{\text{sgd}} \le \left(
        \frac{\sigma_\star^2}{\eps} +
        \kappa_\star \log
        \left(\frac{\kappa_\star \norm{I(z)}_{H_\star}^2}{\eps}\right)
        \right) \poly\log\frac{p}{\delta} \,.
    \]
\end{corollary}
\begin{proof}
    We use the shorthand $\Delta_\star := \norm{I(z)}_{H_\star}^2$. 
    We have that the total error is bounded as $
     \expect\left[ \norm{\hat I_n(z) - I(z)}_{H_\star}^2
    \, \middle| \Gcal \right]  \le 6\eps$ 
    if each of the terms of \Cref{prop:sgd-total-error} is bounded by $\eps$. 
    These conditions are (ignoring constants and the $\poly\log(p/\delta)$ term): 
    \begin{enumerate}[label=(\alph*)]
        \item \label{item:sgd-1}
        ${R^2 p_\star^2}/({n \mu_\star}) \le \eps$ holds, or the stronger condition 
        $n \ge \max\{1, R^2/\mu_\star\} p_\star^2 /\eps$ holds.
        \item \label{item:sgd-2}
        $R^2p_\star^2 \kappa_\star/(n \mu) \exp(-t/(16\kappa_\star)) \le \eps$ holds.
        \item \label{item:sgd-3}
        $\Delta_\star \kappa_\star \exp(-t / (16\kappa_\star))\leq \epsilon$ or 
        $t \ge 16 \kappa_\star \log(\Delta_\star \kappa_\star / \eps)$ holds. 
        \item \label{item:sgd-4}
        $p^2 p_\star / (nt) \le \eps$ or that $nt \ge p^2 p_\star / \eps$. 
        \item \label{item:sgd-5}
        $R^2 p_\star p^2 / (\mu_\star n t) \le \eps$ or that $nt \ge \frac{R^2 p_\star p^2}{\mu_\star \eps}$. 
        \item \label{item:sgd-6}
        $p^2 \Delta_\star / t \le \eps$ or that 
        $t \ge p^2 \Delta_\star / \eps$. 
    \end{enumerate}
    Under the assumption that $\eps < \Delta_\star$ (or else there is nothing to estimate), 
    the conditions \ref{item:sgd-1} and \ref{item:sgd-6} together imply that the conditions \ref{item:sgd-4} and \ref{item:sgd-5} hold. 
    Similarly, the conditions \ref{item:sgd-1} and \ref{item:sgd-3} together imply that condition \ref{item:sgd-2} holds. 
    Therefore, it suffices to have conditions 
     \ref{item:sgd-1},  \ref{item:sgd-3}, and  \ref{item:sgd-6}, which is the first claim. 
     For the second one, note that the total number of Hessian-vector product calls is $\max\{n, t\} \le n+t$. 
\end{proof}

We now prove \Cref{prop:sgd-total-error}.

\begin{proof}[Proof of \Cref{prop:sgd-total-error}]
    We denote $\Delta_\star := \norm{I(z)}_{H_\star}^2$
    and $\Delta_n := \norm{I_n(z)}_{H_n(\theta_n)}^2$ in this proof. 
    Under the event $\Gcal$, we have 
    \begin{align} \label{eq:sgd-proof-1}
        \norm{I_n(z) - I(z)}_{H_\star}^2 \le \frac{R^2 p_\star^2}{n \mu_\star} \poly\log \frac{p}{\delta} =: E_n \,.
    \end{align}
    The computational bound \Cref{lem:sgd:computational_bound} implies that 
    \begin{align*}
        \expect\left[ \norm{\hat I_n(z) - I_n(z)}_{H_n(\theta_n)}^2 
        \middle| \, Z_{1:n}\right]
        \le \kappa_n \Delta_n \exp\left(- \frac{t}{4\kappa_n} \right) + \frac{\sigma_n^2}{t} \,.
    \end{align*}
    Invoking \Cref{prop:cond-num} and \Cref{lem:sgd_noise} (which requires $n$ large enough as assumed), we can write 
    \begin{align} \label{eq:sgd-proof-2}
        \expect\left[
        \norm{\hat I_n(z) - I_n(z)}_{H_\star}^2 
        \middle| \, \Gcal \right]
        \le C \kappa_\star \Delta_\star \exp\left(- \frac{t}{16\kappa_\star} \right) + C_{K_1, M_2} \frac{p^2}{t} 
        \left( 
            \frac{p_\star}{n} 
            + \frac{\Delta_\star R^2 p_\star}{\mu_\star n}  + \Delta_\star 
        \right) \log\frac{p}{\delta} \,.
    \end{align}
    We invoke the triangle inequality to complete the proof. 
\end{proof}

The total error bounds rely on the following upper bound
of the noise term $\sigma_n^2$ in terms of the population quantities. 
Recall that, for $A, J \in \reals^{p \times p}$ with $J$ being p.s.d., the weighted spectral norm $\mynorm{A}_J := \mynorm{J^{1/2} A J^{1/2}}_2$.

\begin{lemma}\label{lem:sgd_noise}
    Under Assumptions~\ref{assm:pseudo_self_conc}, \ref{assm:sub_gaussian_grad}, \ref{assm:bounded_hessian_1}, we have, with probability at least $1-\delta$,
    \begin{align*}
        \sigma_n^2 \le C_{K_1, M_2} \cdot p^2 \left[\frac{p_\star}{n} \log{\frac{e}{\delta}} + \frac{\norm{I(z)}_{H_\star}^2}{n} \left[ \frac{R^2 p_\star}{\mu_\star} \log{\frac{e}{\delta}} + \log{\frac{2p}{\delta}} \right] + \norm{I(z)}_{H_\star}^2 \right]
    \end{align*}
    whenever $n \ge C_{K_1, M_2}\big( p_\star (R^2/\mu_\star + 1/\rho) \log{(e/\delta)} + \log{(2p/\delta)} \big)$.
\end{lemma}
\begin{proof}
    Let $\Hcal_n(Z) := H_n(\est)^{-1/2} H(Z, \est) H_n(\theta_n)^{-1/2}$.
    Then
    \begin{align*}
        \Tr(\Sigmam_n)
        &= \Tr\left\{ \frac1n \sum_{i=1}^n [\Hcal_n(Z_i) - \id_p] H_n(\theta_n)^{1/2}I_n(z) I_n(z)^\top H_n(\theta_n)^{1/2} [\Hcal_n(Z_i) - \id_p] \right\} \\
        &= \Tr\left\{ \frac1n \sum_{i=1}^n [\Hcal_n(Z_i) - \id_p]^2 H_n(\theta_n)^{1/2} I_n(z) I_n(z)^\top H_n(\theta_n)^{1/2} \right\} \\
        &= I_n(z)^\top H_n(\theta_n)^{1/2} \left\{ \frac1n \sum_{i=1}^n [\Hcal_n(Z_i) - \id_p]^2 \right\} H_n(\theta_n)^{1/2} I_n(z).
    \end{align*}
    Note that $n^{-1} \sum_{i=1}^n \Hcal_n(Z_i) = \id_p$.
    It follows that
    \begin{align}
        \Tr(\Sigmam_n)
        &= I_n(z)^\top H_n(\theta_n)^{1/2} \left[ \frac{1}{n} \sum_{i=1}^n \Hcal_n(Z_i)^2 \right] H_n(\theta_n)^{1/2} I_n(z) - \mynorm{I_n(z)}_{H_n(\theta_n)}^2 \nonumber \\
        &= I_n(z)^\top H_n(\est)^{1/2} \left[ \frac1n \sum_{i=1}^n H(Z_i, \est) H_n(\est)^{-1} H(Z_i, \est) \right] H_n(\est)^{1/2} I_n(z)  - \mynorm{I_n(z)}_{H_n(\theta_n)}^2 \nonumber \\
        &= I_n(z)^\top H_n(\est)^{1/2} H_n(\est)^{-1/2} H_\star^{1/2} \mathcal{A}_n H_\star^{1/2} H_n(\est)^{-1/2} H_n(\est)^{1/2} I_n(z) - \mynorm{I_n(z)}_{H_n(\theta_n)}^2 \nonumber \\
        &\le \left[ \mynorm{\Acal_n}_2 \mynorm{H_n(\est)^{-1/2} H_\star H_n(\est)^{-1/2}}_2 - 1 \right] \mynorm{I_n(z)}_{H_n(\theta_n)}^2 \label{eq:trace_Mn_bound},
    \end{align}
    where
    \begin{align*}
        \Acal_n := \frac1n \sum_{i=1}^n H_\star^{-1/2} H(Z_i, \est) H_\star^{-1/2} H_\star^{1/2} H_n(\est)^{-1} H_\star^{1/2} H_\star^{-1/2} H(Z_i, \est) H_\star^{-1/2}.
    \end{align*}
    The term $\mynorm{H_n(\est)^{-1/2} H_\star H_n(\est)^{-1/2}}_2$ has been controlled in \Cref{prop:bound_inverse_hess}.
    Since
    \begin{align*}
        \mynorm{I_n(z)}_{H_n(\theta_n)}^2 \le 2\mynorm{I_n(z) - I(z)}_{H_n(\theta_n)}^2 + 2\mynorm{I(z)}_{H_n(\theta_n)}^2 
    \end{align*}
    it can be controlled using \Cref{thm:if-stat-main}.
    It remains to control $\norm{\mathcal{A}_n}_2$.
    Note that
    \begin{align}
        \norm{\Acal_n}_2
        &\le \Tr(\Acal_n)
        = \Tr\left\{ \left[ \frac1n \sum_{i=1}^n \big(H_\star^{-1/2} H(Z_i, \est) H_\star^{-1/2}\big)^2 \right] H_\star^{1/2} H_n(\est)^{-1} H_\star^{1/2} \right\} \nonumber \\
        &\le p \mynorm{\left[ \frac1n \sum_{i=1}^n \big(H_\star^{-1/2} H(Z_i, \est) H_\star^{-1/2}\big)^2 \right] H_\star^{1/2} H_n(\est)^{-1} H_\star^{1/2}}_2 \nonumber \\
        &\le p \mynorm{\frac1n \sum_{i=1}^n \big(H_\star^{-1/2} H(Z_i, \est) H_\star^{-1/2}\big)^2}_2 \mynorm{H_\star^{1/2} H_n(\est)^{-1} H_\star^{1/2}}_2 \label{eq:An_bound}.
    \end{align}
    Again, the term $\mynorm{H_\star^{1/2} H_n(\est)^{-1} H_\star^{1/2}}_2$ can be controlled via \Cref{prop:bound_inverse_hess}.
    As for the term
    \begin{align}\label{eq:sum_matrix_square}
        \mynorm{\frac1n \sum_{i=1}^n \big(H_\star^{-1/2} H(Z_i, \est) H_\star^{-1/2}\big)^2}_2,
    \end{align}
    it can be bounded by 1) using the Lipschitzness of the Hessian to replace $\est$ by $\theta_\star$, and 2) using the Matrix Bernstein inequality.
    
    Let us prove the result rigorously.
    Define
    \begin{equation*}
        r_n := \sqrt{CK_1^2 \log{(8e/\delta)} \frac{p_\star}{n}} \quad \mbox{and} \quad t_n := \frac{CM_2}{-1 + \sqrt{1 + C n / \log{(16p/\delta)}}}.
    \end{equation*}
    Define the following events
    \begin{align*}
        \event_1 &:= \left\{ \norm{\est - \theta_\star}_{H_\star}^2 \le r_n^2 \right\} \\
        \event_2 &:= \left\{ \norm{H_\star^{1/2} H_n(\est)^{-1} H_\star^{1/2} - \id_p}_2 \le \frac{Rr_n/\sqrt{\mu_\star} + t_n}{1 - Rr_n/\sqrt{\mu_\star} - t_n} \right\} \\
        \event_3 &:= \left\{ \norm{I_n(z) - I(z)}_{H_\star}^2 \le \left[ M_2 r_n + (\norm{S(z, \theta_\star)}_{H_\star^{-1}} + M_2 r_n) \frac{Rr_n/\sqrt{\mu_\star} + t_n}{1 - Rr_n/\sqrt{\mu_\star} - t_n} \right]^2 \right\} \\
        \event_4 &:= \left\{ \mynorm{ \frac{1}{n} \sum_{i=1}^n [H_\star^{-1/2} H(Z_i, \theta_\star) H_\star^{-1/2}]^2 - \expect\left\{[H_\star^{-1/2} H(Z, \theta_\star) H_\star^{-1/2}]^2\right\}}_2 \le \frac12 \right\}.
    \end{align*}
    Let $Q := [H_\star^{-1/2} H(z, \theta_\star) H_\star^{-1/2}]^2 - \expect\left\{[H_\star^{-1/2} H(Z, \theta_\star) H_\star^{-1/2}]^2\right\}$.
    Under Assumption \ref{assm:bounded_hessian_1}, it holds that
    \begin{align*}
        \mynorm{[H_\star^{-1/2} H(Z, \theta_\star) H_\star^{-1/2}]^2}_2 \le \mynorm{H_\star^{-1/2} H(Z, \theta_\star) H_\star^{-1/2}}_2^2 \le M_2^2.
    \end{align*}
    As a result, it holds that $\mynorm{Q}_2 \le 2M_2^2$.
    Moreover, we have
    \begin{align*}
        \mynorm{\expect[QQ^\top]}_2 \le \expect\mynorm{QQ^\top}_2 \le \expect\mynorm{Q}_2^2 \le 4M_2^4
    \end{align*}
    and, similarly, $\mynorm{\expect[Q] \expect[Q^\top]}_2 \le 4 M_2^4$.
    Consequently, $\mynorm{\Var(Q)}_2 \le 8 M_2^4$.
    This, together with \Cref{lem:bounded_bernstein} implies that $Q$ satisfies a matrix Bernstein condition with $K_2 = 2M_2^2$ and $\sigma_H^2 = 8 M_2^4$.
    Analogously, \Cref{assm:matrix_berstein} holds true with $K_2 = 2M_2$ and $\sigma_H^2 = 4 M_2^2$.
    In the following of the proof, we assume $n \ge C \max\{ M_2^4\log(2p/\delta), K_1^2\log(e/\delta)p_\star(R^2/\mu_\star + 1/\rho) \}$.
    This implies that $\norm{\theta_n - \theta_\star}_{H_\star} < \rho$ on the event $\event_1$.
    Furthermore, we have $Rr_n/\sqrt{\mu_\star} \le 1/6$ and $t_n \le 1/6$, and thus
    \begin{align}\label{eq:bound_event_2}
        \frac{Rr_n/\sqrt{\mu_\star} + t_n}{1 - Rr_n/\sqrt{\mu_\star} - t_n} \le 1/2.
    \end{align}
    
    \textbf{Step 1. Prove the bound on the event $\event_1 \event_2 \event_3 \event_4$.}
    By the event $\event_2$ and \eqref{eq:bound_event_2}, it holds that
    \begin{align}\label{eq:sandwich_hess_bound}
        \norm{H_\star^{1/2} H_n(\est)^{-1} H_\star^{1/2}}_2, \norm{H_n(\est)^{-1/2} H_\star H_n(\est)^{-1/2}}_2 \le \frac32,
    \end{align}
    and $H_n(\theta_n) \preceq 2H_\star$.
    It follows that 
    \begin{align*}
         \norm{I_n(z) - I(z)}_{H_n(\theta_n)}^2 \le  2\norm{I_n(z) - I(z)}_{H_\star}^2 \quad \mbox{and} \quad \norm{I(z)}_{H_n(\theta_n)}^2 \le  2\norm{I(z)}_{H_\star}^2.
    \end{align*}
    As a result,
    \begin{align}\label{eq:In_bound}
        \mynorm{I_n(z)}_{H_n(\theta_n)}^2
        \le 2\mynorm{I_n(z) - I(z)}_{H_n(\theta_n)}^2 + 2\mynorm{I(z)}_{H_n(\theta_n)}^2
        \le 4\norm{I_n(z) - I(z)}_{H_\star}^2 + 4\norm{I(z)}_{H_\star}^2.
    \end{align}
    By the event $\event_3$ and \eqref{eq:bound_event_2}, it holds that
    \begin{align}\label{eq:In_diff_bound}
        \norm{I_n(z) - I(z)}_{H_\star}^2 \le \frac92 M_2^2 r_n^2 + 2\norm{S(z, \theta_\star)}_{H_\star^{-1}}^2 \left( \frac{Rr_n/\sqrt{\mu_\star} + t_n}{1 - Rr_n/\sqrt{\mu_\star} - t_n} \right)^2.
    \end{align}
    On the event $\event_4$, we get
    \begin{align*}
        \mynorm{\frac1n \sum_{i=1}^n \big(H_\star^{-1/2} H(Z_i, \theta_\star) H_\star^{-1/2}\big)^2}_2
        \le \frac12 + \mynorm{\expect\left\{[H_\star^{-1/2} H(Z, \theta_\star) H_\star^{-1/2}]^2\right\}}_2 \le \frac12 + M_2^2.
    \end{align*}
    Furthermore, by \Cref{lem:hessian_lip}, it holds that
    \begin{align*}
        \mynorm{H(Z_i, \est) - H(Z_i, \theta_\star)}_{H_\star^{-1}} \le R e^{R\norm{\theta_n - \theta_\star}_2} \norm{H(Z_i, \theta_\star)}_{H_\star^{-1}} \norm{\est - \theta_\star}_2.
    \end{align*}
    Note that $\norm{H(z, \theta_\star)}_{H_\star^{-1}} \le M_2$ and $R\norm{\est - \theta_\star}_2 \le R\norm{\est - \theta_\star}_{H_\star}/\sqrt{\mu_\star} \le 1/2$ by the event $\event_1$.
    It follows that
    \begin{align*}
        \mynorm{H_\star^{-1/2} H(Z_i, \est) H_\star^{-1/2} - H_\star^{-1/2} H(Z_i, \theta_\star) H_\star^{-1/2}}_2 = \mynorm{H(Z_i, \est) - H(Z_i, \theta_\star)}_{H_\star^{-1}} \le M_2.
    \end{align*}
    Since $\norm{A^2 - B^2}_2 \le \norm{A(A-B)}_2 - \norm{(A-B)B}_2 \le (\norm{A}_2 + \norm{B}_2) \norm{A - B}_2$, we get
    \begin{align*}
        \mynorm{\big(H_\star^{-1/2} H(Z_i, \est) H_\star^{-1/2}\big)^2 - \big(H_\star^{-1/2} H(Z_i, \theta_\star) H_\star^{-1/2}\big)^2}_2
        \le 2M_2^2,
    \end{align*}
    and thus
    \begin{align}\label{eq:square_mat_bernstein_bound}
        &\quad \mynorm{\frac1n \sum_{i=1}^n \big(H_\star^{-1/2} H(Z_i, \est) H_\star^{-1/2}\big)^2}_2 \nonumber \\
        &\le \mynorm{\frac1n \sum_{i=1}^n \big(H_\star^{-1/2} H(Z_i, \theta_\star) H_\star^{-1/2}\big)^2}_2 \; + \nonumber \\
        &\quad \mynorm{\frac1n \sum_{i=1}^n \big(H_\star^{-1/2} H(Z_i, \est) H_\star^{-1/2}\big)^2 - \frac1n \sum_{i=1}^n \big(H_\star^{-1/2} H(Z_i, \theta_\star) H_\star^{-1/2}\big)^2}_2 \le 4M_2^2.
    \end{align}
    Putting \eqref{eq:trace_Mn_bound}, \eqref{eq:An_bound}, \eqref{eq:sandwich_hess_bound}, \eqref{eq:In_bound}, \eqref{eq:In_diff_bound}, and \eqref{eq:square_mat_bernstein_bound} together, we obtain
    \begin{align*}
        \Tr(\Sigmam_n) \le (CpM_2^2 - 1) \left[18M_2^2 r_n^2 + 8 \norm{S(z, \theta_\star)}_{H_\star^{-1}}^2 \left( \frac{Rr_n/\sqrt{\mu_\star} + t_n}{1 - Rr_n/\sqrt{\mu_\star} - t_n} \right)^2 + 4\norm{I(z)}_{H_\star}^2 \right].
    \end{align*}
    Now the claim follows from $\norm{\Sigmam_n}_2 \le \Tr(\Sigmam_n)$ and $I(z) = H_\star^{-1} S(z, \theta_\star)$.
    
    \textbf{Step 2. Control the probability of $\event_1 \event_2 \event_3 \event_4$.}
    According to \Cref{lem:est_error,prop:bound_inverse_hess}, we have $\Prob(\event_1) \ge 1 - \delta/4$ and $\Prob(\event_2) \ge 1 - \delta/4$.
    Following a similar proof as \Cref{thm:if-stat-main} and noticing that $\norm{H(z, \theta)}_{H_\star^{-1}} \le M_2$ for all $\norm{\theta - \theta_\star}_{H_\star} \le \rho$, we obtain $\Prob(\event_3) \ge 1 - \delta/4$.
    Finally, invoking the matrix Bernstein inequality yields $\Prob(\event_4) \ge 1 - \delta/4$.
    Hence, we have $\Prob(\event_1 \event_2 \event_3 \event_4) \ge 1 - \delta$.
\end{proof}

\subsection{Variance Reduction: SVRG and Accelerated SVRG}
We minimize the quadratic $g_n$ from \eqref{eq:quadratic} with 
SVRG~\cite{johnson2013accelerating} or its accelerated variant~\cite{lin2018catalyst,allen2017katyusha}. 
Let $u_\star = \argmin_u f(u)$ denote the minimizer of $f_n(u)$. 
A Taylor expansion gives us the expression 
\[
    f(u) - f(u_\star) = \frac{1}{2} \norm{u - u_\star}_{H_n(\theta_n)}^2 \,.
\]
Combining this fact with standard convergence bounds of SVRG and accelerated SVRG (cf.~\Cref{sec:techn:convergence} for a review) give us the following computational bound. 
\begin{theorem} \label{thm:comp:svrg}
Suppose that the loss function $\ell$ is convex and $L$-smooth, i.e., $0 \preceq \grad^2 \ell(\cdot, z) \preceq L \id_d$ for all $z \in \Zcal$. 
Further, assume that $f_n$ is $\mu_n$ strongly convex, i.e., $H_n(\theta_n) \succeq \mu_n \id_d$. Then, SVRG starting at $u_0 \in \reals^d$ returns an iterate $u_t$ satisfying $\expect\left[\norm{u_t - u_\star}_{H_n(\theta_n)}^2 \middle| Z_{1:n} \right] \le \eps$ after $t_{\mathrm{svrg}}$ steps where
\[
    t_{\mathrm{svrg}} \le C(n + \kappa_n) \log\left(  
        \frac{\kappa_n \norm{u_0 - u_\star}_{H_n(\theta_n)}^2}{\eps}
    \right) \,,
\]
where $\kappa_n = L/\mu_n$ and $C$ is an absolute constant. 
Accelerated SVRG satisfies the same condition after $t_{\mathrm{asvrg}}$ steps where 
steps where
\[
    t_{\mathrm{asvrg}} \le C\left(n + \sqrt{n \kappa_n}\right) \log\left(  
        \frac{\kappa_n \norm{u_0 - u_\star}_{H_n(\theta_n)}^2}{\eps}
    \right) \,.
\]
\end{theorem}

This gives us the following full error bound. 
\begin{corollary}[Total Computational Cost; Variance Reduction]
\label{cor:total:svrg}
    Fix $\eps > 0$. Consider the setting of \Cref{thm:if-stat-main}, and let $\Gcal$ denote the high probability event under which its conclusions hold.
    Choose a sample size $n$ such that 
    \[
        n = C_{K_1, K_2, \sigma_H}   \frac{R^2 p_\star^2 }{\mu_\star \eps} \poly\log \frac{p}{\delta} \,.
    \]
    Then, the number $N_{\mathrm{svrg}}$ of gradient and Hessian-vector oracle calls required to obtain a point $\hat I_n(z)$ using SVRG initialized at $u_0 = 0$ such that 
    $\expect\left[\norm{\hat I_n(z) - I(z)}_{H_\star}^2 \, | \Gcal \right] \le \eps$ is bounded by 
    \[
        N_{\mathrm{svrg}} \le 
        C_{K_1, K_2, \sigma_H}  \,\,
        \kappa_\star\left(1 +  
        \frac{R^2  p_\star^2 }{L \eps}  \right) \,  \, \log\left(\frac{\kappa_\star \norm{I(z)}_{H_\star}^2}{\eps} + \kappa_\star\right) \poly\log \frac{p}{\delta}
         \,.
    \]
    The corresponding number $N_{\mathrm{asvrg}}$ for accelerated SVRG is 
    \[
        N_{\mathrm{asvrg}} \le 
        C_{K_1, K_2, \sigma_H}  \,\,  \kappa_\star
        \left(\sqrt{\frac{R^2 p_\star^2}{L \eps}} +  
        \frac{R^2 p_\star^2}{L \eps}  \right) \,  \, \log\left(\frac{\kappa_\star \norm{I(z)}_{H_\star}^2}{\eps} + \kappa_\star \right) \poly\log \frac{p}{\delta}
         \,.
    \]
\end{corollary}
\begin{proof}
    The proof is identical to that of \Cref{cor:total:cg} with \Cref{thm:comp:svrg} invoked instead of \Cref{prop:cgd_bound}.
\end{proof}

\subsection{Low Rank Approximation}

Consider the eigenvalue decomposition $H_n(\theta_n) = Q \Lambda Q\T$, where  $\Lambda = (\lambda_1, \cdots, \lambda_p)$ contains the eigenvalues of $H_n(\theta_n)$ in non-increasing order.
Recall that this method relies on approximating $H_n(\theta_n)$ with its low-rank approximation $Q \Lambda_k Q\T$
where $\Lambda_k = \diag(\lambda_1, \cdots, \lambda_k, 0, \cdots, 0)$ to approximate the product with a vector $v$ as
$
    H_n(\theta_n)^{-1} v =
    Q \Lambda^{-1} Q\T v \approx Q \Lambda_k^+ Q\T v \,,
$
where $\Lambda_k^+ = \diag(\lambda_1^{-1}, \cdots, \lambda_k^{-1}, 0, \cdots, 0) $ is the pseudoinverse of $\Lambda$. 
The rank-$k$ approximation of $v = H_n(\theta_n)^{-1} u$ is given by 
$v_k = Q \diag(\lambda_1^{-1}, \cdots, \lambda_k^{-1}, 0 \cdots, 0) Q\T u$. 

Consequently, this section gives bounds for the method of \citet{arnoldi}, who compute the low-rank approximation of the Hessian using the Lanczos/Arnoldi iterations~\cite{lanczos1950iteration,arnoldi_it}. 

The computational bound we obtain depends on the low rank $k$. 
\begin{proposition} \label{prop:a:low-rank}
    Let $\lambda_1 \ge \cdots \ge \lambda_d$ denote the eigenvalues of $H_n(\theta_n)$. Then, the low-rank estimate $\hat I_{n, k}(z)$ of $I_n(z)$ satisfies 
    \[
        \norm*{\hat I_{n, k}(z) - I_n(z)}^2_{H_n(\theta_n)} 
        \le \norm{I_n(z)}_2^2 \, \sum_{i=k+1}^p \lambda_i \,.
    \]
    We have the following two regimes depending on the decay of eigenvalues $\lambda_i(H_n(\theta_n))$:
    \begin{itemize}[nosep]
        \item If $\lambda_i(H_n(\theta_n)) \le L \, i^{-\beta}$ for some $\beta > 1$, we have 
        \[
            \norm*{\hat I_{n, k}(z) - I_n(z)}^2_{H_n(\theta_n)}
            \le C_\beta \frac{\kappa_n \norm{I_n(z)}^2_{H_n(\theta_n)}}{k^{\beta-1}} \,.
        \]
        \item If $\lambda_i(H_n(\theta_n)) \le L \exp(-\nu(k-1))$ for some $\nu > 0$, we have 
        \[
            \norm*{\hat I_{n, k}(z) - I_n(z)}^2_{H_n(\theta_n)}
            \le  C_\nu \kappa_n \, \exp(-\nu k)  \norm{I_n(z)}^2_{H_n(\theta_n)} \,.
        \]
    \end{itemize}
\end{proposition}
\begin{proof}
    Denote $v = \grad \ell(\theta_n, z)$ and $u_\star = - H_n(\theta_n)^{-1} v$. Let $q_1, \cdots, q_p$ denote the columns of $Q$.
    Using $Q\T Q = \id_p$, we get
    \begin{align*}
        \norm*{\hat I_{n, k}(z) - I_n(z)}^2_{H_n(\theta_n)} 
        &= v\T Q(\Lambda^{-1} - \Lambda_k^+) \Lambda  (\Lambda^{-1} - \Lambda_k^+) Q\T v \\
        &= u_\star \T Q \Lambda (\Lambda^{-1} - \Lambda_k^+) \Lambda  (\Lambda^{-1} - \Lambda_k^+) Q u_\star \\
        &= \sum_{i=k+1}^p \lambda_i \inp{q_i}{u}_2^2 
        \le \sum_{i=k+1}^p \lambda_i \norm{u_\star}_2^2 \,,
    \end{align*}
    where the last inequality follows from the Cauchy-Schwarz inequality and $\norm{q_i}_2 = 1$.
    For the second part of the proof, we use the bound $\norm{u}_2^2 \le \norm{u}_A^2 / \lambda_{\min}(A)$ together with
    \[
        \sum_{i=k+1}^p i^{-\beta} \le \int_k^\infty x^{-\beta} \D x = 
        \frac{k^{-(\beta - 1)}}{\beta - 1} \,,
        \quad \text{and} \quad
        \sum_{i=k+1}^p \exp(-\nu(i-1)) \le \frac{\exp(-\nu k)}{1- \exp(-\nu)} \,.
    \]
\end{proof}

\begin{corollary}[Total Computational Cost; Low-Rank Approximation]
\label{cor:total:low-rank}
    Fix $\eps > 0$. Consider the setting of \Cref{thm:if-stat-main}, and let $\Gcal$ denote the high probability event under its conclusions hold.
    Choose a sample size
    \[
        n \ge C_{K_1, K_2, \sigma_H, R}   \frac{p_\star^2 }{\mu_\star \eps} \poly\log \frac{p}{\delta} \,.
    \]
    Then, under $\Gcal$, the rank-$k$ approximation $\hat I_{n, k}(z)$ satisfies
    $\norm{\hat I_{n, k}(z) - I(z)}_{H_\star}^2 \le \eps$ for all $k$ no smaller than
    \[
        k_\star = \min\left\{ k \,: \, \sum_{i=k+1}^p \lambda_i(H_\star) \, \norm{I_n(z)}_2^2 \le \eps/32 \right\} \,.
    \]
    We have the following two regimes depending on the decay of eigenvalues $\lambda_i(H_\star)$:
    \begin{itemize}[nosep]
        \item If $\lambda_i(H_\star) \le L \, i^{-\beta}$ for some $\beta > 1$, we have 
        \[
            k_\star \le C_\beta \,  \left( 
                \frac{ \kappa_\star \norm{I(z)}_{H_\star}^2}{\eps} + \kappa_\star
            \right)^{\frac{1}{\beta-1}} \,.
        \]
        \item If $\lambda_i(H_\star) \le L \exp(-\nu(k-1))$ for some $\nu > 0$, we have 
        \[
            k_\star \le \frac{1}{\nu} \log \left( \frac{\kappa_\star \norm{I(z)}^2_{H_\star}}{\eps} + \kappa_\star \right) \,.
        \]
    \end{itemize}
\end{corollary}
\begin{proof}
    The proof follows from combining \Cref{prop:a:low-rank} with \Cref{prop:a:total_error}. 
\end{proof}

\section{Most Influential Subset: Statistical Error Bound}
\label{sec:a:subset-if}

Our goal in this section is to prove \Cref{thm:subset:main}.

\subsection{Setup}
Throughout, we assume that the Hessian $\grad^2_\theta F(\theta)$ of the population is invertible for all $\theta \in \Theta$. 
For a continuously differentiable test function $h$ such as the loss of a test example $h(\theta) = \ell(z_{\text{test}}, \theta)$, recall that  we define the population influence as 
\begin{align}
    \SIF_\alpha(h) = 
    \sup_{Q \ll P} 
    \left\{ 
        - \grad_\theta h(\theta_\star)\T \, \grad^2_\theta H_\star^{-1} \, 
        \expect_{Z \sim Q} [\grad_\theta \ell(Z, \theta_\star)] 
        \,:\,
        \frac{\D Q}{\D P} \le \frac{1}{1-\alpha}
    \right\} \,.
\end{align}

We characterize the convergence of $\SIF_{n, \alpha}(h)$ towards $\SIF_\alpha(h)$ via finite sample bounds.
Recall that, for $A, J \in \reals^{p \times p}$ with $J$ being p.s.d., the weighted spectral norm $\mynorm{A}_J := \mynorm{J^{1/2} A J^{1/2}}_2$.

We retain \Cref{assm:pseudo_self_conc} but strengthen 
the other assumptions.

\begin{customasmp}{\ref{assm:sub_gaussian_grad}'}[Bounded Gradient]\label{assm:bounded_gradient}
    The normalized gradient is bounded in a neighborhood of $\theta_\star$, i.e., there exist $M_1 \ge 1, \rho \in (0, 1]$ such that $\mynorm{\nabla \ell(z, \theta)}_{H_\star^{-1}} \le M_1$ for all $z \in \Zcal$ and $\mynorm{\theta - \theta_\star}_{H_\star} \le \rho$.
\end{customasmp}
If the normalized gradient $H_\star^{-1/2} \grad\ell(z, \theta_\star)$ is bounded, then it is also sub-Gaussian, as required by \Cref{assm:sub_gaussian_grad}. In addition, we make this assumption in a neighborhood of $\theta_\star$. 
For the next assumption, we strengthen the Bernstein condition on the normalized Hessian into a spectral norm bound in a neighborhood around $\theta_\star$.

\begin{customasmp}{\ref{assm:matrix_berstein}'}[Bounded Hessian]\label{assm:bounded_hessian}
    The normalized Hessian is bounded in a neighborhood of $\theta_\star$, i.e., there exist $M_2 \ge 1, \rho \in (0, 1]$ such that $\mynorm{H(z, \theta)}_{H_\star^{-1}} \le M_2$ for all $z \in \Zcal$ and $\mynorm{\theta - \theta_\star}_{H_\star} \le \rho$.
\end{customasmp}

Finally, we also require that the gradient and Hessian of the test function $h$ are bounded. 
\begin{customasmp}{4}[Bounded Test Function]\label{assm:bounded_test}
    There exist $M_1', M_2', \rho > 0$ such that $\mynorm{\nabla h(\theta)}_{H_\star^{-1}} \le M_1'$ and $\mynorm{\nabla^2 h(\theta)}_{H_\star^{-1}} \le M_2'$ for all $\mynorm{\theta - \theta_\star}_{H_\star} \le \rho$.
\end{customasmp}

\subsection{Proof of the Statistical Bound of Theorem \ref{thm:subset:main}}

Recall that the maximum subset influence is defined as
\begin{align*}
    \SIF_{\alpha, n}(h) = 
    \max_{w \in W_\alpha} \sum_{i=1}^n w_i v_i, \quad \text{where }
    v_i = - \inp*{\grad h(\theta_n)}{H_n(\theta_n)^{-1} \grad \ell(Z_i, \theta_n)} \,.
\end{align*}
Here $H_n(\theta_n)^{-1} \nabla \ell(Z_i, \theta_n) = -I_n(Z_i)$.
Hence, the maximum subset influence can be equivalently defined as
\begin{align*}
    \SIF_{\alpha, n}(h) = 
    \max_{w \in W_\alpha} \sum_{i=1}^n w_i \inp{\grad h(\theta_n)}{I_n(Z_i)}.
\end{align*}

We state and prove the precise version of \Cref{thm:subset:main} below.
Note that we give a bound in terms of $\left|I_{\alpha, n}(h) - I_\alpha(h)\right|$ while the main paper 
gave a bound in terms of the square. 

\begin{customthm}{\ref{thm:subset:main}}
    Under Assumptions \ref{assm:pseudo_self_conc}, \ref{assm:bounded_gradient}, \ref{assm:bounded_hessian}, and \ref{assm:bounded_test}, it holds that, with probability at least $1 - \delta$,
    \begin{align*}
        \myabs{I_{\alpha,n}(h) - I_\alpha(h)}
        \le \frac{C_{M_1, M_2, M_1', M_2'}}{(1-\alpha)\sqrt{n}} \left(
        R\sqrt{\frac{p_\star}{\mu_\star} \log\left(\frac{e}{\delta}\right) } + 
        \sqrt{\log\left(\frac{2p}{\delta}\right)} + \sqrt{\log\left(\frac{n}{\delta}\right)} \right) \,.
    \end{align*}
    whenever $n \ge C_{M_1, M_2} \left(
    \big( \frac{R^2}{\mu_\star} + \frac{1}{\rho} \big) p_\star \log\left(\frac{e}{\delta}\right) + \log\left(\frac{2p}{\delta}\right) \right)$.
\end{customthm}

The proof centrally relies on the following duality property of the superquantile. 
\begin{lemma}[\citet{rockafellar2000optimization}] \label{lem:supq:duality}
For any integrable random variable $Z \sim P$ and any $\alpha \in (0, 1)$, the superquantile satisfies the equivalent expressions
\[
    \supq_\alpha(Z) = \inf_{\eta \in \reals} \left\{\eta + \frac{1}{1-\alpha} \expect(Z - \eta)_+ \right\}
    = \sup_{Q \ll P} \left\{  \expect_{Z \sim Q}[Z] \, :\, \frac{\D Q}{\D P} \le \frac{1}{1-\alpha} \right\} \,.
\]
\end{lemma}

We now prove \Cref{thm:subset:main}.
\begin{proof}[Proof of \Cref{thm:subset:main}]
    Define the shorthand for the per-point influence as 
    \begin{align*}
        \psi_n(z, \theta) := \nabla h(\theta)^\top H_n(\theta)^{-1} \nabla \ell(z, \theta) \quad \mbox{and} \quad
        \psi(z, \theta) := \nabla h(\theta)^\top H(\theta)^{-1} \nabla \ell(z, \theta).
    \end{align*}
    Motivated by the alternate expression for the superquantile in \Cref{lem:supq:duality}, we will define 
     \begin{align*}
        \varphi_{n, n}(\theta, \eta) &:= \eta + \frac{1}{(1-\alpha)n} \sum_{i=1}^n \left(-\psi_n(Z_i, \theta)  - \eta\right)_+ \,, \\
        \varphi_n(\theta, \eta) &:= \eta + \frac{1}{(1-\alpha)n} \sum_{i=1}^n \left(-\psi(Z_i, \theta)  - \eta\right)_+ \,, \\
        \varphi(\theta, \eta) &:= \eta + \frac{1}{1-\alpha} \, \expect_{Z \sim P} \left(- \psi(Z, \theta) - \eta\right)_+ \, .
    \end{align*}
    According to \Cref{lem:supq:duality}, it holds that
    \begin{align*}
        \myabs{I_{\alpha,n}(h) - I_\alpha(h)} = \myabs{\inf_{\eta \in \reals} \varphi_{n, n}(\est, \eta) - \inf_{\eta \in \reals} \varphi(\theta_\star, \eta)},
    \end{align*}
    By the triangle inequality,
    \begin{align}\label{eq:sub_influence_two_inf}
        \myabs{\inf_{\eta \in \reals} \varphi_{n, n}(\est, \eta) - \inf_{\eta \in \reals} \varphi(\theta_\star, \eta)}
        \le \underbrace{\myabs{\inf_{\eta \in \reals} \varphi_{n, n}(\est, \eta) - \inf_{\eta \in \reals} \varphi_n(\est, \eta)}}_{\Acal} + \underbrace{\myabs{\inf_{\eta \in \reals} \varphi_n(\est, \eta) - \inf_{\eta \in \reals} \varphi(\theta_\star, \eta)}}_{\Bcal}.
    \end{align}

    As before, we prove the bound on some events and control the probability of these events.
    Before we start, we make two observations.
    First, according to \Cref{lem:bounded_subG} and \Cref{assm:bounded_gradient}, the sub-Gaussian gradient assumption, \Cref{assm:sub_gaussian_grad}, holds true with $K_1 = CM_1$.
    Second, let $Q := H_\star^{-1/2} H(Z, \theta_\star) H_\star^{-1/2} - I_p$.
    Under Assumption \ref{assm:bounded_hessian}, it holds that $\mynorm{Q}_2 = \mynorm{H(Z, \theta_\star) - H_\star}_{H_\star^{-1}} \le 1 + M_2 \le CM_2$.
    Moreover, we have
    \begin{align*}
        \mynorm{\expect[QQ^\top]}_2 \le \expect\mynorm{QQ^\top}_2 \le \expect\mynorm{Q}_2^2 \le C^2M_2^2
    \end{align*}
    and, similarly, $\mynorm{\expect[Q] \expect[Q^\top]}_2 \le C^2 M_2^2$.
    Consequently, $\mynorm{\Var(Q)}_2 \le 2C^2 M_2^2$.
    This, together with \Cref{lem:bounded_bernstein}, implies that \Cref{assm:matrix_berstein} holds true with $K_2 = M_2$ and $\sigma_H^2 = 2C^2 M_2^2$.
    
    Fix $\varepsilon > 0$ and denote $M := e M_1 M_1'$.
    Let $\mathcal{R}_\varepsilon$ be an $\varepsilon$-net of $[-M, M]$.
    It is clear that $\myabs{\mathcal{R}_\varepsilon} \le \frac{M}{\epsilon}+1$.
    Denote
    \begin{align*}
        r_n := \sqrt{CM_1^2 \frac{p_\star}{n}\log{(2e/\delta)}} \quad \mbox{and} \quad t_n := \frac{C M_2}{-1 + \sqrt{1 + C n / \log{(4p/\delta)}}}.
    \end{align*}
    Define the following events
    \begin{align*}
        \event_1 &:= \left\{ \mynorm{\nabla \ell_n(\theta_\star)}_{H_\star^{-1}}^2 \le \frac{1}{n} C M_1^2 p_\star\log(3e/\delta) \right\} \\
        \event_2 &:= \left\{ (1 - t_n)  H_\star\preceq  H_n(\theta_\star) \preceq (1 + t_n) H_\star \right\} \\
        \event_3 &:= \left\{ \myabs{\varphi_n(\theta_\star, \eta) - \varphi(\theta_\star, \eta)} \le \frac{M}{1 - \alpha} \sqrt{\frac{2\log{(6\myabs{\mathcal{R}_\epsilon}/\delta)}}{n}} \mbox{ for all } \eta \in \mathcal{R}_\varepsilon \right\}.
    \end{align*}
    In what follows, we assume that
    \begin{align}\label{eq:n_large_enough_subset}
        n \ge \max\bigg\{ CM_2^2\log(6p/\delta), CM_1^2 p_\star \left( \frac{R^2}{\mu_\star} + \frac{1}{\rho} \right)\log(3e/\delta)\bigg\}.
    \end{align}
    From the proof of \Cref{prop:bound_inverse_hess}, we know that  $t_n \le 1/3$,
    \begin{align}\label{eq:est_error_3}
        \mynorm{\est - \theta_\star}_{H_\star}^2 \le r_n^2 = \frac1n CM_1^2  p_\star \log{(2e/\delta)}\quad \mbox{on the event } \event_1 \event_2,
    \end{align}
    and $\Prob(\event_k) \ge 1 - \delta/3$ for $k \in \{1, 2\}$.
    
    \textbf{Step 1. Control $\Acal$}.
    Since $(\cdot)_+$ is $1$-Lipschitz, we get
    \begin{align}
        \myabs{\varphi_{n, n}(\est, \eta) - \varphi_n(\est, \eta)}
        &\le \frac1{(1-\alpha)n} \sum_{i=1}^n \myabs{\psi_n(Z_i, \est) - \psi(Z_i, \est)} \nonumber \\
        &\le \frac1{(1-\alpha)n} \sum_{i=1}^n \mynorm{\nabla h(\est)}_{H_\star^{-1}} \mynorm{H_n(\est)^{-1} - H(\est)^{-1}}_{H_\star} \mynorm{\nabla \ell(Z_i, \est)}_{H_\star^{-1}}, \label{eq:fnn_diff_fn}
    \end{align}
    where the last inequality follows from the definition of matrix spectral norm.
    By \eqref{eq:n_large_enough_subset} and \eqref{eq:est_error_3}, we have the $\mynorm{\est - \theta_\star}_{H_\star} \le 1$.
    It then follows from Assumptions \ref{assm:bounded_gradient} and \ref{assm:bounded_test} that $\mynorm{\nabla \ell(Z_i, \est)}_{H_\star^{-1}} \le M_1$ and $\mynorm{\nabla h(\est)}_{H_\star^{-1}} \le M_1'$.
    It remains to control $\mynorm{H_n(\est)^{-1} - H(\est)^{-1}}_{H_\star}$.
    By the triangle inequality, we have
    \begin{align*}
        \mynorm{H_n(\est)^{-1} - H(\est)^{-1}}_{H_\star}
        \le \mynorm{H_n(\est)^{-1} - H_\star^{-1}}_{H_\star} + \mynorm{H(\est)^{-1} - H_\star^{-1}}_{H_\star}.
    \end{align*}
    The first term above has been taken care of in \Cref{prop:bound_inverse_hess}:
    \begin{align*}
        \mynorm{H_n(\est)^{-1} - H_\star^{-1}}_{H_\star} \le \frac{Rr_n/\sqrt{\mu_\star} + t_n}{1 - Rr_n/\sqrt{\mu_\star} - t_n}.
    \end{align*}
    The second term can be controlled similarly:
    \begin{align*}
        \mynorm{H(\est)^{-1} - H_\star^{-1}}_{H_\star} \le \frac{Rr_n/\sqrt{\mu_\star}}{1 - Rr_n/\sqrt{\mu_\star}}.
    \end{align*}
    Putting all together, we obtain
    \begin{align}\label{eq:Acal}
        \Acal \le \sup_{\eta \in \reals}\myabs{\varphi_{n, n}(\est, \eta) - \varphi_n(\est, \eta)}
        \le \frac{M_1 M_1'}{(1 - \alpha)} \left( \frac{Rr_n/\sqrt{\mu_\star} + t_n}{1 - Rr_n/\sqrt{\mu_\star} - t_n} + \frac{Rr_n/\sqrt{\mu_\star}}{1 - Rr_n/\sqrt{\mu_\star}} \right).
    \end{align}

    \textbf{Step 2. Control $\Bcal$.}
    On a high level, we first apply a covering number argument to restrict $\eta$ to a finite number of values.
    We then control the absolute difference $\myabs{\varphi_n(\theta_n, \eta) - \varphi(\theta_\star, \eta)}$ on this finite subset.
    
    \textbf{Step 2.1. Restrict $\eta$ to a compact subset.}
    According to Assumptions \ref{assm:bounded_gradient} and \ref{assm:bounded_test}, it holds that, for any $\mynorm{\theta - \theta_\star}_{H_\star} \le 1$,
    \begin{align*}
        \myabs{\psi(z, \theta)} \le M_1 M_1' \mynorm{H(\theta)^{-1}}_{H_\star} \le M_1 M_1' e^{R\mynorm{\theta - \theta_\star}_2},
    \end{align*}
    where the last inequality follows from \Cref{prop:hessian}.
    Recall that we have shown $\mynorm{\theta_n - \theta_\star}_{H_\star} \le 1$ and $\mynorm{\theta_n - \theta_\star}_2 \le 1/R$.
    It then follows that $\myabs{\psi(z, \theta)} \le e M_1 M_1' = M$.
    Consequently, we have
    \begin{align*}
        \varphi_n(\theta_n, \eta) =
        \begin{cases}
            \eta \ge \varphi_n(\theta_n, M) & \mbox{if } \eta \ge M \\
            \eta + \frac1{(1-\alpha)n} \sum_{i=1}^n [\psi(Z_i, \theta) - \eta] \ge \varphi_n(\theta_n, -M) & \mbox{if } \eta \le -M.
        \end{cases}
    \end{align*}
    Therefore, it holds that $\inf_{\eta \in \reals} \varphi_n(\theta_n, \eta) = \inf_{\myabs{\eta} \le M} \varphi_n(\theta_n, \eta)$.
    Similarly, it can be shown that $\inf_{\eta \in \reals} \varphi(\theta_\star, \eta) = \inf_{\myabs{\eta} \le M} \varphi(\theta_\star, \eta)$.
    
    \textbf{Step 2.2. Restrict $\eta$ to a finite subset.}
    By the triangle inequality, we have
    \begin{align*}
        \myabs{\varphi_n(\est, \eta) - \varphi_n(\est, \eta')} &\le \frac{1}{(1-\alpha)n}\sum_{i=1}^n \myabs{(-\psi(Z_i,\theta_n)-\eta)_+ - (-\psi(Z_i,\theta_n)-\eta')_+} + \myabs{\eta - \eta'} \\
        &\leq \frac{1}{1-\alpha} \myabs{\eta - \eta'} + \myabs{\eta -\eta'}, \quad \mbox{$(\cdot)_+$ is $1$-Lipschitz}\\
        &= \frac{2-\alpha}{1-\alpha} \myabs{\eta - \eta'}.
    \end{align*}
    For any $\eta \in [-M, M]$, we define $\pi(\eta)$ to be the projection of $\eta$ onto $\mathcal{R}_\varepsilon$, i.e., $\myabs{\eta - \pi(\eta)} \le \varepsilon$.
    As a result,
    \begin{align*}
         \varphi_n(\est, \pi(\eta)) &\le  \varphi_n(\est, \eta) + \frac{2-\alpha}{1-\alpha}\epsilon,
    \end{align*}
    which implies
    \begin{align*}
        \inf_{\eta \in [-M, M]} \varphi_n(\est, \eta) \le \inf_{\eta \in \mathcal{R}_\epsilon} \varphi_n(\est, \eta) \le \inf_{\eta \in [-M, M]} \varphi_n(\est, \eta) + \frac{2-\alpha}{1-\alpha}\epsilon.
    \end{align*}
    Similarly,
    \begin{align*}
        \inf_{\eta \in [-M, M]} \varphi(\theta_\star, \eta) \le \inf_{\eta \in \mathcal{R}_\varepsilon} \varphi(\theta_\star, \eta) \le \inf_{\eta \in [-M, M]} \varphi(\theta_\star, \eta) + \frac{2-\alpha}{1-\alpha} \epsilon.
    \end{align*} 
    From these results we can further conclude that
    \begin{align*}
        \myabs{\inf_{\eta \in [-M, M]} \varphi_n(\theta_n, \eta) - \inf_{\eta \in [-M, M]} \varphi(\theta_\star, \eta)}
        &\le \myabs{\inf_{\eta \in \mathcal{R}_\epsilon} \varphi_n(\theta_n, \eta) - \inf_{\eta \in \mathcal{R}_\epsilon} \varphi(\theta_\star, \eta)} + \frac{2-\alpha}{1-\alpha} \epsilon \\
        &\le \sup_{\eta \in \mathcal{R}_\varepsilon} \myabs{\varphi_n(\est, \eta) - \varphi(\theta_\star, \eta)} + \frac{2-\alpha}{1-\alpha} \varepsilon.
    \end{align*}
    Therefore, using the results from Step 2.1, we obtain
    \begin{align}
        \Bcal
        &= \myabs{\inf_{\eta \in [-M, M]} \varphi_n(\theta_n, \eta) - \inf_{\eta \in [-M, M]} \varphi(\theta_\star, \eta)} \nonumber \\
        &\le \underbrace{\sup_{\eta \in \mathcal{R}_\varepsilon} \myabs{\varphi_n(\est, \eta) - \varphi_n(\theta_\star, \eta)}}_{\Bcal_1} + \underbrace{\sup_{\eta \in \mathcal{R}_\varepsilon} \myabs{\varphi_n(\theta_\star, \eta) - \varphi(\theta_\star, \eta)}}_{\Bcal_2} + \frac{2-\alpha}{1-\alpha} \varepsilon \label{eq:Bcal}.
    \end{align}
    
    \textbf{Step 2.3. Control $\Bcal_1$.}
    By the $1$-Lipschitzness of $(\cdot)_+$, we have
    \begin{align*}
        \myabs{\varphi_n(\est, \eta) - \varphi_n(\theta_\star, \eta)} &\le \frac{1}{(1-\alpha)n} \sum_{i=1}^n \myabs{\psi(Z_i, \est) - \psi(Z_i, \theta_\star)}.
    \end{align*}
    It follows from the triangle inequality that
    \begin{align*}
        \myabs{\psi(Z_i, \est) - \psi(Z_i, \theta_\star)}
        \le D_1 + D_2 + D_3,
    \end{align*}
    where
    \begin{align*}
        D_1 &:= \myabs{\nabla h(\theta_n)^\top [H(\est)^{-1} - H_\star^{-1}] \nabla \ell(Z_i, \est)} \\
        D_2 &:= \myabs{\nabla h(\est)^\top H_\star^{-1} [\nabla \ell(Z_i, \est) - \nabla \ell(Z_i, \theta_\star)]} \\
        D_3 &:= \myabs{[\nabla h(\est) - \nabla h(\theta_\star)]^\top H_\star^{-1} \nabla \ell(Z_i, \theta_\star)}.
    \end{align*}
    Following the derivation of Step 1, it holds that
    \begin{align*}
        D_1 \le M_1 M_1' \frac{Rr_n/\sqrt{\mu_\star}}{1 - Rr_n/\sqrt{\mu_\star}}.
    \end{align*}
    To control $D_2$, we use the mean value theorem to write $\nabla \ell(Z_i, \est) - \nabla \ell(Z_i, \theta_\star) = \nabla^2 \ell(Z_i, \bar \theta) (\est - \theta_\star)$ for some $\bar \theta \in \conv\{\est, \theta_\star\}$.
    As a result,
    \begin{align*}
        D_2 \le \mynorm{\nabla h(\est)}_{H_\star^{-1}} \mynorm{\nabla^2 \ell(Z_i, \bar \theta)}_{H_\star^{-1}} \mynorm{\est - \theta_\star}_{H_\star} \le M_2 M_1' r_n,
    \end{align*}
    where the last inequality follows from \eqref{eq:est_error_3} and Assumptions \ref{assm:bounded_gradient} and \ref{assm:bounded_test}.
    Similarly, we can show that $D_3 \le M_1 M_2' r_n$.
    Therefore,
    \begin{align}\label{eq:Bcal_1}
        \Bcal_1 \le \frac1{1-\alpha}\left[ M_1 M_1' \frac{Rr_n/\sqrt{\mu_\star}}{1 - Rr_n/\sqrt{\mu_\star}} + M_1 M_2' r_n + M_2 M_1' r_n \right].
    \end{align}
    
    \textbf{Step 2.4. Control $\Bcal_2$.}
    By the event $\event_3$, it holds that
    \begin{align}\label{eq:Bcal_2}
        \Bcal_2 \le \frac{M}{1-\alpha} \sqrt{\frac{2 \log{(6\myabs{\mathcal{R}_\epsilon}/\delta)}}{n}} \le \frac{M}{1-\alpha} \sqrt{\frac{2 \log{(12M/(\delta\epsilon))}}{n}}
    \end{align}
    since $\myabs{\mathcal{R}_\epsilon} \le M/\epsilon + 1 \le 2M/\epsilon$.
    Setting $\epsilon = 1/\sqrt{n}$ and combining \eqref{eq:sub_influence_two_inf}, \eqref{eq:Acal}, \eqref{eq:Bcal}, \eqref{eq:Bcal_1}, and \eqref{eq:Bcal_2} lead to, after simplification,
    \begin{align*}
        \myabs{\inf_{\eta \in \reals} \varphi_{n, n}(\est, \eta) - \inf_{\eta \in \reals} \varphi(\theta_\star, \eta)} 
        \le \frac{C_{M_1, M_2, M_1', M_2'}}{(1-\alpha)\sqrt{n}} \left(
        R\sqrt{\frac{p_\star}{\mu_\star} \log\left(\frac{e}{\delta}\right) } + 
        \sqrt{\log\left(\frac{2p}{\delta}\right)} + \sqrt{\log\left(\frac{n}{\delta}\right)} \right) 
    \end{align*}
    
    \textbf{Step 2.5. Control $\Prob(\event_1 \event_2 \event_3)$.}
    Recall from Step 2.1 that $\myabs{\psi(z, \theta_\star)} \le M$ for all $z \in \Zcal$.
    This yields, for all $\eta \in \mathcal{R}_\epsilon$,
    \begin{align*}
        0 \le (-\psi(z, \theta_\star) - \eta)_+ \le M - \eta \le 2M.
    \end{align*}
    Consequently, it follows from Hoeffding's inequality that $\Prob(\event_3) \ge 1 - \delta/3$.
    Since $\Prob(\event_k) \ge 1 - \delta/3$ for $k \in \{1, 2\}$ (\Cref{prop:bound_inverse_hess}), we obtain $\Prob(\event_1 \event_2 \event_3) \ge 1 - \delta$, which completes the proof.
\end{proof}

\section{Experimental Details}\label{appx:exp_details}
We conduct our experimentation on six datasets (two simulated, two small datasets from economics, and two natural language datasets). Here, we provide full details of the experimentation used in this paper. We start with the dataset and model details in Appendix \ref{appx:exp_data}, hyperparameter choices in Appendix \ref{appx:exp_hyperparam}, and evaluation methodology in Appendix \ref{appx:exp_eval}. 

\subsection{Data and Models}\label{appx:exp_data}

\subsubsection{Linear Regression Simulation}
We simulate a linear model with orthogonal design, which we solve using penalized ridge regression to illustrate the theoretical influence function bound results in \Cref{thm:if-stat-main}. 
Following \cite{marco},
we simulate a model $y_i = x_i^T\theta + \mu_i$ for varying sample sizes $n \in [15, 10000]$. Each $x_i$ is i.i.d. standard normal variables and $\theta \in \mathbb{R}^9$ is fixed ahead of time. We introduce contamination into the dataset with $\mu_i = (1-b_i) \Ncal(0,1) + b_i \Ncal(0,10)$ where $b_i \sim \text{Bernoulli}(.1)$. All experimental results are the average of 100 simulations. 

\subsubsection{Logistic Regression Simulation}
We simulate a simple logistic regression model to illustrate the theoretical influence function bound results in Theorem \ref{thm:if-stat-main}. We simulate a model $y_i \sim \text{Binomial}(p_i)$, where $p_i  = \big(1+\exp(-(x_i\T \theta + \mu_i))\big)^{-1}$ for varying sample sizes $n \in [15, 1000]$. Each $x_i$ is i.i.d. standard normal variables and $\theta \in \mathbb{R}^9$ is fixed ahead of time. Similar to the linear regression case, we introduce contamination into the dataset with $\mu_i = (1-b_i) \Ncal(0,1) + b_i \Ncal(0,10)$ where $b_i \sim \text{Bernoulli}(.1)$. All experimental results are the average of 100 simulations. 

\subsubsection{Oregon Medicaid Dataset}

The dataset's covariates contains economic and demographic factors, as well as whether treatment was given. The goal is to predict various attributes of the health of a person. 

\myparagraph{Data}
This dataset comes from the Oregon Medicaid study \cite{oregon_exp}. In 2008, Oregon instituted a lottery system for choosing low-income adult resident to enroll in the Medicaid program. Due to the nature of the lottery, it simulates a randomized controlled design study. A year later, a comprehensive survey was conducted on both the treatment group (those who had won the lottery) and the control group (those who did not win the lottery). We analyzed the effects of the treatment (L) on two different health outcomes: overall health indicated by a binary self-reported measure of positive (not fair, good, very good, or excellent) or negative (poor), and the number of days with good physical or mental health in the past 30 days. After removing all datapoints without entries for each response variable, we used $n = 22517$ for the overall health indicator model and $n = 20902$ for the number of days of good health model.  
\begin{table}[t]
    \centering
    \renewcommand{\arraystretch}{1.2}
    \begin{tabular}{ll}
    \toprule
   \textbf{Variable Name}  &\textbf{Description}
    \\
    \hline
    hhsize & Household size including adults and children\\
    wave\_survey & Weights used for each draw of the survey (out of 8 draws)\\
    employ\_hrs & Average hours worked per week\\
    edu & Highest level of education completed\\
    dia\_dx* & Diagnosed by a health professional with diabetes/sugar diabetes\\
    ast\_dx* & Diagnosed by a health professional with asthma\\
   hbp\_dx* & Diagnosed by a health professional with high blood pressure\\
   emp\_dx* &Diagnosed by a health professional with COPD\\
    dep\_dx* & Diagnosed by a health professional with depression or anxiety\\
    ins\_any &Currently have any type of insurance\\
    ins\_ohp* & Currently have OHP insurance\\
   ins\_private* & Currently have private insurance\\
     ins\_other* & Currently have other insurance\\
    ins\_months & Number of months (in last 6 months) have had insurance\\
   \bottomrule
    \end{tabular}
    \caption{\textbf{Explanatory variables used in the Oregon Medicaid experimentation.} The "Variable Name" corresponds to the name used in the original analysis \cite{oregon_exp}, and then a brief description is given. Variables with a (*) are binary.
    }
    \label{tab:oregon_indp_var}
\end{table} 
\myparagraph{Models}
We use ordinary least squares to solve a linear system where outcomes per individual $i$ in a household $h$ is denoted by $y_{ih}$. Since all individuals in a household chosen by the lottery can apply for Medicaid, the variable $L_h$ is equal to one if the household $h$ won the Medicaid lottery and zero otherwise. Lastly, we use a set of demographic and economic covariates $x_i$ (shown in the \Cref{tab:oregon_indp_var}). Using these, we estimate the following model for each response variable $y_{ih}$ using the model:
\begin{align*}
    y_{ih} = \theta_0 + \theta_1 L_h + \theta_2 x_{i} + \epsilon_{ih} \,.
\end{align*}
Therefore, the covariates for each person are $x_{ih} = (1, x_i, L_h)$, where $\eps_{ih}$ is assumed to be zero mean Gaussian noise. 

We ran each model with increasing sample size; for the overall health indicator model (binary classification task) we used $n = 49, 169, 575, 1954, 6634$, and for the number of days of good health model (regression) we used $n = 49, 167, 559, 1869, 6251$. The model that ran using all the training data for each model was considered the population results. All experimental results are the average of 5 repetitions. 

\subsubsection{Cash Transfer}
\myparagraph{Data}
The cash transfer dataset comes from a study of the impact of Progresa, a social program in Mexico that gives cash gifts to low income households~\cite{cash_exp}. Although, the effects on the population receiving the cash transfers is important, \citet{cash_exp} argue that we must also analyze the impact on the remaining members of the village that are not eligible in order to understand the full impact of the program.  However, due to concerns that the non-poor households might have a large influence, the authors decided to limit the range of consumption outcomes for these households (less than 10,000). This results in robustness in the analysis for the poor household, but sensitive results for the non-poor households. For our analysis, we will  only use data from time period 8. After removing all entries with no response variable (household consumption), we used the remaining $n = 19180$ datapoints. 

\myparagraph{Model}
Following the analysis in Table 1 from \cite{cash_exp}, we use total household consumption $C_i$ for an individual $i$ as the response variable, and a set of demographic and variables $X_i$ as covariates (shown in Table \ref{tab:cash_indp_var}). Lastly, we use $\text{Poor}_i$ and $\text{Nonpoor}_i$, which are interaction terms between the treatment (getting cash transfer) and being a poor (non-poor) household, as our dependent variables of interest. The model is as below,
\begin{align}
    C_i = \theta_0 + \theta_1 \text{Poor}_i + \theta_2 \text{Nonpoor}_i + \theta_3 X_i
\end{align}
The model was run with increasing sample size $n = 49, 164, 540, 1775,  5835$. The model ran using all the training data for each model was considered the population results. All experimental results are the average of 5 repetition.
\begin{table}[t]
    \centering
    \renewcommand{\arraystretch}{1.2}
    \begin{tabular}{ll}
    \toprule
   \textbf{Variable Name}  &\textbf{Description}
    \\
    \hline
    hhhsex* & Sex of head of household\\
    hectareas &  Land size (hecta-acres)\\
    vhhnum & Number of household in the village\\
    hhhage\_cl &  Age of head of household\\
    hhhspouse\_cl* &  Head of household is married\\
   \bottomrule
    \end{tabular}
    \caption{\textbf{Explanatory variables used in the Cash Transfer experimentation.} The "Variable Name" corresponds to the name used in the original analysis \cite{cash_exp}, and then a brief description is given. Variables with a (*) are binary.  }
    \label{tab:cash_indp_var}
\end{table}

\subsubsection{Question-Answering with zsRE} 

\myparagraph{Data}
This is a question-answering task, in which the inputs $x_i$ are factual questions and the targets $y_i$ are the answers.  We used the Zero-Shot Relation Extraction (zsRE) dataset~\citep{zsre}, with custom test/train split provided by \citep{knowledge_editor}. 
An example of this data can be found in Table \ref{tab:lm_data}. We use a subsample of size $4499$ for our experiments. We take the full dataset of $n=4499$ as the population and experiment with subsamples of size $49, 122, 182, 302,$ and $ 743$. The test dataset has size $n_{\text{test}} = 200$. All experimental results are the average of 5 repetitions. 

\myparagraph{Model}
For these experiments, we use a BART-base model, which was fine-tuned on the zsRE dataset by \citet{knowledge_editor}. BART-base models have 12-layers, 768-hidden units, 16 heads, and 139M parameters \cite{lewis2020bart}. Each model was fine-tuned on a subset of the full data of size $n \in \{49, 122, 182, 302, 743, 4499\}$. Fine-tuning was done using stochastic gradient descent using the Adam optimizer with a learning rate of $\gamma = 10^{-6}$ for 20 iterations. 

\subsubsection{Wikitext} 

\myparagraph{Data}
The next task is an open-ended text continuation task. The prompt $x_i$ is a natural language text sequence, while the generation $y_i$ is a 10 token continuation of the prompt. The dataset consists of random passages from WikiText-103.  We use a subsample of size $1903$ for our experiments. We take the full dataset of $n=1903$ as the population and experiment with subsamples of size $40, 105, 275, 724,$ and $ 1903$. The test dataset has size $n_{\text{test}} = 200$. All experimental results are the average of 5 repetitions. An example of this data can be found in \Cref{tab:lm_data}. 

\myparagraph{Model}
We use a DistilGPT-2 model for this experiment, which was  finetuned on the WikiText-103 dataset~\cite{merity2017pointer}. DistilGPT2 models have 6-layers, 768-hidden units, 12 heads, and 82M parameters \cite{distilgpt2}. Each model was fine-tuned on a subset of the full data of size $n \in \{40, 105, 275, 724\}$. Fine-tuning was done using stochastic gradient descent using Adam optimizer with a learning rate of $\gamma = 10^{-6}$ for 20 iterations. 
\begin{table}
\centering
\begin{tabular}{p{2cm}p{8cm}p{6cm}} 
\toprule
\textbf{Task} & \textbf{Input ($x_i$)} & \textbf{Output ($y_i$)}\\
\hline
zsRE & What country did The Laughing Cow originate?& France\\
\hline
 WikiText &  The interchange is considered by Popular Mechanics to be one of "The World's 18 Strangest Roadways" because of its height (as high as a 12-story building), its 43 permanent bridges and other unusual... & design and construction features. In 2006, the American Public Works Association named the High Five Interchange\\
\bottomrule
\end{tabular}
\caption{\textbf{Examples of the zsRE and WikiTextdataset}. The zsRE data consists of an input question $x_i$, and target answer $y_i$. The WikiText data has a paragraph as the input $x_i$ and the next 10 token continuation as the output $y_i$.}
\label{tab:lm_data}
\end{table} 

\subsection{Hyperparameters}\label{appx:exp_hyperparam}
The hyperparameters for each experimentation are detailed below. 

\myparagraph{Linear Regression Simulation}
The linear simulation was run with a penalization hyperparameter for the Ridge regression, $\alpha = 10^{-3}$.

\myparagraph{Oregon Medicaid Dataet}
This was run with a regularization parameter of $0.01$.

\myparagraph{Cash Transfer Dataset}
This was run with a regularization parameter of $0.01$.

\myparagraph{zsRE}
Each of the methods requires a different set of hyperparameters, we list these in Table \ref{tab:lm_hyperparam}. We note that we use the same regularization parameter for each method $\lambda_1 = 100$. We used twice as many SGD epochs as SVRG epochs, because one iteration in SVRG takes twice as many Hessian-vector product class as SGD. We ran the Arnoldi method for 30 iteration, which is less than SGD, this was due to lack of memory to run the Arnoldi method for more iterations (discussed in our limitations for this method). 

\begin{table}[t]
    \centering
    \renewcommand{\arraystretch}{1.2}
    \begin{tabular}{llll}
    \toprule
   \textbf{Approx. Method}  &\textbf{Hyperparameter} & \textbf{zsRE} & \textbf{WikiText}
    \\
    \hline
   & Max. Iterations & $100$&$100$\\
    \cline{2-4}
   Conjugate Gradient & Early stopping & $0.01$&$0.01$\\
    \hline
     & Number of epochs &$ 50$&$50$\\
    \cline{2-4}
   SGD &Learning rate & $5\times 10^{-4}$&$1\times 10^{-2}$\\
    \hline
     & Number of epochs & $25$&$25$\\
    \cline{2-4}
   SVRG &Learning rate & $5\times 10^{-4}$&$1\times 10^{-3}$\\
    \hline
     & Number of iterations & $30$&$30$\\
    \cline{2-4}
     Arnoldi& Top\_k eigen. & $10$& $10$\\
     \cline{2-4}
     & Number of iterations & $30$ & $50$\\
   \bottomrule
    \end{tabular}
    \caption{\textbf{Hyperparameters for the language model experiements; zsRE and WikiTExt}.
    }
    \label{tab:lm_hyperparam}
\end{table} 
\myparagraph{WikiText} Similar to zsRE, each method requires a different set of hyperparameters, refer to Table \ref{tab:lm_hyperparam}. We note that we use the same regularization parameter for each method $\lambda_1 = 1$.

\subsection{Evaluation Methodology and Other Details}\label{appx:exp_eval}
Here, we specify the quantities that appear on the x and y axes of the plots in this paper. We also give some extra details of the experimentation. 

\myparagraph{$x$ Axis} We are interested in how the empirical influence function differs from the population influences functions as sample size increases. Therefore, on the $x$ axis we place the size of the subset (sample size) of the original population that was used to calculate the empirical influence. 

\myparagraph{$y$ Axis} In each of our experimentation's we demonstrate how certain quantities change as the sample size increases. For both of the simulations and the small economic datasets, we calculate the normalized  Hessian difference between the empirical influence and population's influence, $||I_n(z) - I(z)||^2_{H_\star}$.Lastly, for the $y$ axis for both of the language model experiments (zsRE and WikiText), we compute the difference in the influence on the test set between the empirical and population influence, $G_n(z) - G(z)$. 

\myparagraph{Software} We used Python 3.7.11, Pytorch 1.10.2 and HuggingFace Transformers 4.16.2.

\myparagraph{Hardware} All experiments were run on 4 NIVIDIA Titan V GPU with 12GB memory.

\section{Technical Definitions, Tools, and Results}
\label{sec:tools}

\subsection{Definitions}

\label{sec:defs}
\begin{theorem}[Integral (Cauchy) form of remainder]\label{defn:integral_form_remainder}
Let $f(x)$ be a differentiable function on interval $I$ around a real number $a$ and $T_{n,a}(b)$ be the $n$th Taylor polynomial of a real number $b$ around $a$. For $n\geq 0$ and $b \neq a$ in the interval $I$
\begin{align*}
    f(b) = T_{n,a}(b) + \int_a^b \frac{f^{(n+1)}(t)}{n!}(b-t)^n dt.
\end{align*}
Moreover, if $n=0$ then
\begin{align*}
    f(b) = f(a) + \int_a^b f'(t) dt.
\end{align*}
\end{theorem}

\begin{definition}[Sub-Gaussian variable]\label{defn:sub_gaussian}
Let $S \in \mathbb{R}$ be a mean-zero random variable. We say $S$ is sub-Gaussian with variance parameter $\sigma^2$, if for any $\lambda \in \mathbb{R}$
\begin{align*}
    \mathbb{E}[\exp(\lambda S)] \leq \exp\bigg(\frac{\sigma^2 \lambda^2}{2}\bigg).
\end{align*}
Moreover, we define the sub-Gaussian norm of $S$ as
\begin{align*}
    \norm{S}_{\psi_2} := \inf\left\{ t>0: \mathbb{E}\left[\exp\left( \frac{S^2}{t^2} \right)\right] \le 2 \right\}.
\end{align*}
\end{definition}

\begin{definition}[Sub-Gaussian vector]\label{defn:sub_gaussian_scalar}
Let $S \in \mathbb{R}^p$ be a mean-zero random vector. We say $S$ is sub-Gaussian if $\langle S,s \rangle$ is sub-Gaussian for every $s\in \mathbb{R}^p$. Moreover, we define the sub-Gaussian norm of S as
\begin{align*}
    \norm{S}_{\psi_2} := \sup_{\norm{s}_2 =1}\norm{\langle S,s\rangle}_{\psi_2}.
\end{align*}
Note that $\norm{.}_{\psi_2}$ is a norm and satisfies, e.g., the triangle inequality.
\end{definition}

\begin{definition}[Matrix Bernstein condition]\label{defn:matrix_bernstein}
Let $H \in \mathbb{R}^{p \times p}$ be a zero-mean symmetric random matrix. We say $H$ satisfies a Bernstein condition with parameter $b >0$ if, for all $j\geq3$,
\begin{align*}
    \mathbb{E}[H^j] \preceq \frac12 j!b^{j-2} \Var(H).
\end{align*}
\end{definition}

\begin{definition}[Pseudo self-concordance] Let $\mathcal{X} \subset \mathbb{R}^p$ be open and $f:\mathcal{X} \rightarrow \mathbb{R}$ be a closed convex function. For a constant $R >0$, we say $f$ is pseudo self-concordant on $\mathcal{X}$ if
\begin{align*}
    \vert D^3_xf(x)[u,u,v]\vert \leq R \norm{u}^2_{\nabla^2f(x)}\norm{v}_2
\end{align*}
\end{definition}

\subsection{Implications of Pseudo Self-Concordance}
\label{sub:self_concord}

We give in this section useful properties of pseudo self-concordant functions.
We denote by $f: \reals^p \rightarrow \reals$ a pseudo self-concordant function with parameter $R$ throughout this section.

The next result shows that the Hessian of a pseudo self-concordant function cannot vary too fast.
\begin{proposition}[\citet{bach2010self}, Prop.~1]
\label{prop:hessian}
    For any $x, y \in \reals^p$, we have
    \begin{align*}
        e^{-R\norm{y-x}_2} \nabla^2 f(x) \preceq \nabla^2 f(y) \preceq e^{R\norm{y-x}_2} \nabla^2 f(x).
    \end{align*}
\end{proposition}

We prove below a Lipschitz-type property for the normalized Hessian of a pseudo self-concordant function.
Let $A, J \in \reals^{p \times p}$ where $J$ is p.s.d.
We denote $\norm{A}_{J} := \norm{J^{1/2} A J^{1/2}}$.
\begin{lemma}\label{lem:hessian_lip}
    Let $J \in \reals^{p \times p}$ be p.s.d.
    For any $x_1, x_2, x_\star \in \reals^p$, we have
    \begin{align*}
        \norm{\nabla^2 f(x_2) - \nabla^2 f(x_1)}_{J} \le R e^{R\norm{x_1 - x_\star}_2 \vee \norm{x_2 - x_\star}_2} \norm{\nabla^2 f(x_\star)}_{J} \norm{x_2 - x_1}_2.
    \end{align*}
\end{lemma}
\begin{proof}
    Take an arbitrary $v \in \reals^{p}$ with $\norm{v}_2 = 1$, and denote $\bar v := J^{1/2} v$.
    It holds that
    \begin{align*}
        \abs{\bar v^\top \nabla^2 f(x_2) \bar v - \bar v^\top \nabla^2 f(x_1) \bar v}
        &= \abs{D^2f(x_2)[\bar v, \bar v] - D^2 f(x_1)[\bar v, \bar v]}
        = \abs{D^3f(\bar x)[\bar v, \bar v, x_2 - x_1]}
    \end{align*}
    for some $\bar x \in \text{Conv}\{x_1, x_2\}$ by the mean value theorem.
    By the pseudo self-concordance of $f$, we obtain
    \begin{align*}
        \abs{D^3f(\bar x)[\bar v, \bar v, x_2 - x_1]} \le R \norm{\bar v}_{\nabla^2 f(\bar x)}^2 \norm{x_2 - x_1}_2.
    \end{align*}
    According to \Cref{prop:hessian}, we know $\nabla^2 f(\bar x) \preceq e^{R \norm{\bar x - x_\star}_2} \nabla^2 f(x_\star)$.
    As a result,
    \begin{align*}
        R \norm{\bar v}_{\nabla^2 f(\bar x)}^2 \norm{x_2 - x_1}_2 \le R e^{R \norm{x_1 - x_\star}_2 \vee \norm{x_2 - x_\star}_2} \bar v^\top \nabla^2 f(x_\star) \bar v \norm{x_2 - x_1}_2.
    \end{align*}
    Therefore,
    \begin{align*}
        \norm{\nabla^2 f(x_2) - \nabla^2 f(x_1)}_{J} 
        &= \sup_{\norm{v} = 1} \abs{\bar v^\top \nabla^2 f(x_2) \bar v - \bar v^\top \nabla^2 f(x_1) \bar v} \\
        &\le \sup_{\norm{v} = 1} R e^{R \norm{x_1 - x_\star}_2 \vee \norm{x_2 - x_\star}_2} \bar v^\top \nabla^2 f(x_\star) \bar v \norm{x_2 - x_1}_2 \\
        &\le R e^{R \norm{x_1 - x_\star}_2 \vee \norm{x_2 - x_\star}_2} \norm{\nabla^2 f(x_\star)}_{J} \norm{x_2 - x_1}_2.
    \end{align*}
\end{proof}

The next result shows that the local distance between the minimizer of $f$ and an arbitrary point $x$ only depends on the local information at $x$.
Its original version was given by \citet[Proposition 2]{bach2010self} and we state here a variant of it.

\begin{proposition}\label{prop:f_localization}
    Let $x \in \reals^p$ be such that $\nabla^2 f(x) \succ 0$.
    Whenever $\norm{\nabla f(x)}_{\nabla^2 f(x)^{-1}} \le \sqrt{\lambda_{\min}(\nabla^2 f(x))} / (2R)$, the function $f$ has a unique minimizer $x_\star$ and
    \begin{align*}
        \norm{x_\star - x}_{\nabla^2 f(x)} \le 4\norm{\nabla f(x)}_{\nabla^2 f(x)^{-1}}.
    \end{align*}
\end{proposition}

The lemma below is an inequality for the spectral norm used in the proof of \Cref{prop:bound_inverse_hess}.
Even though we prove it for general matrices $A$ and $B$, we will only use it for $B = I_d$.
\begin{lemma}\label{lem:spectral_norm}
    Let $A$ and $B$ be two p.d.~matrices of size $p \times p$.
    Assume that $\norm{A - B} \le s < \lambda_{\min}(B)$.
    Then we have
    \begin{align*}
        \norm{A^{-1} - B^{-1}} \le \frac{s}{\lambda_{\min}(B)\big(\lambda_{\min}(B) - s\big)}.
    \end{align*}
    In particular, if $B = I_p$ and $\norm{I-A} \le 1$, we have 
    \[
        \norm{A^{-1} - I} \le \frac{\norm{I-A}}{1-\norm{I-A}} \,.
    \]
\end{lemma}
\begin{proof}
    Since $\norm{A - B} \le s$, it holds that
    \begin{align*}
        B - sI_p \preceq A \preceq B + sI_p.
    \end{align*}
    It then follows from $\lambda_{\min}(B) I_p \preceq B$ that
    \begin{align*}
        [1 - s/\lambda_{\min}(B)] B \preceq A \preceq [1 + s/\lambda_{\min}(B)] B.
    \end{align*}
    As a result, we obtain
    \begin{align*}
        \frac1{1 + s/\lambda_{\min}(B)} B^{-1} \preceq A^{-1} \preceq \frac1{1 - s/\lambda_{\min}(B)} B^{-1}.
    \end{align*}
    Hence,
    \begin{align*}
        \norm{A^{-1} - B^{-1}} \le \frac{s/\lambda_{\min}(B)}{1 - s/\lambda_{\min}(B)} \norm{B^{-1}} \le \frac{s}{\lambda_{\min}(B)[\lambda_{\min}(B) - s]}.
    \end{align*}
\end{proof}

\subsection{Concentration of Random Vectors and Matrices}
\label{sub:rand_vec_mat}

It follows from \citet[Eq.~(2.17)]{vershynin2018high} that a bounded random vector is sub-Gaussian.
\begin{lemma}\label{lem:bounded_subG}
    Let $S$ be a random vector such that $\norm{S}_2 \overset{a.s.}{\le} M$ for some constant $M > 0$.
    Then $S$ is sub-Gaussian with $\norm{S}_{\psi_2} \le M/\sqrt{\log{2}}$.
\end{lemma}

As a direct consequence of \citet[Prop.~2.6.1]{vershynin2018high}, the sum of i.i.d.~sub-Gaussian random vectors is also sub-Gaussian.
\begin{lemma}\label{lem:sum_subG}
    Let $S_1, \dots, S_n$ be i.i.d.~sub-Gaussian random vectors, then we have $\norm{\sum_{i=1}^n S_i}_{\psi_2}^2 \le C \sum_{i=1}^n \norm{S_i}_{\psi_2}^2$.
\end{lemma}

We call a random vector $S \in \reals^d$ isotropic if $\Expect[S] = 0$ and $\Expect[SS^\top] = \id_d$.
The following theorem is a tail bound for quadratic forms of isotropic sub-Gaussian random vectors.
\begin{theorem}[\citet{ostrovskii2021finite}, Theorem A.1]\label{thm:isotropic_tail}
Let $S \in \reals^d$ be an isotropic random vector with $\norm{S}_{\psi_2} \le K$, and let $J \in \reals^{d \times d}$ be positive semi-definite.
Then,
\begin{align*}
    \Prob(\norm{S}_{J}^2 - \Tr(J) \ge t) \le \exp\left(-c\min\left\{ \frac{t^2}{K^2 \norm{J}_2^2}, \frac{t}{K\norm{J}_\infty} \right\} \right).
\end{align*}
In other words, with probability at least $1 - \delta$, it holds that
\begin{align}
  \norm{S}_{J}^2 - \Tr(J) \le C K^2\left( \norm{J}_2 \sqrt{\log{(e/\delta)}} + \norm{J}_{\infty} \log{(1/\delta)} \right),
\end{align}
where $C$ is an absolute constant.
\end{theorem}

The next lemma, which follows from \citet[Eq.~(6.30)]{wainwright2019high}, shows that a matrix with bounded spectral norm satisfies the matrix Bernstein condition.
\begin{lemma}\label{lem:bounded_bernstein}
    Let $H$ be a zero-mean random matrix such that $\norm{H}_2 \overset{a.s.}{\le} M$ for some constant $M > 0$.
    Then $H$ satisfies the matrix Bernstein condition with $b = M$ and $\sigma_H^2 = \norm{\Var(H)}_2$.
    Moreover, $\sigma_H^2 \le 2M^2$.
\end{lemma}

The next theorem is the Bernstein bound for random matrices.

\begin{theorem}[\citet{wainwright2019high}, Theorem 6.17]\label{thm:bernstein_matrix}
Let $\{H_i\}_{i=1}^n$ be a sequence of zero-mean independent symmetric random matrices that satisfies the Bernstein condition with parameter $b > 0$.
Then, for all $t > 0$, it holds that
\begin{align}
  \Prob\left( \norm*{\frac1n \sum_{i=1}^n H_i} \ge t \right) \le 2 \Rank\left(\sum_{i=1}^n \Var(H_i)\right) \exp\left\{ -\frac{n t^2}{2(\sigma^2 + bt)} \right\},
\end{align}
where $\sigma^2 := \frac1n \norm{\sum_{i=1}^n \Var(H_i)}_2$.
\end{theorem}

\subsection{Generalized Linear Models Satisfy \Cref{thm:if-stat-main} Assumptions}\label{sec:techn:glms_examples}
The assumptions used to prove \Cref{thm:if-stat-main} hold for generalized linear models under some regularity conditions.
We give two concrete examples here. 

\noindent \textit{1. Least Squares:} Let $\Zcal \subset B_{p, M} \times B_{1, M}$, where $B_{p, M} := \{x \in \reals^p: \norm{x}_2 \le M\}$ for some $M > 0$.
Consider the loss $\ell(z, \theta) := \frac{1}{2}(y-\theta^\top x)^2$ where $z=(x, y)$ denotes an input-output pair.
Assume that $H(\theta_\star) = \expect[XX^\top] \succ 0$.
\begin{enumerate}[label=(\alph*), nosep]
    \item Pseudo self-concordance. Since $\nabla^2_\theta \ell(z, \theta) = xx^\top \succeq 0$ and $\nabla^3_\theta \ell(z, \theta) = 0$, the loss $\ell$ is pseudo self-concordant for all $R \ge 0$.
    \item Sub-Gaussian gradient. Note that $\norm{\nabla_\theta \ell(Z, \theta_\star)}_2 = \norm{XX^\top \theta_\star - XY}_2 \le M^2 (\norm{\theta_\star}_2 + 1)$ and $H(\theta_\star) = \expect[XX^\top] \succ 0$.
    This is sufficient to guarantee that the normalized gradient
    $H(\theta_\star)^{-1/2} \nabla \ell(Z,\theta_\star)$ is sub-Gaussian
    (cf. \Cref{lem:bounded_subG}). 
    \item Bernstein Hessian. Note that $\norm{\nabla_\theta^2 \ell(Z, \theta_\star)}_2 = \norm{XX^\top}_2 \le M^2$, 
    the standardized Hessian\\
    $H(\theta_\star)^{-1/2} \nabla_\theta^2 \ell(Z, \theta_\star) H(\theta_\star)^{-1/2} - I_p$ satisfies the matrix Bernstein condition 
    (cf. \Cref{lem:bounded_bernstein}).
\end{enumerate}

\noindent \textit{2. Logistic Regression:} Let $\Zcal \subset B_{p,M} \times \{\pm 1\}$ for some $M > 0$.
Consider the loss $\ell(z, \theta) = \log\big(1 + \exp(-y \inp{\theta}{x})\big)$ and let $\sigma(z) = \frac{1}{1+e^{-z}}$. Assume that $H(\theta_\star) \succ 0$.
\begin{enumerate}[label=(\alph*), nosep]
    \item Pseudo self-concordance. Note that $\nabla_\theta^2 \ell(z, \theta) = \sigma(\theta^\top x) [1 - \sigma(\theta^\top x)] xx^\top$ and $D_\theta^3 \ell(z, \theta)[u, u, v] = \sigma(\theta^\top x) [1 - \sigma(\theta^\top x)] [1 - 2\sigma(\theta^\top x)] (u^\top x)^2 (v^\top x)$.
    It follows that $\abs{D^3_\theta \ell(z, \theta)[u,u,v]} \le M \norm{v}_2 \norm{u}^2_{\nabla^2 \ell(z,\theta)}$ and thus $\ell$ is pseudo self-concordant with $R \ge M$.
    \item Sub-Gaussian gradient. Note that $\norm{\nabla_\theta \ell(Z, \theta_\star)}_2 = \norm{[1 - \sigma(Y \theta_\star^\top X)] Y X}_2 \le M$.
    Therefore, the normalized gradient $H(\theta_\star)^{-1/2} \nabla \ell(Z,\theta_\star)$ is sub-Gaussian (cf. \Cref{lem:bounded_subG}).
    \item Bernstein Hessian. Note that $\norm{\nabla_\theta^2 \ell(Z, \theta_\star)}_2 \le \norm{XX^\top}_2 / 4 \le M^2/4$. It follows that the standardized Hessian $H(\theta_\star)^{-1/2} \nabla_\theta^2 \ell(Z, \theta_\star) H(\theta_\star)^{-1/2} - I_p$ satisfies the matrix Bernstein condition (cf. \Cref{lem:bounded_bernstein}).
\end{enumerate}

\subsection{Convergence Bounds of Optimization Algorithms} \label{sec:techn:convergence}
We recall here the convergence bounds of various linear system solvers. 

\myparagraph{Stochastic Gradient Descent}
We give here the convergence bounds of tail-averaged stochastic gradient descent (SGD) for general strongly convex quadratics from \cite{jain2017parallelizing, jain2017markov}. 

Suppose we wish to minimize the function 
\begin{align} \label{eq:sgd:fn}
    f(u) = \frac{1}{2} \inp{u}{A u} + \inp{b}{u} \,,
\end{align}
where $A \in \reals^{d \times d}$ is strictly positive definite and $b \in \reals^d$ is given.
Denote $u_\star = \argmin_u f(u) = -A^{-1} b$. 

Starting from some $u_0 \in \reals^d$, consider the SGD iterations
\begin{align} \label{eq:sgd:iter}
    u_{t+1} = u_t - \gamma (\hat A_t  u_t + b) \,,
\end{align}
where $\hat A_t$ is a stochastic estimator of the Hessian $A$. 
We make the following assumptions:
\begin{enumerate}[label=(\alph*),nolistsep,leftmargin=\widthof{ (3) }]
    \item The Hessian estimator $\hat A$ of $A$ is unbiased, i.e., $\expect[\hat A] = A$.  Further, we have the second moment bound $\expect[\hat A^2] \preceq B^2 A$ for some $B^2 > 0$. 
    If $\hat A \preceq L \id$ almost surely, then $B^2 \le L$ is always true. 
    \item The minimal eigenvalue of the Hessian $A$ is bounded $\lambda_{\min}(A) \ge \mu$ for some $\mu > 0$. 
\end{enumerate}

The bounds depend on the covariance matrix of the stochastic gradients at $u = u_\star$:
\[
    \Sigma := \expect\left[
        (\hat A u_\star + b) (\hat A u_\star + b)\T
    \right]
    = \expect\left[ 
        \hat A A^{-1} bb\T A^{-1} \hat A
    \right] - bb\T \,.
\]
The noise contribution is characterized by
the trace of the sandwich matrix
\[
    \sigma^2 := 
    \Tr(A^{-1/2}\Sigma A^{-1/2}) 
    = \expect\left[
    u_\star\T A^{1/2} (A^{-1/2} \hat A A^{-1/2}  - I)^2 \, A^{1/2} u_\star 
    \right]
    \,.
\]
The degree of misspecification is captured by the scalar 
\[
    \rho = \frac{d  \, \norm{A^{-1/2} \Sigma A^{-1/2}}_2}{\Tr(A^{-1/2} \Sigma A^{-1/2})} \,.
\]

\begin{theorem}[\cite{jain2017parallelizing,jain2017markov}]
\label{thm:techn:sgd}
    Consider the sequence $(u_t)_{t=0}^\infty$ produced by stochastic gradient descent \eqref{eq:sgd:iter} on function \eqref{eq:sgd:fn} with a step size $\gamma = 1/(2B^2)$. 
    The tail-averaged iterate $\bar u_t = (2/t) = \sum_{\tau = t/2}^t u_\tau$ satisfies 
    \[
        \expect\norm{\bar u_\tau - u_\star}_A^2
        \le 2 \kappa \exp\left(
        - \frac{t}{4\kappa}\right)\norm{u_0 - u_\star}_A^2 
        + 8(1 + \rho) \frac{\sigma^2}{t} \,,
    \]
    where $\kappa = B^2 / \mu$ is a condition number. 
\end{theorem}

\myparagraph{Stochastic Variance Reduced Gradient (SVRG) and its Acceleration}

Consider the optimization problem 
\[
    \min_{u \in \reals^d} \left[ f(u)  = \frac{1}{n} \sum_{i=1}^n f_i(u) \right]\,,
\]
where each $f_i$ is $L$-smooth and convex, and $f$ is $\mu$-strongly convex. If each $f_i$ is the quadratic 
\[
    f_i(u) = \frac{1}{2} \inp{u}{A_i u} + b \,,
\]
then the smoothness is equivalent to $0 \preceq A_i \preceq L \id_d$ for each $i$ and the strong convexity to $A := (1/n)\sum_{i=1}^n A_i \succeq \mu \id_d$.
Let $u_\star = \argmin f(u)$. For the  quadratic example above, we have 
$u_\star = A^{-1}b$

The following is the convergence bound for SVRG~\cite{johnson2013accelerating}. 
\begin{theorem}[\cite{hofmann2015variance}] 
\label{thm:techn:svrg}
    The sequence $(u_t)$ produced by SVRG satisfies 
    \[
        \expect[ f(u_t) - f(u_\star)]
        \le C_1 \kappa \exp\left(
        - \frac{t}{C_2(n + \kappa)}
        \right) \left(f(u_0) - f(u_\star)\right) \,,
    \]
    for $\kappa = L/\mu$ and some absolute constants $C_1$ and $C_2$. 
\end{theorem}

Accelerated SVRG~\cite{lin2018catalyst,allen2017katyusha} satisfies the following bound. 

\begin{theorem}
\label{thm:techn:csvrg}
    The sequence $(u_t)$ produced by accelerated SVRG satisfies 
    \[
        \expect[ f(u_t) - f(u_\star)]
        \le  C_1 \kappa \exp\left(
        - \frac{t}{C_2(n + \sqrt{n\kappa})}
        \right) \left(f(u_0) - f(u_\star)\right) \,,
    \]
    where $\kappa = L/\mu$ is the condition number and $C_1$ and $C_2$ are absolute constants. 
\end{theorem}

\subsection{Superquantile Review} \label{sec:a:superquantile}

\begin{figure}
    \centering
    \includegraphics[width=0.5\textwidth]{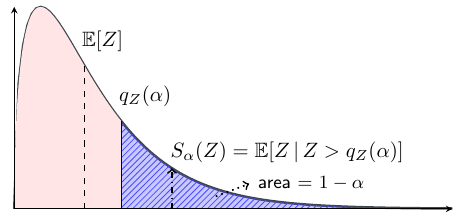}
    \caption{Expectation, quantile, and superquantile of a continuous random variable $Z$ at level $\alpha \in (0, 1)$.}
    \label{fig:superquantile}
\end{figure}

We review the various equivalent expressions of the superquantile. Consider a real-valued random variable $Z$ with distribution $P$, cumulative distribution function $F_Z$ and quantile function $q_Z(\alpha) = F_Z^{-1}(\alpha)$.

The following are equivalent expressions for the superquantile: 
\begin{align}
\begin{aligned}
    \supq_\alpha(Z) &= \sup\left\{ \expect_{Q} [Z] \, :\, \frac{\D Q}{\D P} \le \frac{1}{1-\alpha} \right\} \\
    &= \inf_{\eta \in \reals} \left\{ \eta + \frac{1}{1-\alpha} \expect_P \left(Z - \eta\right)_+  \right\} \\
    & = \frac{1}{1-\alpha} \int_\alpha^1 q_Z(\beta) \, \D \beta \,.
\end{aligned}
\end{align}
When $Z$ is a continuous random variable, the third expression is equivalent to (see \Cref{fig:superquantile})
\[
    S_\alpha(Z) = \expect[ Z \, | \, Z > q_Z(\alpha)] \,.
\]  
When $Z$ is discrete  and
takes equiprobable values $z_1, \ldots, z_n$, the three expressions above reduce to the following
\begin{align}
\begin{aligned}
    \supq_\alpha(Z) &= 
    \max\left\{ 
    \sum_{i=1}^n w_i z_i \,: \, 
    0 \le w_i \le  \frac{1}{(1-\alpha)n} 
    \text{ for all } i \in [n] \,,
    \sum_{i=1}^n w_i = 1
    \right\} \\
    &= \min_{\eta \in \reals} \left\{ 
        \eta + \frac{1}{(1-\alpha)n} \sum_{i=1}^n (z_i - \eta)_+ 
    \right\} \\
    &= \frac{1}{(1-\alpha)n} \sum_{i \in I} z_i  + \frac{\delta_\alpha}{1-\alpha} q_Z(\alpha) \,,
\end{aligned}
\end{align}
where $I = \{i \,:\, z_i > q_Z(\alpha)\}$ 
and $\delta_\alpha = F_Z(q_Z(\alpha)) - \alpha$. Note that $\delta_\alpha = 0$  when $\alpha n$ is an integer.

\vfill

\end{document}